%% file: main.tex
\PassOptionsToPackage{table}{xcolor} 
\documentclass{article} 
\usepackage{iclr2027_conference,times}

\input{math_commands.tex}

\usepackage[utf8]{inputenc} 
\usepackage[T1]{fontenc}    
\usepackage{amsthm}
\usepackage{hyperref}       
\usepackage{url}            
\usepackage{booktabs}       
\usepackage{amsfonts}       
\usepackage{nicefrac}       
\usepackage{microtype}      
\usepackage{xcolor}
\usepackage{graphicx}
\usepackage{mathtools}
\usepackage{amssymb}
\usepackage{tikz}
\usepackage{bm}
\usepackage{fontawesome5}
\usepackage{pifont}
\usepackage{algorithm}
\usepackage{algpseudocode}
\usepackage{algorithmicx}
\usepackage{tcolorbox}
\usepackage{mdframed}
\usepackage{enumitem}
\usepackage{caption}
\usepackage{subcaption}
\usepackage{listings}
\usepackage{wrapfig}
\usetikzlibrary{arrows.meta,positioning,calc, fit, backgrounds, bending}
\usepackage{etoc}

\newcommand{\indep}{\perp\!\!\!\perp}

\definecolor{forwardbg}{RGB}{225,240,255}
\definecolor{backwardbg}{RGB}{255,235,225}

\newcommand{\cmark}{\ding{51}} 
\newcommand{\xmark}{\ding{55}}

\lstdefinestyle{python}{
  language=Python,
  basicstyle=\ttfamily\small,
  keywordstyle=\color{blue},
  commentstyle=\color{gray}\itshape,
  stringstyle=\color{red!70!black},
  showstringspaces=false,
  breaklines=true,
  frame=none,
  numbers=none
}
\newtheorem{assumption}{Assumption}

\newtheorem{theorem}{Theorem}[section]
\newtheorem*{theorem*}{Theorem}
\newtheorem{lemma}{Lemma}[section]

\newmdenv[
  leftline=true,
  topline=false,
  bottomline=false,
  rightline=false,
  linewidth=2pt,
  linecolor=darkgray,
  skipabove=\baselineskip,
  skipbelow=\baselineskip,
  nobreak=true 
]{theorembox}
\newtcolorbox{lightskybluequote}{
  colback=blue!6,
  colframe=blue!25,
  boxrule=0.4pt,
  arc=2pt,
  left=6pt,right=6pt,top=6pt,bottom=6pt
}
\newtcolorbox{lightsandquote}{
  colback=orange!8!white,
  colframe=orange!20!black,
  boxrule=0.4pt,
  arc=2pt,
  left=6pt,right=6pt,top=6pt,bottom=6pt
}

\title{From Tokens to Policy: \textit{Causal and Interpretable Heterogeneous Treatment Effects Identification}} 

\author{
    Riccardo Cadei$^1$,
    Frank Otchere$^2$,
    Nyasha Tirivayi$^2$, \\
    \textbf{Gustavo Angeles Tagliaferro}$^2$, 
    \textbf{Falco J. Bargagli-Stoffi} $^3$,    
    \textbf{Francesco Locatello}$^1$ 
    \\
    $^1$ISTA \hspace{1cm} $^2$UNICEF \hspace{1cm} $^3$UCLA
}

\iclrfinalcopy 

\begin{document}
\maketitle
\lhead{Preprint}
\vspace{-0.5cm}
\begin{abstract}
\vspace{-0.1cm}
\looseness=-1Heterogeneous Treatment Effect (HTE) identification is essential to explain the impact of an intervention and optimize our policies accordingly. 
Existing approaches trade expressivity for interpretability, but, if some active heterogeneity drivers are unmeasured, methods at both ends of this spectrum allow for spurious HTE characterization with no causal reading.
In this work, we focus on controlled experiments and argue that an oracle HTE causal characterization via the latent interactors is now within reach, thanks to (i) more extensive pre-treatment measurements, i.e., multi-modal and multi-view, and (ii) scalable representations with minimal human supervision.
We then re-frame HTE identification as a Markov-blanket discovery problem on a sufficient and aligned pre-treatment representation, and introduce \emph{Neural EXposure Interaction Search} (NEXIS), an iterative procedure with a provable selection guarantee.
We deploy NEXIS on two anti-poverty programs in Africa, augmenting each with satellite imagery capturing previously unmeasured environmental effect modifiers, leading to new \textit{candidate} prescriptive guidelines to optimize the programs' next iterations.

\vspace{0.1cm}
\centering \faGlobe \textit{ Website}: \href{https://www.riccardocadei.com/NEXIS/}{\texttt{https://www.riccardocadei.com/NEXIS/}}
\end{abstract}

\vspace{-0.48cm}
\section{Introduction}
\vspace{-0.2cm}
\looseness=-1
Real-world interventions rarely work the same way for everyone. 
Understanding \textit{why} and \textit{how} a mechanism varies is essential to optimize our policies accordingly~\citep{rothman1980concepts, gail1985testing}. 
Traditional Heterogeneous Treatment Effect (HTE) estimation approaches trade expressivity for interpretability over the available pre-treatment measurements. 
Yet, even an accurate and interpretable estimation does not tell policymakers how to intervene accordingly, due to potential spurious dependencies. 
Indeed, HTE is \textit{causally} explained by the standalone treatment interactions \citep{vanderweele2009distinction}, potentially unmeasured, but it is also reflected in their proxies, common-cause companions, and upstream indirect modifiers \citep{vanderweele2007four}, which offer a misleading or no causal handle on the response. 
In the absence of an interaction's  measurement and identification, several approximate characterizations may arise, arbitrarily predictive in-distribution but drastically failing to generalize, e.g., under a new policy regime.

\begin{lightsandquote}
  \vspace{-0.05cm}
  \centering\textbf{Motivation: \textit{Disentangling the impact of anti-poverty programs beyond survey data.}}\\[2pt]
  \begin{minipage}[t]{0.52\textwidth}
    \small
    \looseness=-1Anti-poverty programs such as YOP in Uganda \citep{blattman2014generating} and LEAP 1000 in Ghana \citep{GhanaLEAP1000EvaluationTeam2018}, offer cash transfers and livelihood support for economic re-integration, and household welfare. They commonly monitor household demographics despite ignoring the environmental context, e.g., water access. They are then evaluated on average, but to optimize resource allocation in the next iterations, policymakers need to \textit{causally} explain all the impact heterogeneity, retrievable only by integrating external data sources, potentially multi-modal and unstructured, e.g., satellite imagery (see Figure~\ref{fig:both}).  Standalone demographic covariates are insufficient \citep{burke2021using, athey2017state}.
  \end{minipage}
  \hfill
  \begin{minipage}[t]{0.46\textwidth}
    \captionsetup{type=figure, font=footnotesize}
    \begin{subfigure}[t]{0.485\linewidth}
      \includegraphics[width=\linewidth]{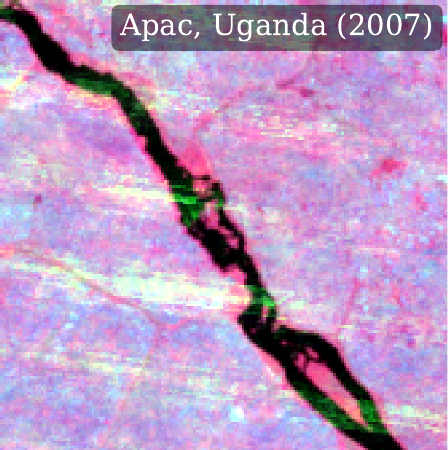}
      \vspace{-0.4cm}
      \caption{\footnotesize{Perennial river}}
      \label{fig:map}
    \end{subfigure}\hfill
    \begin{subfigure}[t]{0.49\linewidth}
      \includegraphics[width=\linewidth]{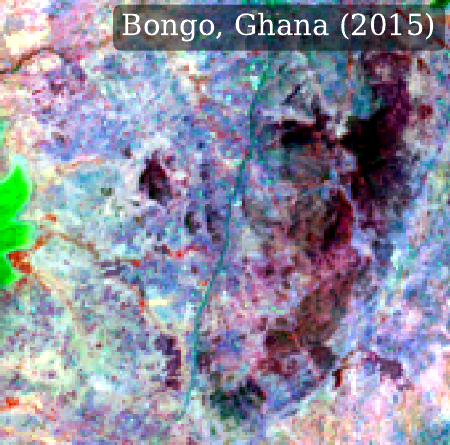}
      \vspace{-0.4cm}
      \caption{\footnotesize{Ephemeral waterway}}
      \label{fig:second}
    \end{subfigure}
    \vspace{-0.2cm}
    \caption{\footnotesize{New candidate environmental interactors in anti-poverty programs.}}
    \label{fig:both}
  \end{minipage}
\vspace{-0.15cm}
\end{lightsandquote}

Such causal characterization has been historically out of reach: existing HTE estimation approaches either model the treatment effect flexibly without identifying its primitive drivers \citep{wager2018estimation, athey2019generalized, hahn2020bayesian, kunzel2019metalearners, nie2021quasi, kennedy2023towards, hill2011bayesian, shalit2017estimating, jerzak2023image}, or identify interpretable subgroups over curated structured covariates \citep{foster2011subgroup, athey2016recursive, bargagli2020causal, hines2022variable, paillard2024measuring, bargagli2022heterogeneous}, with no guarantee of separating the (possibly unobserved) interactions from their spurious correlates. We argue two recent advances finally brought causal HTE identification within reach in controlled experiments:
\begin{enumerate}[label=\roman*., leftmargin=2.2em, itemsep=0pt, topsep=-3pt]
    \item \textit{Increasing measurement capacity}, making it more plausible that any treatment interaction is captured, even if entangled in complex multi-modal and multi-source observations, e.g., satellite imagery, sensor streams or text \citep{burke2021using, athey2017state}.
    \item \textit{Modern representation learning}, enabling sufficiently good representations, i.e., aligned with the latent factors of variation of interest, and scalable to the higher data volume. Particularly we focus on sparse-autoencoder dictionaries \citep{cunningham2023sparse, bricken2023monosemanticity, gao2025scaling, templeton2024scaling, makelov2024towards, pach2025sparse} learned from pre-trained foundation-model representations.
\end{enumerate}

\vspace{3pt}
Even in this novel data regime, identifying the treatment interactions among all the existing non-causal effect modifiers is not trivial. Furthermore, since only partial disentanglement is achieved in practice \citep{chanin2024absorption}, almost any coordinate inherits marginal heterogeneity from the treatment interactors, making \textit{effect modification search} fundamentally misaligned from \textit{interactors search}.
Recovering only direct modifiers requires conditional reasoning to identify the minimal and sufficient HTE representation, if observed. We then recast the search to a Markov-blanket discovery problem \citep{tsamardinos2003algorithms, aliferis2010local} for the effect heterogeneity over the learned dictionary, introducing \emph{Neural EXposure Interaction Search} (NEXIS), a forward-backward Markov-blanket procedure based on conditional effect-equivalence testing. We prove high probability recovery guarantees, and provide extensive empirical validation and ablations on semi-synthetic experiments, with practical variants for different data regimes.
 
We then deploy NEXIS on two real-world anti-poverty programs, the Youth Opportunities Program (YOP) in Uganda \citep{blattman2014generating} and Livelihood Empowerment Against Poverty Programme (LEAP 1000) in Ghana \citep{GhanaLEAP1000EvaluationTeam2018}, augmenting each with novel satellite imagery, capturing previously ignored environmental factors, providing concrete guidelines for the next iterations of these programs.
To the best of our knowledge, this work introduces the first procedure to identify a \textit{causal} and \textit{interpretable} HTE characterization from complex observations, translating raw measurements end-to-end into practical and robust guidelines for policy adaptation. 

In summary, our contributions are:
\begin{itemize}[itemsep=0pt, topsep=-3pt, leftmargin=2.2em, ]
    \item targeting \textbf{\textit{causal} and \textit{human-interpretable} HTE identification} and formulating it as a Markov-blanket discovery problem over a sufficient and aligned pre-treatment representation;
    \item introducing NEXIS, a \textbf{principled procedure} for the target HTE identification, with recovery guarantees, tests and ablations providing practical guidelines for different data regimes;
    \item deploying NEXIS on \textbf{two anti-poverty programs} in Africa which we complement with satellite imagery, surfacing novel prescriptive heterogeneity explanations for policy adaptation. 
\end{itemize}

\vspace{-0.05cm}
\section{Causal and Interpretable Heterogeneous Treatment Effect}
\label{sec:formalism}

\vspace{-0.3cm}
We formalize here the population target and the corresponding observed-space recovery problem.
\underline{Note}: \textit{For notational clarity, we present our development for a randomized experiment with binary treatment under perfect compliance, but the same ideas extend naturally to other causal heterogeneity functionals, such as local average treatment effect (LATE) heterogeneity \citep{bargagli2022heterogeneous} or average-treatment-effect-on-the-treated (ATT) heterogeneity \citep{callaway2021difference} in quasi-experimental settings.}

\vspace{-0.2cm}
\subsection{Heterogeneous effect characterization via Direct Effect Modifiers}
\label{sec:hte_complex}

\vspace{-0.2cm}
Consider a controlled experiment with $n$ units indexed by $i \in [n] := \{1,\ldots,n\}$, drawn \textit{i.i.d.} from a population distribution $\mathcal{P}$, where a binary treatment $T \in \{0,1\}$ is randomly assigned and potentially affecting an outcome variable $Y \in \mathbb{R}$. Let $Y(0), Y(1)$ be the potential outcomes under control and treatment respectively \citep{rubin1974estimating}. Assuming consistency and no interference (SUTVA), define $\tau := Y(1) - Y(0)$. The treatment effect can vary across units due to the treatment interaction with the complete set of pre-treatment unit characteristics $\bm{W}^{\mathrm{dir}} \in \mathbb{R}^r$, called \emph{direct} effect modifiers or \emph{interactors} \citep{vanderweele2009distinction, vanderweele2007four}. For policy guidance we aim to identify the corresponding Conditional Average Treatment Effect (CATE) functional: 
\begin{equation} 
\label{eq:tau_w} \tau(\bm{w}) := \mathbb{E}[\tau \mid \bm{W}^{\mathrm{dir}} = \bm{w}] \quad \forall\bm w \in \text{supp} (\bm W^{\mathrm{dir}}).
\end{equation}
Targeting the $\bm{W}^{\mathrm{dir}}$ characterization is desirable for two reasons:
\begin{enumerate}[label=\roman*., leftmargin=2.2em, itemsep=0pt, topsep=-3pt]
    \item It represents the \emph{minimal} and \emph{sufficient} summary of the effect heterogeneity: by definition $\bm{W}^{\mathrm{dir}}$ $d$-separates every other pre-treatment variable from the treatment effect $\tau$ \citep{richardson2013single} providing a Markov blanket of the treatment effect heterogeneity \citep{pearl1988probabilistic, koller2009probabilistic}.
    \item It is \emph{causal}\footnote{Note that the CATE is per-se a causal estimand, but not with respect to its characterization.} and \emph{interpretable}, identifying directly the corresponding interventional heterogeneity functional: 
\vspace{-0.3cm}
\begin{equation} 
\label{eq:tau_do_w} \tau^{do}(\bm{w}) := \mathbb{E}[\tau \mid do(\bm{W}^{\mathrm{dir}} = \bm{w})] \quad \forall\bm w \in \text{supp} (\bm W^{\mathrm{dir}}),
\end{equation}
\vspace{-0.48cm}

regardless of the other modifiers' distribution, while marginal interventions, e.g., on a specific $\bm{W}^{\mathrm{dir}}$ coordinate, are sub-population specific due to potential inter-direct modifier interactions.
\end{enumerate}
 But Equation \ref{eq:tau_w} is not the unique effect heterogeneity characterization, and $\bm{W}^{\mathrm{dir}}$ is generally latent.
\begin{wrapfigure}{r}{0.38\textwidth}
  \centering
  \vspace{-0.4cm}
  \begin{tikzpicture}[
      node distance=2.0cm,
      >=Latex,
      var/.style={circle, draw=black, thick, minimum size=9mm, inner sep=0pt, font=\small},
      edge/.style={->, draw=black, thick}
    ]
      \tikzset{every node/.style={anchor=center}}
      
      \node[var] (T) {$T$};
      \node[var, right=2.6cm of T] (Y) {$Y$};
      
      \coordinate (m) at ($(T)!0.5!(Y)$);
      \node[var, above=2.1cm of m] (Wdir) {$\bm{W}^{\mathrm{dir}}$};
      \node[var, left=1.0cm of Wdir] (Wind) {$\bm{W}^{\mathrm{ind}}$};
      \node[var, right=1.0cm of Wdir] (Wprx) {$\bm{W}^{\mathrm{prx}}$};
      \node[var, below=0.5cm of Wdir] (Wcc) {$\bm{W}^{\mathrm{cc}}$};

      \draw[edge] (T) -- (Y);
      \draw[edge] (Wdir) -- (Y);
      \draw[edge] (Wind) -- (Wdir);
      \draw[edge] (Wdir) -- (Wprx);
      \draw[edge] (Wind) -- (Wcc);
    \end{tikzpicture}
    
    \caption{\small Effect modification taxonomy in a causal model \citep{vanderweele2007four}. \underline{Note:} The illustrated causal relationships are \textit{only} necessary conditions for the described taxonomy. A causal model without structural equations is not sufficient to characterize effect modifications and interactions, e.g., distinguishing between direct effect modifiers from prognostic variables.}
    \label{fig:causal_model}
    \vspace{-0.8em}
\end{wrapfigure}

There are indeed other types of effect modifiers, i.e., variables varying the treatment effect within their strata.
\citet{vanderweele2007four} characterized them, distinguishing the \textit{direct} modifiers \(\bm W^{\mathrm{dir}}\), from their ancestors \emph{indirect} modifiers \(\bm W^{\mathrm{ind}}\), the \emph{proxies} \(\bm W^{\mathrm{prx}}\) descending from direct modifiers, and modifiers \emph{by common cause} \(\bm W^{\mathrm{cc}}\) sharing a common ancestor with a direct modifier, as illustrated in the causal model
in Figure \ref{fig:causal_model}. Among these, direct modifiers provide the minimal, direct causal characterization targeted here: indirect modifiers can affect \(\tau\) only through mediation by a direct modifier, while proxies and common-cause modifiers have no causal pathway to \(\tau\) \citep{richardson2013single}, and should be interpreted cautiously. Indeed, consider an anti-poverty program where market presence is a latent direct modifier. Road density visible from satellite imagery may surface as a heterogeneity correlate, while being an effect modifier by common cause. A naive interpretation would suggest building roads to amplify the program effect, while the actual driver is market presence; a road built to a place with no market would not impact the effect.

Furthermore \(\bm W^{\mathrm{dir}}\) is rarely fully measured directly, and it is instead entangled in complex pre-treatment observations potentially multi-source and multi-modal. Let \(\bm X \in \mathcal X\) denote the aggregation of all such pre-treatment measurements, e.g., combining structured demographic covariates from questionnaire responses, with unstructured signals such as satellite imagery. We assume:
\vspace{-0.15cm}
\begin{theorembox}
    \begin{assumption}[Measurement Sufficiency]
    \label{ass:measurement_sufficiency}
    The measured pre-treatment variables \(\bm X\) entangle all the heterogeneous treatment effect information, i.e., 
    \begin{equation}
    \label{eq:z_suff_w}
     \E[\tau \mid \bm X] = \E[\tau \mid \bm W^{\mathrm{dir}}]
    \qquad \text{a.s.}
    \end{equation}
    \end{assumption}
\end{theorembox}
\vspace{-0.15cm}
\looseness=-1Assumption~\ref{ass:measurement_sufficiency} is plausible in well-designed experiments where the relevant factors of variation are controlled or measured, even when the actual latent modifiers are unknown, provided access to extensive pre-treatment measurements. It is also necessary for completeness: if violated, some heterogeneity is unmeasured, so no method can recover its full causal characterization or extrapolate it reliably.
 
Existing methods can estimate causal CATEs from complex measurements, but generally do not distinguish direct modifiers from correlated proxies \citep{yaounifying, lopezpaz2026invariance}. 
However, a causal characterization is required for policy: ``\textit{which additional interventions or treatment assignment revisiting can increase the impact of a program?}".
We overcome this trade-off by leveraging modern \textit{representation learning} for extensive interpretable hypotheses generation, which would not scale with human annotations alone without prohibitive cost and bias.
To the best of our knowledge, this work targets for the first time a causal HTE characterization from complex observations, i.e., identifying the direct modifiers $\bm W^{\mathrm{dir}}$, which is also interpretable by construction due to causal mechanisms sparsity \citep{scholkopf2021toward} and representation minimality.

 \vspace{-0.2cm}
\subsection{Valid representation for Direct Effect Modifiers identification}
\label{sec:repr_learning}
 \vspace{-0.1cm}
Pretrained foundation models offer versatile representations of complex pretreatment measurements and often linearly encode semantic attributes \citep{chen2020simple,dosovitskiy2020image,radford2021learning,caron2021emerging}; whether they preserve a study's direct effect modifiers is task-dependent. Furthermore, learning end-to-end a meaningful representation on a target experiment is generally statistically infeasible due to the limited data. We then propose to represent the pre-treatment measurements constructing an encoder $\psi:\mathcal{X}\rightarrow \mathbb{R}^m$ (with $m\ll |\mathcal{X}|$), e.g., by learning a sparse autoencoder (SAE) on top of a frozen foundation-model encoder for the modality of interest \citep{cunningham2023sparse,bricken2023monosemanticity}, optionally combined with classical dictionary-learning components \citep{mairal2009online} and appended hand-crafted covariates from the processable modalities\footnote{Different hypotheses generation strategies can be considered provided satisfying the method requirements.}. The resulting dictionary coordinates
\begin{equation}
\label{eq:z_from_x}
\bm Z := \psi(\bm X) \in \mathbb R^m,
\end{equation}
\looseness=-1 are sparse candidates for post-hoc interpretation \citep{oquab2023dinov2, zhai2023sigmoid, park2023linear}, and each can be labeled, e.g., by inspecting its top-activating inputs via a vision-language model. Two distinct desiderata on the learned representation arise naturally: 

\vspace{-0.2cm}
\begin{theorembox}
\begin{assumption}[Representation Sufficiency]
\label{ass:representation_sufficiency}
All the heterogeneous-treatment-effect information in the pre-treatment measurements $\bm X$ is preserved in the representation $\bm Z$, i.e.,
\begin{equation}
    \label{eq:z_suff_x}
     \E[\tau \mid \bm Z] = \E[\tau \mid \bm X]
    \qquad \text{a.s.}
    \end{equation}
\end{assumption}
\end{theorembox}
\vspace{-0.3cm}
In autoencoders, reconstruction loss encourages information preservation but does not certify Representation Sufficiency, which is necessary for full HTE functional identification.

\vspace{-0.2cm}
\begin{theorembox}
\begin{assumption}[Principal Alignment]
\label{ass:principal_alignment}
There exists a unique injective map \(\pi : [r] \hookrightarrow [m]\), \(k \mapsto j_k\), such that for every \(k \in [r]\):
\begin{equation}
    \label{eq:principal_alignment}
    W^{\mathrm{dir},k} \;\indep\; \bm Z^{[m] \setminus \{j_k\}} \;\Big|\; Z^{j_k}.
\end{equation}
We call \(\mathcal{S}^\star := \text{\emph{Im}}(\pi) = \{j_1, \ldots, j_r\}\subseteq [m]\) the set of \emph{principal proxy coordinates} for the (latent) direct effect modifiers, and \(\bm{Z}^{\mathcal{S}^\star}=\{Z^{j_1}, ..., Z^{j_r}\}\) the corresponding \emph{principal proxies}.
\end{assumption}
\end{theorembox}
 
\vspace{-0.3cm}
Principal Alignment is a sufficiency, not purity, condition: a distinct principal coordinate screens each direct modifier from the rest, while the coordinate may remain polysemantic and the modifier correlated with other coordinates. It is a population-level idealization of the regime SAEs target for \citep{makelov2024towards, pach2025sparse}, and whether it holds for a specific dictionary can be probed empirically for supervised concepts \citep{yao2025third} (Appendix~\ref{sec:semi-synt:alignment}).
Principal Alignment maps the latent characterization of Equation~\ref{eq:tau_do_w} to the observed-space target $\mathcal{S}^{\star}$. Under faithfulness, Theorem~\ref{thm:nexis_consistency} makes it the unique mean Markov boundary of the CATE. Faithfulness makes every principal proxy marginally relevant, but not the converse: effect modification is broader than direct interaction, and entanglement transmits its signal to non-principal coordinates. Section~\ref{sec:power_paradox} shows why greater experimental power widens this gap; Section~\ref{sec:nexis} introduces NEXIS with a selection guarantee, completing the policy-design loop below:

\begin{lightsandquote}
\textbf{From Tokens to Policy:}

\vspace{0.8cm}
\centering
\(\underbrace{\tikz[remember picture, baseline]{\node[inner sep=1pt, anchor=base] (A1) {\(T\)};} \xrightarrow{\;\;\;\;} Y \xleftarrow{\;\;\;\;} \bm{X}}_{\text{i. controlled experiment}}\;\underbrace{\xrightarrow{\;\psi\;}\; \bm{Z}}_{\substack{\text{ii. represent}\\ \text{(\textit{interpretable})}}}\;\underbrace{\xrightarrow{\;\substack{\text{NEXIS}}\;}\; \bm{Z}^{\mathcal{S}^\star} \;\approx\; \tikz[remember picture, baseline]{\node[inner sep=1pt, anchor=base] (W) {\(\bm{W}^{\mathrm{dir}}\)};}}_{\substack{\text{iii. identify} \\ \text{(\textit{causal})}}}\)%
\begin{tikzpicture}[remember picture, overlay]
    \draw[-{Computer Modern Rightarrow[scale=1.2]}, thick, shorten >=2pt, shorten <=2pt]
    (W.north) to[out=110, in=70, looseness=0.45]
    node[above, font=\scriptsize] {\(\text{iv. update policy}\)}
    (A1.north);
\end{tikzpicture}
\\[4pt]
\small\itshape
\end{lightsandquote}
 
\section{On the distinction between Effect Modification and Interaction}
\label{sec:power_paradox}

\vspace{-0.1cm}
Marginal effect-modification testing is the natural first attempt to recover $\mathcal{S}^\star$. Faithfulness ensures population recall: every principal proxy is marginally relevant. But the converse fails. Effect modification need not reflect direct interaction, and learned representations entangle modifier signal across coordinates, so non-principal coordinates can be genuinely relevant. The target is the minimal sufficient representation of the CATE, not all marginal modifiers.

\vspace{-0.25cm}
\paragraph{Effect modification entanglement.}
In learned representations, a coordinate can principally proxy a non-direct modifier and non-principally proxy many others through entanglement \citep{chanin2024absorption}, regardless of Principal Alignment. Let
\begin{equation}
\label{eq:mstar}
\mathcal{M}^\star
:=
\bigl\{j \in [m] :
\mathbb{P}\!\bigl(\E[\tau \mid Z^j] \neq \E[\tau]\bigr) > 0\bigr\}
\end{equation}
the marginal effect-modifier coordinates within the learned representation.
Every $j_k \in \mathcal{S}^\star$ belongs to $\mathcal{M}^\star$\footnote{Assuming local mean-adjacency faithfulness clause (see Assumption~\ref{ass:faithfulness} with $A = \emptyset$).}, but so does any entangled coordinate carrying real modifier signal. In a reconstruction-rich dictionary with $m \gg r$, non-principal proxies vastly outnumber the principal ones, making $\mathcal{M}^\star$ a redundant, non-causal superset of the target \citep{vanderweele2009distinction}. Marginal activation magnitude cannot separate them: a non-principal proxy of a strong direct modifier can outrank the principal proxy of a weaker one, and the cutoff is unknown.%

\vspace{-0.25cm}
\paragraph{Experimental Power Paradox.}
Marginal testing fails structurally, not from low power. As power grows through larger $n$, stronger signals, or more sensitive tests, it converges to $\mathcal{M}^\star$ and \emph{accumulates} non-direct modifiers backed by genuine evidence. It is consistent for the wrong target: power widens the visible gap from $\mathcal{S}^\star$. Recovering $\mathcal{S}^\star$ instead requires conditional tests that rule out redundant candidates, the design principle behind NEXIS.

\vspace{-0.25cm}
\paragraph{Empirical evidence.}
Similarly to \citet{mencattini2025exploratory}, we design a semi-synthetic benchmark simulating different latent data generating processes pairing them with real-world face images from CelebA \citep{liu2015deep} as observed pre-treatment measurements. A TopK SAE ($m{=}9{,}216$ codes) trained on SigLIP~2 patch representations \citep{tschannen2025siglip} provides the candidate dictionary; Appendix~\ref{sec:semi-synt:setup} details its construction and the supervised principal-coordinate labeling used to assess alignment.
Figure~\ref{fig:power_paradox_synth} reports Precision and Recall in recovering the two principal coordinates for \texttt{wearing a hat} and \texttt{wearing eyeglasses} with marginal screening and NEXIS
as effect size and sample size grow.
Marginal screening faithfully tracks $\mathcal{M}^\star\approx[m]$ while diverging from
$\mathcal{S}^\star$ with precision collapse; NEXIS recovers $\mathcal{S}^\star$ increasingly well
as power grows. Experiment details and extended ablations are reported in Appendix~\ref{sec:semi-synt}.

\vspace{-0.25cm}
\begin{figure}[h!]
  \centering
  \includegraphics[width=\textwidth]{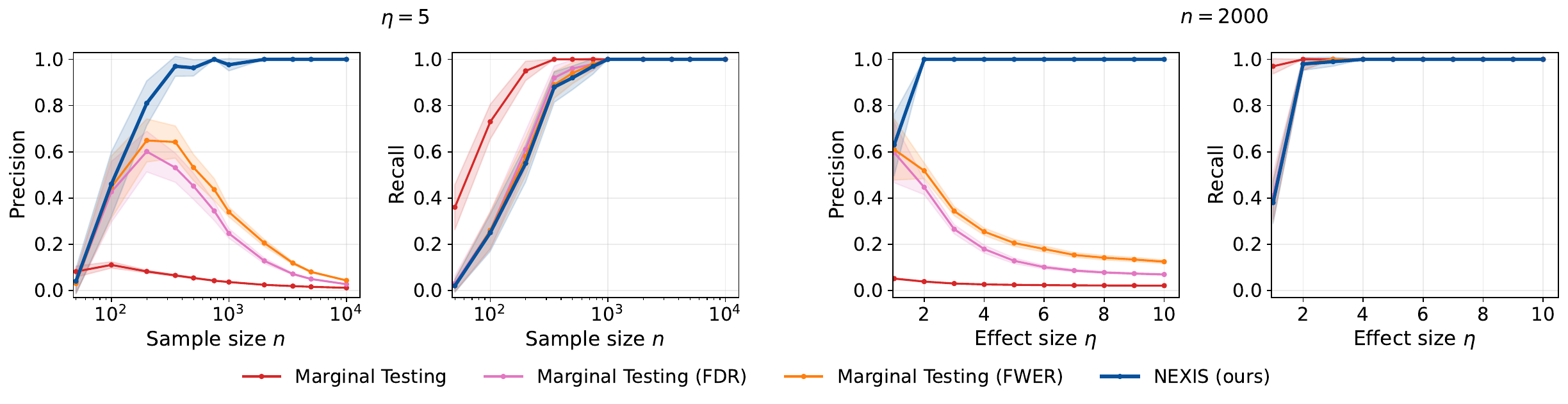}
  \caption{\small \textit{Experimental power paradox} on CelebA: Increasing the sample size ($n$) or the effect magnitude ($\eta$) of the experiment, marginal screening accumulates non-direct modifiers (false discoveries) and diverges from $\mathcal{S}^\star$ recovery, i.e., precision collapse. NEXIS robustly mitigates such behavior by iterative conditioning.}
  \label{fig:power_paradox_synth}
\end{figure}

\vspace{-0.35cm}
\section{Neural EXposure Interaction Search}
\label{sec:nexis}

\vspace{-0.25cm}
Given a randomized experiment and pre-treatment representation $\bm Z$, the task is to recover $\mathcal{S}^\star$ without observing $\bm W^{\mathrm{dir}}$. \emph{Neural EXposure Interaction Search} (NEXIS) adds the coordinate carrying the strongest residual heterogeneity signal, conditional on the current selection, then removes any coordinate redundant given a subset of the others. It uses \emph{CATE-equivalence tests}: for candidate \(j \in [m]\) and selection \(S \subseteq [m] \setminus \{j\}\), $p_j(S)$ is a valid \(p\)-value for
\begin{equation}
\label{eq:cate_ci_null}
H_0(j \mid S):\; \E[\tau \mid \bm Z^{S \cup \{j\}}] = \E[\tau \mid \bm Z^{S}] \quad \text{a.s.}
\end{equation}
Instantiation details and suggested defaults are presented in Appendix~\ref{sec:method-details}. 
 
\begin{algorithm}[h]
\caption{Neural EXposure Interaction Search (NEXIS)}
\label{alg:nexis}
\begin{algorithmic}[1]
\State \textbf{Input:} randomized experiment with pre-treatment representations; significance $\alpha\in(0,1)$ 
\vspace{0.15cm}
\State $S \gets \emptyset$,\; $S_{\mathrm{prev}} \gets \{-1\}$
\While{$S \neq S_{\mathrm{prev}}$ and $|S| < m$} \Comment{forward search}
    \State $S_{\mathrm{prev}} \gets S, \quad\bar{S} \gets [m] \setminus S$
    \State $j^{\star} \gets \arg\min_{j \in \bar{S}} p_j(S)$
    \If{$p_{j^{\star}}(S) \le \alpha / |\bar{S}|$} $S \gets S \cup \{j^{\star}\}$ 
    \EndIf
\EndWhile
\State $S \gets \{ j \in S : \max_{A \subseteq S \setminus \{j\}} p_j(A) \le \alpha / m \}$ \Comment{terminal backward certification}
\State \Return $S$
\end{algorithmic}
\end{algorithm}
\looseness=-1 NEXIS adapts classic forward--backward Markov-blanket discovery \citep{tsamardinos2003algorithms, aliferis2010local} to the mean Markov boundary of the unobserved CATE over a learned representation. Its forward step is IAMB's grow phase \citep{tsamardinos2003algorithms}; after it stops at $\widetilde S$, the terminal backward step applies the MMPC/HITON-PC subset rule \citep{tsamardinos2003time, aliferis2010local} at level $\alpha/m$. On the recall event, a false coordinate survives only by rejecting the fixed true null $H_0(j \mid \mathcal S^\star)$, so precision needs validity only there. The terminal step costs at most $|\widetilde S|\,2^{|\widetilde S|-1}$ tests, practically feasible by the causal mechanism sparsity principle. The Bonferroni forward gate \citep{bonferroni1936teoria} can be relaxed, e.g., to FDR \citep{benjamini1995controlling}; Appendices ~\ref{sec:method-details:rho} and ~\ref{sec:method-details:backward} discuss the optional interleaved backward step and practical spectral-gap default $\rho$.

\begin{theorembox}
\begin{theorem}[Causal Identification]
\label{thm:nexis_consistency} \hspace{-0.1cm}
In a randomized experiment with Principal Aligned pre- treatment measures with faithfulness and test validity/consistency, NEXIS' outcome \(\widehat{\mathcal{S}}_n\) satisfies:
\begin{equation}
\label{eq:selection_consistency}
\liminf_n \mathbb{P}(\widehat{\mathcal{S}}_n = \mathcal{S}^\star) \ge 1 - \alpha,
\end{equation}
and, additionally assuming Measurement and Representation Sufficiency:
\begin{equation}
\label{eq:causal_id}
    \tau^{do}(\bm W^{\mathrm{dir}}) = \tau(\bm W^{\mathrm{dir}}) \;=\; \E[\tau \mid \bm Z^{\mathcal S^{\star}}] \quad \text{a.s.}
\end{equation}
\end{theorem}
\end{theorembox}

\begin{proof}[Proof sketch]
Principal Alignment and composition faithfulness guarantee $\mathcal{S}^\star$ sufficiency, while local mean-adjacency faithfulness guarantees its necessity. Test consistency guarantees asymptotic recall, while the terminal step and test validity bound asymptotic false retention by $\alpha$, testing only at the fixed set $\mathcal{S}^\star$ rather than at data-dependent ones (post-selection inference). Measurement and Representation Sufficiency supply the causal interpretation. Full proof in Appendix~\ref{sec:proofs}.
\end{proof}

\vspace{-0.2cm}
Principal Alignment\footnote{Without Principal Alignment, the one-coordinate-per-modifier target need not be unique.}, faithfulness, and test validity and consistency ensure recovery of the principal proxies; Measurement and Representation Sufficiency further guarantee that such proxies capture all the heterogeneity information, thereby identifying the joint (latent direct modifiers) interventional effect. As such, Theorem~\ref{thm:nexis_consistency} supports policy hypotheses and counterfactual comparisons, e.g., across bundled program configurations.
See Appendix~\ref{sec:method-details} for implementation details and practical variants; see Appendix~\ref{sec:semi-synt} for empirical validation and extensive ablations.

\vspace{-0.2cm}
\section{Related Work}
\label{sec:related}
\vspace{-0.2cm}
 
\paragraph{Heterogeneous treatment effect identification.}

The bulk of existing HTE estimation methods trade expressivity for interpretability and require human supervision for downstream policy relevance. \emph{Decision rule}-based methods \citep{foster2011subgroup, athey2016recursive, bargagli2020causal} return interpretable subgroups but operate on expensive and potentially limited hand-crafted covariates.
\emph{Forest}-based estimators \citep{wager2018estimation, athey2019generalized, hahn2020bayesian} estimate the effect heterogeneity flexibly on structured domains, but return pointwise predictions without explicit heterogeneity characterization. \emph{Meta-learners} \citep{kunzel2019metalearners, nie2021quasi} are agnostic to the base estimator and can in principle wrap pretrained foundation-model nuisances on unstructured covariates, but they have not been demonstrated to do so at experimental scale, and they share the same uninterpretable pointwise output. \emph{Post-hoc} CATE interpretation methods \citep{hines2022variable, paillard2024measuring, boileau2025nonparametric, bargagli2026causal} sit on top of any base estimator and rank candidate modifiers by importance recovering interpretability. However, all these methods have not been developed to operate in the context of entangled and  multimodal representations of the potential effect modifiers space. \emph{Deep Learning}-based estimators \citep{jerzak2023image} extend HTE estimation to unstructured measurements such as imagery, gaining expressivity at the cost of modifier-level interpretability. Table~\ref{tab:related_methods} positions NEXIS against these families. NEXIS achieves all the desiderata by leveraging modern experimental practice, i.e., extensive measurement collections supporting Measurement Sufficiency, and modern representation learning, i.e., sparse-autoencoder-based dictionaries supporting Representation Sufficiency and Principal Alignment.

\vspace{-0.15cm}
\begin{table}[h]
\caption{\small Heterogeneous Treatment Effect estimation properties by family.}
\label{tab:related_methods}
\vspace{-0.25cm}
\centering
\small
\begin{minipage}[t]{0.62\textwidth}
\vspace{0pt}
\centering

\renewcommand{\arraystretch}{1.2}
\resizebox{\linewidth}{!}{%
\begin{tabular}{cccc}
\toprule
\textbf{Method} &
\textbf{Empiricist} &
\textbf{Interpretable} &
\textbf{Prescriptive} \\
\midrule

\rowcolor{gray!10}
\textit{Forest}-based
& \xmark & \xmark & \xmark \\

Meta-Learners
& (potentially) & \xmark & \xmark \\

\rowcolor{gray!10}
Post-hoc interpretations
& \xmark & \cmark & \xmark \\

\textit{Decision rule}-based
& \xmark & \cmark & \xmark \\

\rowcolor{gray!10}
\textit{Deep Learning}-based
& \cmark & \xmark & \xmark \\

\textbf{NEXIS} (\textit{ours})
& \cmark & \cmark & \cmark \\
\bottomrule
\end{tabular}%
}
\end{minipage}
\hfill
\begin{minipage}[t]{0.36\textwidth}
\vspace{0pt}
\raggedright

\vspace{0.1cm}
\textbf{Empiricist}: data-driven feature extraction from raw pre-treatment observations.

\textbf{Interpretable}: human-readable HTE characterization.

\textbf{Prescriptive}: targets treatment-interactive factors.

\vspace{0.05cm}
\cmark{} = fully,\quad
\xmark{} = not fully.
\end{minipage}
\end{table}

\vspace{-0.1cm}
\paragraph{Representation learning for causal downstream tasks.}
Causal Representation Learning was first articulated as the recovery of causally meaningful latent variables from raw observations as a general-purpose target \citep{scholkopf2021toward}, an ambitious goal that faces fundamental identifiability barriers. The field has since moved toward well-specified causal targets, in the spirit already suggested by Causal Feature Learning \citep{chalupka2017causal} albeit in controlled, low-dimensional settings. Pretrained encoders are now integrated into specific causal pipelines: for treatment-effect estimation on unstructured outcomes \citep{cadei2024smoke, cadei2025causal}, for unsupervised effect discovery on unstructured outcomes \citep{mencattini2025exploratory}, for confounding adjustment via text embeddings \citep{veitch2020adapting}, and for HTE estimation on imagery \citep{jerzak2023image}. Our work extends this empiricist line, unlocking the identification of a causal and interpretable HTE characterization from unstructured \emph{pre-treatment} observations.

\vspace{-0.1cm}
\paragraph{Dictionary learning and interpretability.}
Foundation-model representations admit several routes to a candidate concept dictionary: classical sparse coding \citep{mairal2009online}, sparse autoencoders \citep{cunningham2023sparse, bricken2023monosemanticity, gao2025scaling}, and recent variants addressing entanglement failure modes such as feature absorption \citep{chanin2024absorption}, including transcoders \citep{dunefsky2024transcoders} and Matryoshka SAEs \citep{bussmann2025learning}. Once a concept is identified, an interpretation pipeline translates it into a human-readable summary via top-activating examples, supervised probes, or LM-assisted descriptions of activating inputs \citep{bills2023language}. NEXIS is agnostic to the specific dictionary and interpretation pipeline; particularly we train top-\textit{k} SAEs on top of foundation-model encoders for pre-treatment representations, and call VLM-assisted summaries contrasting top- against bottom-activating inputs for interpretation, but alternative pipelines could be considered.

\vspace{-8pt}
\section{Extending anti-poverty programs interpretation}
\vspace{-6pt}
\looseness=-1We apply NEXIS to two cash-transfer evaluations augmented with satellite imagery. In each case, it surfaces interpretable candidate effect modifiers, including landscape features absent from the original surveys. With no ground truth, these results are hypotheses for field validation, not established causal modifiers or policy prescriptions. For quantitative validation on semi-synthetic experiments, see Appendix~\ref{sec:semi-synt}.

\begin{lightsandquote}
\vspace{-2pt}
  \noindent\textbf{Results: \textit{From demographic to environmental anti-poverty programs impact heterogeneity.}}
    In both programs, the selected modifiers are mainly environmental factors, coarsely proxied or absent in survey data, while no individual-level demographic is selected.
\vspace{-5pt}
\end{lightsandquote}

\subsection{Case study 1: Youth Opportunities Program (Uganda)}
\label{sec:uganda_main}

\vspace{-0.1cm}
The Youth Opportunities Program (YOP) \citep{blattman2014generating} is 
an anti-poverty program launched in 2008 in post-conflict Northern Uganda, awarding cash grants to self-organised youth groups, randomized across $535$ groups ($439$ in our analytic sample). The original analysis reports a positive average effect on productive economic capacity, captured through skilled employment and business asset accumulation. We investigate here the follow-up question of \emph{why} and \emph{how} the impact varies among individuals. Full pipeline, quantitative results, ablations, and limitations are presented in Appendix~\ref{sec:uganda}.

\vspace{-0.1cm}
\textbf{Extending the trial with environmental measurements. \ } YOP environmental heterogeneity investigation was first approached by \citet{jerzak2023image}, who paired the trial with satellite imagery, clustering image embeddings in two groups and comparing the Group Average Treatment Effect (GATE) between them. They found preliminary satellite-based heterogeneity signal, but the analysis was limited by the model abstraction capabilities (binary clustering), outdated satellite images (\textasciitilde{}7 years before the program start and 3 bands only), and a single marginal investigation most likely entangling distinct direct modifiers effects. We re-extracted Landsat-7 multispectral imagery from Google Earth Engine over a 2005--2007 composite immediately preceding the program launch, embedded each tile through the Prithvi-EO geospatial foundation model \citep{jakubik2023foundation}, and learned a sparse-autoencoder dictionary on a held-out Uganda national grid disjoint from the trial sites. The resulting candidate pool\footnote{Paired with hand-crafted spectral indices extracted from the same imagery (NDVI, NDWI, MNDWI, NDBI, EVI, BSI).} unions the active dictionary atoms with the structured demographic covariates already in the trial; NEXIS then selects from it conditionally, and each retained atom is labelled post-hoc by contrasting its top and bottom activating tiles with a vision-language model (VLM) \citep{bai2025qwen2}.

\vspace{-11pt}
 \begin{figure}[h]
  \centering
  \begin{subfigure}[t]{0.495\textwidth}
    \centering
    \caption{\footnotesize Latent direct modifiers of \emph{skilled employment}.}\includegraphics[width=\textwidth]{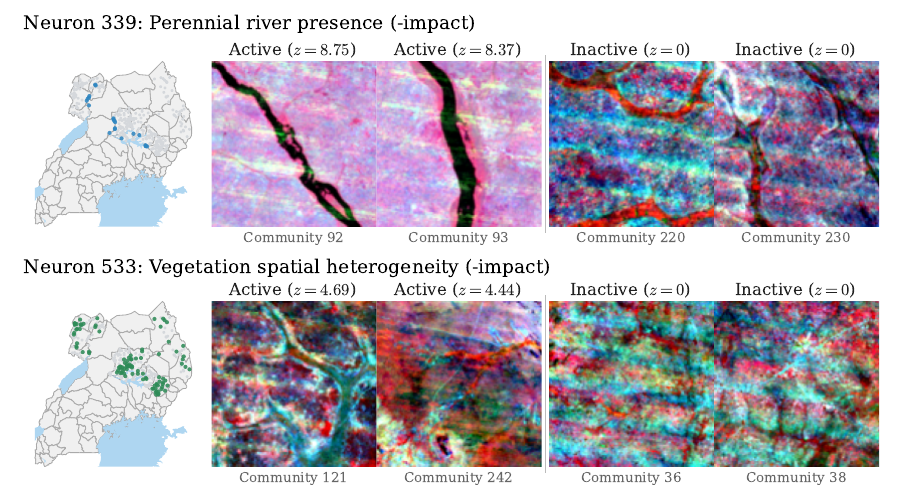}
    \label{fig:uganda_neurons:skilled}
  \end{subfigure}\hfill
  \begin{subfigure}[t]{0.495\textwidth}
    \centering
    \caption{\footnotesize Latent direct modifiers of \emph{business assets}.}
    \includegraphics[width=\textwidth]{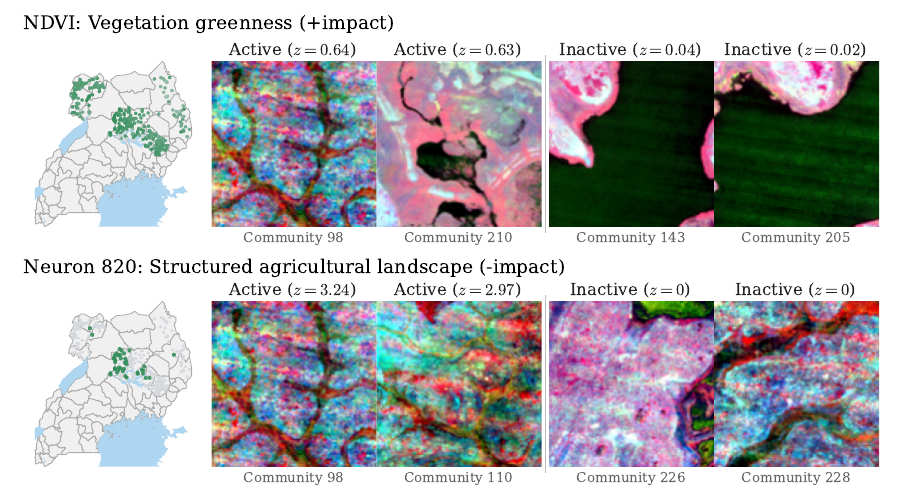}
    \label{fig:uganda_neurons:assets}
  \end{subfigure}
  \vspace{-6mm}
  \caption{\small Satellite-derived discoveries on YOP program impact heterogeneity by NEXIS, with VLM interpretations obtained by contrasting top- and bottom-activating tiles; neurons 339 and 820 are candidates. Each row pairs the Uganda activation map of a feature with its two top- and bottom-activating Landsat tiles in false-color composite (NIR, Green, SWIR).}
  \label{fig:uganda_neurons}
\end{figure}

\vspace{-6pt}
\textbf{Results. \ } 
NEXIS retains two ethnolinguistic groups and one environmental modifier for the first target and NDVI for the second; three additional \textit{candidates} fail the backward step (Figure~\ref{fig:uganda_neurons}, Appendix~\ref{sec:uganda:results}). No individual-level demographic (e.g., age, sex, parental education) is selected.
\begin{enumerate}[label=(\alph*),itemsep=0pt, topsep=-3pt, leftmargin=2.2em, ]
    \item For \textit{skilled employment}, the ethnolinguistic-group modifiers, each aggregating districts sharing a dominant language and broadly reflecting a distinct regional geography, point to post-conflict \textcolor{teal!70!white}{market connectivity} as the key moderator: the two dampening groups (Karamojong and the candidate Lugbara) correspond to conflict-affected or geographically peripheral areas with weak supply-chain integration at trial time, while the amplifying group (Pallisa) lies outside the conflict core with shorter chains to central markets. The two environmental discoveries follow a complementary outside-option logic \citep{roy1951some}: \textcolor{magenta!70!white}{perennial river presence} (a candidate) and \textcolor{magenta!70!white}{vegetation spatial heterogeneity} pick out landscapes where fishing or bush subsistence remains viable, reducing the pull of a new skilled trade.
    \item For \textit{business assets}, regions with higher NDVI, i.e., greater \textcolor{teal!70!white}{vegetation greenness}, exhibit higher program impact: more productive baseline land compounds whatever inputs the grant funds. \textcolor{magenta!70!white}{Structured agricultural landscape}, a candidate, instead dampens returns, plausibly reflecting crowding-out by incumbent commercial-scale agriculture ~\citep{crepon2013labor, egger2022general}.
\end{enumerate}
 None of the environmental modifiers was within reach of the prior approach; they translate into concrete economic theories and adaptations for the next iteration of the program. By contrast, $71$ and $45$ of the $170$ coordinates pass an uncorrected marginal test, mostly correlated proxies of these few direct modifiers: a real-world instance of the experimental power paradox discussed in Section~\ref{sec:power_paradox}.

\subsection{Case study 2: Livelihood Empowerment Against Poverty programme (Ghana)}
\label{sec:ghana_main}
 
\looseness=-1The Livelihood Empowerment Against Poverty (LEAP) 1000 programme \citep{GhanaLEAP1000EvaluationTeam2018} is an anti-poverty program launched in Ghana in 2015 to target early-childhood poverty and stunted nutrition, providing bimonthly cash transfers to extremely poor households with newborns.
So far it has been evaluated on average via a regression discontinuity design 
with difference-in-differences estimation (2015--2017), covering $2{,}331$ households across $162$ communities in Northern and Upper East regions.
We extend here its analysis to impact heterogeneity, focusing on expenditure effect and complementing $24$ survey covariates with novel environmental information from
satellite imagery. For full results, pipeline details and limitations see Appendix~\ref{sec:ghana}.
 
\looseness=-1\textbf{Extending the trial with environmental measurements. \ }
We extract Landsat-8 imagery for 2015 from Google Earth Engine and embed each community tile through Prithvi-EO. We then train an SAE on a held-out Ghana national grid, and pool the active landscape atoms with the spectral indices and the $24$ survey covariates; NEXIS then selects from this pool conditionally, and retained atoms are labelled post-hoc by contrasting their top and bottom activating tiles with a large VLM \citep{bai2025qwen2}.

\begin{figure}[h]
  \vspace{-2mm}
  \centering
  \includegraphics[width=\textwidth]{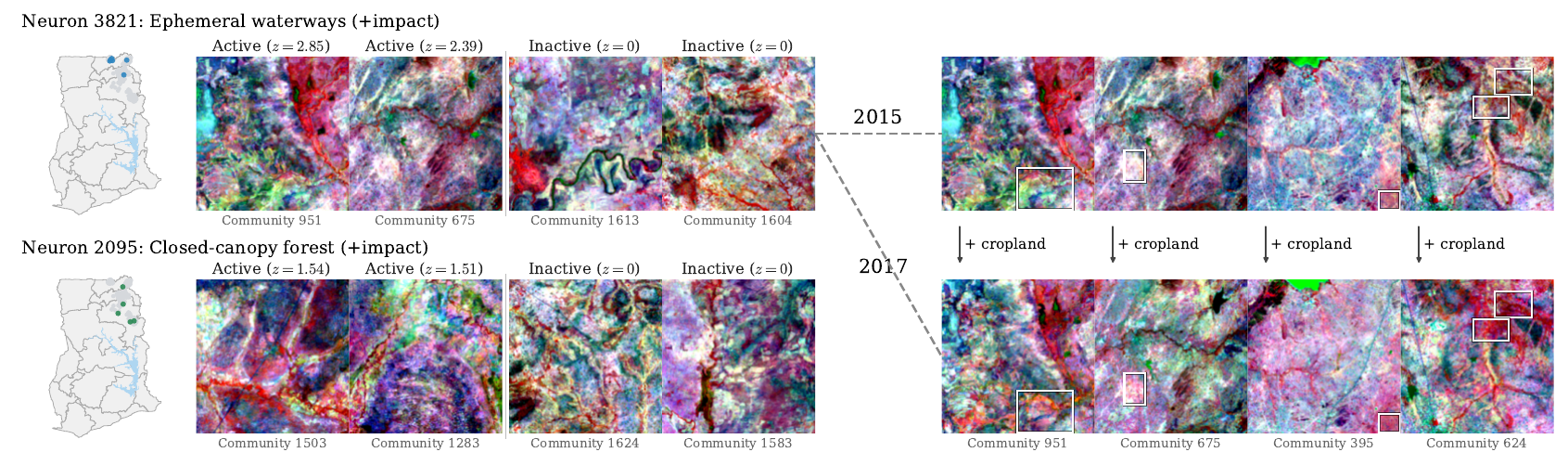}
  \vspace{-5mm}
  \caption{\small \textit{Left}: Satellite-derived discoveries on LEAP1000 programme impact heterogeneity by NEXIS, with VLM interpretations by contrasting top- and bottom-activating Landsat tiles in false-color composite. \textit{Right}: cropland extension in waterway-active communities, 2015--2017 (boxes: largest local NDVI rise, App.~\ref{sec:ghana}).
  }
  \label{fig:ghana_neurons}
  \vspace{-2mm}
\end{figure}
 
\looseness=-1\textbf{Results. \ }
NEXIS retains two satellite coordinates labelled \textcolor{teal!70!white}{ephemeral} \textcolor{teal!70!white}{waterways} and \textcolor{teal!70!white}{closed-canopy forest} (Figure~\ref{fig:ghana_neurons}), and no survey covariate (e.g., household composition, head characteristics), whereas $18$ of $167$ coordinates pass an uncorrected marginal test. In communities with seasonal water corridors, estimated impact is several times the average; a similarly amplified pattern appears in rare forest-patch communities. We hypothesize \textit{complementarity} \citep{de2008returns, prifti2020heterogeneous}: water-adjacent smallholders can invest transfers in irrigation-season agriculture, while forest-edge households can combine cash with non-timber livelihoods. We further support the former hypothesis by tracking cropland expansion 
in VLM comparisons of 2015 and 2017 imagery.

\vspace{-2mm}
\section{Conclusion}
\label{sec:conclusion}
\vspace{-2mm}

\looseness=-1 Identifying causal and interpretable HTE characterizations is essential for trustworthy policy design. We show that this target is within reach by combining richer measurements with learned representations, and introduce the first algorithm for the task. Although we demonstrate it on satellite imagery, the approach applies to any unstructured pre-treatment measurement, such as text, audio or health records, given a valid representation, not necessarily learned by an SAE. Still, NEXIS can only identify the direct modifiers that the measurements capture, and future evaluations should invest in richer, multi-source pre-treatment measurements and representations.
Applications to two anti-poverty programs uncover previously out-of-reach environmental heterogeneity in impact. In both, no individual- or household-level characteristic is selected, suggesting that the heterogeneity lies more in environmental factors than in individual characteristics, although absence of evidence is not evidence of absence.
Structural assumptions determine the target and causal interpretation; valid and consistent tests determine the selection guarantee. When these assumptions cannot be guaranteed, NEXIS discoveries remain exploratory hypotheses until further validated.

\newpage

\subsubsection*{Acknowledgements}
Riccardo Cadei is supported by a Google Research Scholar Award and a Google-initiated gift to Francesco Locatello. Frank Otchere, Nyasha Tirivayi and Gustavo Angeles Tagliaferro are generously funded by the Swedish International Development Cooperation Agency (Sida) as UNICEF Innocenti researchers.
We gratefully acknowledge the Government of Ghana’s Ministry of Gender, Children and Social Protection (MoGCSP), the LEAP Management Secretariat (LMS), the UNICEF Ghana country office, the Institute of Statistical, Social and Economic Research (ISSER), and the Navrongo Health Research Centre (NHRC) for their vital support in implementing and coordinating the evaluation on the LEAP 1000 program, which was generously funded by the United States Agency for International Development (USAID).

\newpage
\bibliographystyle{iclr2027_conference}
\bibliography{refs}

\clearpage
\appendix

\part*{Appendix}
\input{appendix}   

\end{document}

%% file: math_commands.tex
\usepackage{amsmath,amsfonts,bm}

\def\eqref#1{equation~\ref{#1}}

\def\1{\bm{1}}

\DeclareMathAlphabet{\mathsfit}{\encodingdefault}{\sfdefault}{m}{sl}
\SetMathAlphabet{\mathsfit}{bold}{\encodingdefault}{\sfdefault}{bx}{n}

\newcommand{\E}{\mathbb{E}}

\newcommand{\R}{\mathbb{R}}



%% file: appendix.tex
\addcontentsline{toc}{part}{Appendix}

\etocsettocstyle{\section*{Table of contents}}{}
{
\setlength{\parskip}{0pt}
\localtableofcontents
}

\newpage
\section{Proof of Theorem~\ref{thm:nexis_consistency}}
\label{sec:proofs}

We first formally introduce the two background assumptions used in the proof but not explicitly stated in the main text, then restate Theorem~\ref{thm:nexis_consistency} for self-containedness.

\begin{theorembox}
\begin{assumption}[Faithfulness]
\label{ass:faithfulness}
The following two conditions hold. First, the component-wise conditional independences between the direct modifiers and the non-principal coordinates compose jointly:
\begin{equation}
\label{eq:composition}
    W^{\mathrm{dir},k} \indep \bm Z^{[m] \setminus \mathcal{S}^\star} \mid \bm Z^{\mathcal{S}^\star}
    \quad \text{for every } k \in [r]
    \quad\Longrightarrow\quad
    \bm W^{\mathrm{dir}} \indep \bm Z^{[m] \setminus \mathcal{S}^\star} \mid \bm Z^{\mathcal{S}^\star}.
\end{equation}
Second, the CATE is locally mean-adjacency faithful to the principal proxies: for every $j_k \in \mathcal{S}^\star$ and every $A \subseteq [m] \setminus \{j_k\}$, the empty set included,
\begin{equation}
\label{eq:mean_faithfulness}
    \mathbb{P}\!\left(\E[\tau \mid \bm Z^{A \cup \{j_k\}}] \neq \E[\tau \mid \bm Z^{A}]\right) > 0.
\end{equation}
\end{assumption}
\end{theorembox}

Principal Alignment is a screening condition, not a relevance condition: a principal proxy may be polysemantic and its modifier correlated with other coordinates, provided that the principal proxy screens the modifier from the rest of the dictionary. The first faithfulness clause is a restricted composition property, implied by standard DAG faithfulness, that combines these coordinate-wise screening statements \citep{spirtes2000causation}. The second is the local mean-level analogue of adjacency faithfulness used in local causal discovery \citep{aliferis2010local,simionato2024bounding}: no subset of the other dictionary coordinates can render a principal proxy conditionally mean-null, including through exact substitution or cancellation. This uniform form is stronger than marginal relevance and is required because the terminal backward step tests every subset. It is specific to the conditional-mean functional targeted by NEXIS; targeting a different heterogeneity functional would require the corresponding local non-redundancy condition.

\vspace{-0.1cm}
\begin{theorembox}
\begin{assumption}[Test validity and consistency]
\label{ass:test}
For any fixed $S \subseteq [m]$ and $j \in [m] \setminus S$:
\begin{enumerate}[label=(\alph*), leftmargin=2em, itemsep=0pt]
    \item under $H_0(j \mid S)$ (Equation~\ref{eq:cate_ci_null}), the $p$-value $p_j(S)$ is asymptotically valid, i.e., $\limsup_n \mathbb{P}(p_j(S) \le t) \le t$ for every $t \in (0,1)$;
    \item under any fixed alternative to $H_0(j \mid S)$, $\mathbb{P}(p_j(S) < t) \to 1$ as $n \to \infty$ for every $t > 0$.
\end{enumerate}
\end{assumption}
\end{theorembox}

The proof invokes clause (a) only at the fixed set $S = \mathcal{S}^\star$, i.e., for the $m - |\mathcal{S}^\star|$ true nulls $H_0(j \mid \mathcal{S}^\star)$, and clause (b) only for the principal proxies, against $H_0(j_k \mid A)$ for every $A \subseteq [m] \setminus \{j_k\}$; the assumption can be weakened to these instances. We state it for all fixed $S$ because $\mathcal{S}^\star$ is unknown, so in practice a test is justified uniformly; for the linear interaction test, the weaker form needs correct specification only at $\mathcal{S}^\star$.
Clause (a) supplies precision, while clause (b) supplies forward recall and retention of the principal proxies by the terminal backward step; neither argument requires independence among the $p$-values. The assumption is otherwise agnostic to the test instantiation. By default we adopt a linear treatment-interaction test, which satisfies it under a correctly specified conditional linear model and standard moment, nonsingularity, and variance-estimation conditions, and is consistent against alternatives with a nonzero conditional linear interaction. For nonlinear conditional CATEs, the projected covariance measure can target a broader alternative class when its projection learner is sufficiently rich \citep{lundborg2024projected}. Appendix~\ref{sec:method-details:test} states the validity and power conditions of each construction, and Appendix~\ref{sec:semi-synt:method-ablations} reports their ablations.

\newpage
\begin{theorembox}
\begin{theorem*}[Causal Identification]
Given a randomized experiment $\{(\bm z_i,t_i,y_i)\}_{i=1}^n$ with Principally Aligned pre-treatment measurements $\bm z_i\in\R^m$ (Assumption \ref{ass:principal_alignment}), under faithfulness (Assumption \ref{ass:faithfulness}) and tests validity and consistency (Assumption \ref{ass:test}), NEXIS' outcome \(\widehat{\mathcal{S}}_n\) satisfies:
\begin{equation*}
\liminf_n \mathbb{P}(\widehat{\mathcal{S}}_n = \mathcal{S}^\star) \ge 1 - \alpha;
\end{equation*}
and, additionally assuming Measurement Sufficiency (Assumptions \ref{ass:measurement_sufficiency}) and Representation Sufficiency (Assumptions \ref{ass:representation_sufficiency}):
\begin{equation*}
    \tau^{do}(\bm W^{\mathrm{dir}}) = \tau(\bm W^{\mathrm{dir}}) \;=\; \E[\tau \mid \bm Z^{\mathcal S^{\star}}] \quad \text{a.s.}
\end{equation*}
\end{theorem*}
\end{theorembox}

The proof proceeds in three steps:
\begin{enumerate}[label=\roman*., leftmargin=2em, itemsep=2pt, topsep=2pt]
    \item \textbf{Markov-blanket formulation:} Under Principal Alignment and faithfulness, $\mathcal{S}^\star$ is the mean Markov boundary of the CATE on the learned dictionary $[m]$: Lemma~\ref{lem:mb_structure} gives sufficiency using the composition clause, while the local mean-adjacency clause gives non-redundancy given any subset of the other coordinates. \emph{Measurement and Representation Sufficiency are not invoked here:} these are conditions on the dictionary-level target alone.
    \item \textbf{Selection guarantee:} Combined with test validity at fixed conditioning sets, these facts give asymptotic recall of the forward step and retention of $\mathcal{S}^\star$ by the terminal backward step, while every false coordinate the terminal backward step keeps must reject the fixed true null $H_0(j \mid \mathcal{S}^\star)$. This bounds the asymptotic probability of any false discovery by $\alpha$, giving the exact-recovery guarantee.
    \item \textbf{Causal HTE Identification:} Measurement and Representation Sufficiency, paired with the definition of $\bm W^{\mathrm{dir}}$ as the direct effect modifiers, upgrade the structural recovery into a causal-identification statement.
\end{enumerate}

\subsection{Markov-blanket formulation}
\label{sec:proofs:mb}

\begin{theorembox}
\begin{lemma}
\label{lem:mb_structure}
Under Principal Alignment and the composition clause of Assumption~\ref{ass:faithfulness}, the principal proxies are sufficient for the CATE over the dictionary: for every $S \supseteq \mathcal{S}^\star$ (with $S \subseteq[m]$),
    \begin{equation}
    \label{eq:mb_sufficiency}
    \E[\tau \mid \bm Z^{S}] = \E[\tau \mid \bm Z^{\mathcal{S}^\star}] = \E[\tau \mid \bm Z] \quad \text{a.s.}
    \end{equation}
\end{lemma}
\end{theorembox}

\begin{proof}
It is enough to establish the following statement for $R = [m] \setminus \mathcal{S}^\star$, since decomposition then gives it for every $R \subseteq [m] \setminus \mathcal{S}^\star$:
\begin{equation}
\label{eq:tau_ci_zrest}
\tau \;\indep\; \bm Z^{R} \;\Big|\; \bm Z^{\mathcal{S}^\star},
\end{equation}
which directly implies $\E[\tau \mid \bm Z^{\mathcal{S}^\star \cup R}] = \E[\tau \mid \bm Z^{\mathcal{S}^\star}]$ a.s. Equation~\ref{eq:mb_sufficiency} follows by taking $R = [m] \setminus \mathcal{S}^\star$ for the equality to $\E[\tau \mid \bm Z]$, and arbitrary $R \subseteq [m] \setminus \mathcal{S}^\star$ for the intermediate supersets.

Fix any $k \in [r]$. Principal Alignment gives
\begin{equation}
\label{eq:pa_recall}
W^{\mathrm{dir},k} \;\indep\; \bm Z^{[m] \setminus \{j_k\}} \;\Big|\; Z^{j_k},
\end{equation}
which, since $[m] \setminus \{j_k\} = (\mathcal{S}^\star \setminus \{j_k\}) \cup ([m] \setminus \mathcal{S}^\star)$, can be rewritten as $W^{\mathrm{dir},k} \indep (\bm Z^{\mathcal{S}^\star \setminus \{j_k\}}, \bm Z^{[m] \setminus \mathcal{S}^\star}) \mid Z^{j_k}$. The weak union axiom of conditional independence (a graphoid property valid for any probability distribution) yields
\begin{equation}
\label{eq:per_k_independence}
W^{\mathrm{dir},k} \;\indep\; \bm Z^{[m] \setminus \mathcal{S}^\star} \;\Big|\; \bm Z^{\mathcal{S}^\star}.
\end{equation}
Equation~\ref{eq:per_k_independence} holds for each $k \in [r]$ separately. Repeated application of Equation~\ref{eq:composition} combines these per-$k$ statements into the joint
\begin{equation}
\label{eq:joint_wdir_z}
\bm W^{\mathrm{dir}} \;\indep\; \bm Z^{[m] \setminus \mathcal{S}^\star} \;\Big|\; \bm Z^{\mathcal{S}^\star}.
\end{equation}
By the definition of $\bm W^{\mathrm{dir}}$ as a Markov blanket of the treatment-effect heterogeneity (Section~\ref{sec:hte_complex}), $\tau \indep \bm Z \mid \bm W^{\mathrm{dir}}$. Weak union gives $\tau \indep \bm Z^{[m] \setminus \mathcal{S}^\star} \mid (\bm W^{\mathrm{dir}}, \bm Z^{\mathcal{S}^\star})$. Combining this with Equation~\ref{eq:joint_wdir_z} by symmetry, contraction, and decomposition yields Equation~\ref{eq:tau_ci_zrest} for $R = [m] \setminus \mathcal{S}^\star$, and therefore for every subset $R$.
\end{proof}

\textit{Reading Lemma~\ref{lem:mb_structure}.} The lemma says $\mathcal{S}^\star$ captures all the heterogeneity-relevant information $\bm Z$ carries about $\tau$, and adding non-principal coordinates does not destroy sufficiency. Together with the local mean-adjacency clause of Assumption~\ref{ass:faithfulness}, this makes $\mathcal{S}^\star$ the unique mean Markov boundary of the CATE over the dictionary. Indeed, if a sufficient $A$ omitted $j_k$, then conditioning $\E[\tau \mid \bm Z^A]=\E[\tau \mid \bm Z]$ on $\bm Z^{A \cup \{j_k\}}$ would give $\E[\tau \mid \bm Z^{A \cup \{j_k\}}]=\E[\tau \mid \bm Z^A]$, making $H_0(j_k \mid A)$ true and contradicting Equation~\ref{eq:mean_faithfulness}. These facts do not require Measurement or Representation Sufficiency, which determine how much of the true heterogeneity the dictionary captures rather than its internal structure around $\mathcal{S}^\star$.

\subsection{Selection guarantee}
\label{sec:proofs:recovery}

\begin{proof}[Proof of Equation~\ref{eq:selection_consistency}]
Let $\widetilde S_n$ be the selection at which the forward step of Algorithm~\ref{alg:nexis} stops ($\widetilde S$ in Section~\ref{sec:nexis} and Appendix~\ref{sec:method-details:backward}), so that the terminal backward step returns
\begin{equation}
\widehat{\mathcal{S}}_n = \big\{ j \in \widetilde S_n : p_j(A) \le \alpha / m \text{ for every } A \subseteq \widetilde S_n \setminus \{j\} \big\}.
\end{equation}
We prove (i) recall of the forward step, $\mathbb{P}(\mathcal{S}^\star \subseteq \widetilde S_n) \to 1$, then (ii) that the terminal backward step keeps $\mathcal{S}^\star$, then (iii) precision, $\limsup_n \mathbb{P}(\mathcal{S}^\star \subseteq \widetilde S_n,\, \widehat{\mathcal{S}}_n \not\subseteq \mathcal{S}^\star) \le \alpha$, and finally combine.

\medskip
\textit{(i) Recall of the forward step.} Fix any $S \subseteq [m]$ with $\mathcal{S}^\star \not\subseteq S$, let $\bar S = [m] \setminus S$, and pick any $j_k \in \mathcal{S}^\star \setminus S$. The local mean-adjacency clause of Assumption~\ref{ass:faithfulness}, applied with $A = S$, gives that $H_0(j_k \mid S)$ is false at the population level, and Assumption~\ref{ass:test}(b) yields
\begin{equation}
\mathbb{P}\!\big(p_{j_k}(S) < \alpha / |\bar S|\big) \to 1.
\end{equation}
The forward step at selection $S$ admits some $j^\star \in \bar S$ whenever $\min_{j \in \bar S} p_j(S) \le \alpha / |\bar S|$, and the minimum is at most $p_{j_k}(S)$. Define the favorable event $\mathcal{E}_n$ as the intersection of the events $\{\min_{j \in \bar S} p_j(S) \le \alpha / |\bar S|\}$ over all such $S$. This is a finite intersection, so $\mathbb{P}(\mathcal{E}_n) \to 1$. The forward step adds one coordinate per round and never removes one, so it stops after at most $m$ rounds, and on $\mathcal{E}_n$ it cannot stop at a selection that misses a coordinate of $\mathcal{S}^\star$. Hence $\mathbb{P}(\mathcal{S}^\star \subseteq \widetilde S_n) \to 1$.

\medskip
\textit{(ii) The terminal backward step keeps $\mathcal{S}^\star$.} By the local mean-adjacency clause of Assumption~\ref{ass:faithfulness}, $H_0(j_k \mid A)$ is false for every $j_k \in \mathcal{S}^\star$ and every $A \subseteq [m] \setminus \{j_k\}$. Since $m$ is fixed, this is a finite family of fixed alternatives, and Assumption~\ref{ass:test}(b) with a union bound yields
\begin{equation}
\mathbb{P}\!\big(\exists\, j_k \in \mathcal{S}^\star,\ \exists\, A \subseteq [m] \setminus \{j_k\} : p_{j_k}(A) > \alpha / m\big) \to 0.
\end{equation}
The terminal backward step tests each $j_k \in \widetilde S_n$ only at subsets $A \subseteq \widetilde S_n \setminus \{j_k\} \subseteq [m] \setminus \{j_k\}$, and makes all removals simultaneously, so these are the only tests that decide on $j_k$. Hence, with probability $\to 1$, every $j_k \in \mathcal{S}^\star \cap \widetilde S_n$ is kept, and by (i), $\mathbb{P}(\mathcal{S}^\star \subseteq \widehat{\mathcal{S}}_n) \to 1$.

\medskip
\textit{(iii) Precision.} Let $\mathcal{R}_n := \{\mathcal{S}^\star \subseteq \widetilde S_n\}$ and fix $j \in [m] \setminus \mathcal{S}^\star$. On $\mathcal{R}_n$, if $j \in \widehat{\mathcal{S}}_n$, then $j \in \widetilde S_n$ and $\mathcal{S}^\star \subseteq \widetilde S_n \setminus \{j\}$, so $A = \mathcal{S}^\star$ is one of the subsets at which $j$ is tested, and $p_j(\mathcal{S}^\star) \le \alpha / m$. By Lemma~\ref{lem:mb_structure},
\begin{equation}
\E[\tau \mid \bm Z^{\mathcal{S}^\star \cup \{j\}}] = \E[\tau \mid \bm Z^{\mathcal{S}^\star}] \quad \text{a.s.,}
\end{equation}
so $H_0(j \mid \mathcal{S}^\star)$ holds at the population level. Its conditioning set $\mathcal{S}^\star$ is fixed, not selected from the data. A union bound and Assumption~\ref{ass:test}(a) over the $m - |\mathcal{S}^\star|$ coordinates outside $\mathcal{S}^\star$ yield
\begin{equation}
\limsup_n \mathbb{P}\!\big(\mathcal{R}_n,\, \widehat{\mathcal{S}}_n \not\subseteq \mathcal{S}^\star\big) \;\le\; \sum_{j \notin \mathcal{S}^\star} \limsup_n \mathbb{P}\!\big(p_j(\mathcal{S}^\star) \le \alpha / m\big) \;\le\; \frac{m - |\mathcal{S}^\star|}{m}\,\alpha \;\le\; \alpha.
\end{equation}

\medskip
\textit{(iv) Combination.} Since $\widehat{\mathcal{S}}_n \subseteq \widetilde S_n$, the event $\{\mathcal{S}^\star \subseteq \widehat{\mathcal{S}}_n\}$ is contained in $\mathcal{R}_n$, so
\begin{equation}
\mathbb{P}\!\big(\widehat{\mathcal{S}}_n \neq \mathcal{S}^\star\big)
\;\le\; \mathbb{P}\!\big(\mathcal{S}^\star \not\subseteq \widehat{\mathcal{S}}_n\big) + \mathbb{P}\!\big(\mathcal{R}_n,\, \widehat{\mathcal{S}}_n \not\subseteq \mathcal{S}^\star\big).
\end{equation}
By (ii) the first term vanishes, and by (iii) the second has $\limsup$ at most $\alpha$, so $\liminf_n \mathbb{P}(\widehat{\mathcal{S}}_n = \mathcal{S}^\star) \ge 1 - \alpha$. No step uses test validity at a data-dependent conditioning set: (i) and (ii) apply Assumption~\ref{ass:test}(b) to finitely many fixed pairs $(j, A)$, and (iii) applies Assumption~\ref{ass:test}(a) at the fixed set $\mathcal{S}^\star$ only.
\end{proof}

\textit{Role of the corrections.} The Bonferroni gate in the forward step is not needed for the stated error bound: with fixed $m$, test consistency and any fixed positive admission threshold suffice for recall. The correction that supplies precision is the terminal level $\alpha/m$, used in step (iii) to control the $m-|\mathcal{S}^\star|$ fixed nulls $H_0(j \mid \mathcal{S}^\star)$. No additional correction over the exponentially many subset tests is needed for this argument, because a false survivor is bounded using only its test at $\mathcal{S}^\star$. Another terminal rule could replace $\alpha/m$ if it provided the same family-wise bound over these fixed nulls. The fixed-$m$ asymptotic treats the learned dictionary as fixed before the experiment; allowing $m$ to grow with $n$ would require uniform rates for the tests used in the finite-intersection and union-bound arguments above.

\textit{Role of the terminal backward step.} The terminal backward step is what supplies precision. Without it, NEXIS would still satisfy recall by (i), since every $j_k \in \mathcal{S}^\star$ is eventually admitted at some forward step under partial conditioning. But precision would be lost at fixed $\alpha$: a non-principal coordinate $j'$ admitted at some intermediate round, whose conditional null begins to hold only after subsequent additions, would never be re-tested. The terminal backward step re-tests it at every subset of $\widetilde S_n \setminus \{j'\}$, which on $\mathcal{R}_n$ includes the fixed set $\mathcal{S}^\star$, where Assumption~\ref{ass:test}(a) applies. Bounding a false coordinate by its $p$-value at a fixed true conditioning set is the device behind existing error-control bounds for local causal discovery \citep{tsamardinos2008bounding, li2009controlling}.
The interleaved backward step of Appendix~\ref{sec:method-details:backward}, which re-tests each $j$ at $S \setminus \{j\}$ after every forward addition, is not required for the guarantee, and it is not sufficient alone: the current selection $S$ is chosen by the data, so $H_0(j \mid S \setminus \{j\})$ is a data-dependent null, outside the scope of Assumption~\ref{ass:test}, which is the post-selection issue. It remains a finite-sample variant for settings with many or entangled modifiers.

\textit{Role of the forward step.} By the local mean-adjacency clause of Assumption~\ref{ass:faithfulness}, the terminal backward step alone, applied to the whole dictionary $[m]$, would already satisfy steps (ii)--(iv): every $j_k \in \mathcal{S}^\star$ is non-null given any subset, and a false $j$ must reject the fixed true null $H_0(j \mid \mathcal{S}^\star)$. Exhaustive subset testing of this kind underlies family-wise error control in local causal discovery \citep{simionato2024bounding}, but here it would take $m\,2^{m-1}$ tests, conditioning on up to $m-1$ coordinates with $n$ observations. The forward step restricts the terminal step to at most $|\widetilde S_n|\,2^{|\widetilde S_n|-1}$ tests on smaller conditioning sets. In the intended sparse-mechanism regime \citep{scholkopf2021toward}, few direct modifiers imply $|\mathcal{S}^\star| \ll m$, and the forward step is expected to leave $|\widetilde S_n| \ll m$ in practice. This is a computational and finite-sample motivation, not an additional condition for the guarantee. Appendix~\ref{sec:semi-synt:alignment} tests a controlled violation of Principal Alignment, in which one principal coordinate is split in two and the Markov boundary grows to three coordinates; NEXIS still recovers it, with less power at small $n$.

\subsection{Causal Heterogeneous Treatment Effect identification}
\label{sec:proofs:causal}

\begin{proof}[Proof of Equation~\ref{eq:causal_id}]
Chaining Lemma~\ref{lem:mb_structure} with Representation and Measurement Sufficiency:
\begin{equation}
\label{eq:sufficiency_chain}
\E[\tau \mid \bm Z^{\mathcal{S}^\star}] \;=\; \E[\tau \mid \bm Z] \;=\; \E[\tau \mid \bm X] \;=\; \E[\tau \mid \bm W^{\mathrm{dir}}] \quad \text{a.s.}
\end{equation}
By definition, $\bm W^{\mathrm{dir}}$ contains every pre-treatment variable with a direct causal effect on $\tau$ (Section~\ref{sec:hte_complex}); every other pre-treatment variable affects $\tau$, if at all, only through $\bm W^{\mathrm{dir}}$. Hence the joint intervention on $\bm W^{\mathrm{dir}}$ leaves no open back-door path to $\tau$ \citep{pearl2009causality, richardson2013single}, so
\begin{equation}
\label{eq:do_to_cond}
\E[\tau \mid do(\bm W^{\mathrm{dir}})] = \E[\tau \mid \bm W^{\mathrm{dir}}] \quad \text{a.s.}
\end{equation}
Combining Equations~\ref{eq:sufficiency_chain}--\ref{eq:do_to_cond} yields Equation~\ref{eq:causal_id}.
\end{proof}

\textit{Causal interpretation.} The CATE is, in itself, a causal estimand, defined as the expected treatment effect among units sharing the same profile. What is not generally interventional is its characterization by conditioning on a non-direct modifier: indirect modifiers can affect $\tau$ only through $\bm W^{\mathrm{dir}}$, while proxies and common-cause modifiers have no causal pathway to $\tau$. Hence conditioning on a generic non-direct modifier does not in general identify intervening on it, i.e., $\E[\tau \mid do(\cdot)] \neq \E[\tau \mid \cdot]$. Theorem~\ref{thm:nexis_consistency} provides a causal characterization on $\bm Z^{\mathcal{S}^\star}$ at the \emph{joint} level: counterfactual comparisons across bundled configurations of latent direct modifiers are identified, which is the natural object for policy decisions targeting compound conditions (e.g., a community defined jointly by landscape, infrastructure, and demographics).

\textit{From joint to marginal identification.}
Marginal interventional effects on a single direct modifier, $\E[\tau \mid do(W^{\mathrm{dir},k} = w^k)]$, are a strictly stronger target: they require averaging the CATE over the \emph{marginal} distribution of the remaining direct modifiers $\bm W^{\mathrm{dir},-k}$, since intervening on $W^{\mathrm{dir},k}$ severs its inbound dependences without altering the marginal of the others. The natural observational candidate $\E[\tau \mid Z^{j_k}]$ instead averages over the \emph{conditional} distribution of $\bm W^{\mathrm{dir},-k}$ given $Z^{j_k}$. The two coincide under either (i) mutually independent direct modifiers, so that conditional and marginal distributions agree, or (ii) no inter-modifier interactions in $\tau$, in which case the integration measure is irrelevant. Absent these structural conditions, $\E[\tau \mid Z^{j_k} = z]$ admits a valid reading as a \emph{subpopulation-specific} marginal effect for units with $Z^{j_k} = z$, which is useful for targeting within the observational distribution but does not licence population-invariant marginal-intervention claims.

\newpage
\section{Method Details}
\label{sec:method-details}

The main text presents NEXIS in its vanilla form (Algorithm~\ref{alg:nexis}): a forward step followed by a terminal backward step, built around a generic CATE equivalence test, with a Bonferroni-style multiple-testing correction in the forward step. Each of these components admits several natural instantiations, and the abstraction is a feature, not a placeholder: any test family satisfying Assumption~\ref{ass:test} plugs in, and the multiple-testing gate can be exchanged for any procedure controlling the relevant error rate. We use this section to lay out the variants we consider, all of which are tested in the ablations of Appendix~\ref{sec:semi-synt:method-ablations}, and to provide a minimal Python implementation. The four axes we cover are (i)~the CATE equivalence test (Appendix~\ref{sec:method-details:test}), (ii)~the multiple-testing correction in the forward step (Appendix~\ref{sec:method-details:error-control}), (iii)~the spectral gap \(\rho\), which is not part of Algorithm~\ref{alg:nexis} but is our practical default against the residual entanglement of real dictionaries (Appendix~\ref{sec:method-details:rho}), and (iv)~the backward steps, i.e., the terminal backward step and the optional interleaved backward step (Appendix~\ref{sec:method-details:backward}). The implementation snippet of Appendix~\ref{sec:method-details:code} uses the simplest combination, the linear test with Bonferroni correction in the forward step followed by the terminal backward step, which is also the default we recommend for routine use based on the controlled experiments, together with \(\rho = 0.5\).

\subsection{CATE-equivalence test}
\label{sec:method-details:test}

Each NEXIS search step is built around a CATE equivalence test: the procedure asks whether a candidate coordinate \(j \in [m]\) carries residual heterogeneity beyond what the current selection \(S \subseteq [m]\setminus\{j\}\) already explains, i.e.\ whether
\begin{equation}
\label{eq:cate_ci_null_recall}
H_0(j \mid S):\quad \E[\tau \mid \bm Z^{S \cup \{j\}}] = \E[\tau \mid \bm Z^{S}] \quad \text{a.s.}
\end{equation}
Theorem~\ref{thm:nexis_consistency} only requires that the resulting \(p\)-value \(p_j(S)\) be asymptotically valid under the null and consistent against fixed alternatives (Assumption~\ref{ass:test}); any such test plugs in. We consider three test families, spanning the parametric--nonparametric spectrum: the linear test, and the doubly-robust GCM and PCM, each of the latter two with a quadratic and a LightGBM variant.

\paragraph{(i) Linear treatment-interaction test (default).}
The simplest instantiation fits a linear treatment-interaction working model
\begin{equation}
\label{eq:linear_test}
Y \;=\; \alpha + T \beta_T + \bm Z^S \bm\gamma_S + T \cdot \bm Z^S \bm\delta_S + Z^j \beta_j + T \cdot Z^j \delta_j + \varepsilon,
\end{equation}
and tests \(H_0(j \mid S)\) by a \(t\)-test on the interaction coefficient \(\delta_j\). It requires no nuisance estimation, no cross-fitting, and runs in milliseconds per test. Under correctly-specified linearity, Assumption~\ref{ass:test} is satisfied. The price is consistency only against alternatives within the linear-in-\(\bm Z\) class: heterogeneity that is genuinely nonlinear in \(Z^j\) given \(\bm Z^S\) (a threshold, a U-shape) can be missed at any sample size. What matters is that the conditional CATE \(\E[\tau \mid \bm Z^{S \cup \{j\}}]\) be close to linear in the coordinates. This is a property of the data-generating process that a dictionary can expose but not create: when \(\tau\) is linear in the direct modifiers and the dictionary aligns them to near-indicator coordinates, the linear test is well specified. The controlled CelebA experiments are in this regime: the linear test reaches recall \(0.95\) at \(n = 750\) when \(\eta = 5\) and at \(\eta = 2\) when \(n = 2000\), against \(n = 2000\) and \(\eta = 4\) for both GCM variants below, and \(n = 3500\) for both PCM variants, which never reach it at \(n = 2000\) (see Figure~\ref{fig:celeba:test} and the discussion in Appendix~\ref{sec:semi-synt:method-ablations}). We adopt it as the default throughout.

\paragraph{(ii) Doubly-robust GCM.}
When linearity is unwarranted, we replace Equation~\ref{eq:linear_test} with a doubly-robust pseudo-outcome plus a residual-product test in the spirit of \citet{shah2020hardness}. The R-learner residualization \citep{robinson1988root, nie2021quasi}
\begin{equation}
\label{eq:r_learner}
\widehat\varphi \;=\; \frac{(Y - \widehat m(\bm Z^S))(T - e)}{e(1-e)},
\end{equation}
with \(e := \mathbb{P}(T = 1)\) known by design and \(\widehat m\) a cross-fitted estimator of \(\E[Y \mid \bm Z^S]\), yields a pseudo-outcome with \(\E[\widehat\varphi \mid \bm Z] = \E[\tau \mid \bm Z]\). The Generalised Covariance Measure (GCM) then regresses \(\widehat\varphi\) on \(\bm Z^S\) and \(Z^j\) on \(\bm Z^S\), forms the residual product
\begin{equation}
\label{eq:gcm_residual}
R(j \mid S) \;:=\; \big(\widehat\varphi - \widehat\E[\widehat\varphi \mid \bm Z^S]\big) \cdot \big(Z^j - \widehat\E[Z^j \mid \bm Z^S]\big),
\end{equation}
and tests \(\E[R(j\mid S)] = 0\) via the normalized \(z\)-statistic
\begin{equation}
\label{eq:gcm_stat}
T_n(j \mid S) \;=\; \frac{\sqrt{n}\,\bar{R}_n(j \mid S)}{\widehat\sigma_n(j \mid S)}.
\end{equation}
The two variants differ only in the outcome-side nuisances \(\widehat m = \E[Y\mid\bm Z^S]\) and \(\widehat\E[\widehat\varphi\mid\bm Z^S]\), both cross-fitted by \(K\)-fold splitting to keep the nuisance estimate independent of the residual at each unit: degree-2 polynomial (Ridge) regressions in the quadratic variant, gradient-boosted trees \citep{ke2017lightgbm} in the LightGBM variant. In both, the \(Z^j\)-on-\(\bm Z^S\) residualization is kept linear and computed in a single vectorized pass over all candidates. Provided \(\E[Z^j\mid\bm Z^S]\) is well-approximated by its linear projection, \(T_n(j\mid S)\) is asymptotically standard normal under the null and consistent against alternatives whose conditional-mean heterogeneity covaries linearly with \(Z^j\) given \(\bm Z^S\). The quadratic variant adds curvature while keeping the model auditable; the LightGBM variant is the most flexible in the nuisances and the most expensive, since each test fits two boosting models. Neither variant removes the blind spot of the linear test: the GCM targets the expected conditional covariance of \(\widehat\varphi\) and \(Z^j\) given \(\bm Z^S\), which vanishes under \(H_0(j \mid S)\) but also under any alternative whose conditional-mean contrast is uncorrelated with \(Z^j\), e.g., a U-shape, and such alternatives are missed at any sample size.

\paragraph{(iii) Projected covariance measure (PCM).}
We use a variant of the PCM of \citet{lundborg2024projected}, which tests the conditional mean independence of an observed response from a covariate given a conditioning set. Our variant replaces the unobserved \(\tau\) by the pseudo-outcome: since \(\E[\widehat\varphi \mid \bm Z] = \E[\tau \mid \bm Z]\), \(H_0(j \mid S)\) is equivalent to \(\E[\widehat\varphi \mid \bm Z^{S \cup \{j\}}] = \E[\widehat\varphi \mid \bm Z^{S}]\), the PCM null with response \(\widehat\varphi\), covariate \(Z^j\) and conditioning set \(\bm Z^S\). The sample is split into two halves, with \(\widehat m\) cross-fitted within each half. On the first half, a projection \(\widehat f_j\) of the conditional-mean contrast \(\E[\widehat\varphi \mid \bm Z^{S \cup \{j\}}] - \E[\widehat\varphi \mid \bm Z^{S}]\) is learned, without the conditional-variance weighting of the original, and rescaled to unit variance. On the second half, the residual product
\begin{equation}
\label{eq:pcm_residual}
R^{\mathrm{PCM}}(j \mid S) \;:=\; \big(\widehat\varphi - \widehat\E[\widehat\varphi \mid \bm Z^S]\big) \cdot \big(\widehat f_j - \widehat\E[\widehat f_j \mid \bm Z^S]\big)
\end{equation}
is formed and studentized as in Equation~\ref{eq:gcm_stat}. Because \(\widehat f_j\) is learned on independent data, the statistic is asymptotically standard normal under \(H_0(j \mid S)\) however poor the projection, provided the regressions on \(\bm Z^S\) are accurate enough \citep[Theorem~4]{lundborg2024projected}. Under an alternative, \(R^{\mathrm{PCM}}\) has positive mean whenever the residualized projection is positively correlated with the contrast, so the test is one-sided, with power tending to one against the alternatives where this correlation holds with high probability \citep[Theorem~5]{lundborg2024projected}, including the U-shapes invisible to the GCM. The price is sample splitting: the projection only sees half of the data, and the statistic is evaluated on the other half. The two halves then swap roles, and the two one-sided \(p\)-values are combined as \(\min\{1, 2\min(p_1, p_2)\}\), a Bonferroni combination that is valid whatever their dependence; \citet{lundborg2024projected} instead average the statistic over repeated splits. The quadratic variant fits \(\widehat f_j\) by regressing \(\widehat\varphi\) on \((Z^j, (Z^j)^2)\) after partialling out \([1, \bm Z^S]\), vectorized over all candidates, and residualizes it on \(\bm Z^S\) linearly. The LightGBM variant additionally refits \(\E[\widehat\varphi \mid Z^j, \bm Z^S]\) with gradient-boosted trees for the \(32\) candidates with the largest explained sum of squares on the first half, and residualizes these fits on \(\bm Z^S\) with the same learner. Both variants keep quadratic nuisances for \(\widehat m\) and \(\widehat\E[\widehat\varphi \mid \bm Z^S]\).

\paragraph{Trade-offs and default.}
The three families trade power for the class of alternatives they see. Linear is fast, easy to implement, and the most efficient when the conditional CATE is approximately linear in \(\bm Z\), as in our controlled benchmark; we adopt it as the default. The GCM is a compromise: more expressive than the linear test, since it is robust to misspecified nuisances, and more efficient than the PCM when the heterogeneity is monotone. However, it estimates \(\E[\mathrm{Cov}(\widehat\varphi, Z^j \mid \bm Z^S)]\), which vanishes under \(H_0(j \mid S)\) but also under any alternative where \(\E[\tau \mid \bm Z^{S \cup \{j\}}]\) varies with \(Z^j\) while being uncorrelated with \(Z^j\) given \(\bm Z^S\), e.g., a U-shape: against such alternatives it has no power at any sample size \citep{lundborg2024projected}. The PCM targets conditional mean independence itself and is consistent against any alternative its projection can learn; the price is the sample split. On the CelebA benchmark, both assumption-lean families pay a measurable cost in power when the CATE is linear, and only the PCM recovers the modifiers of a U-shaped CATE (Figure~\ref{fig:celeba:test}, Table~\ref{tab:celeba_test_story}). We therefore use the linear test by default, the GCM when the heterogeneity may be nonlinear but monotone, and the PCM when non-monotone heterogeneity cannot be ruled out and enough data are available. Switching test is a one-line change in the implementation of Appendix~\ref{sec:method-details:code}; all other components of NEXIS are unchanged.

\subsection{Error control}
\label{sec:method-details:error-control}

The forward step accepts the most significant candidate, provided its \(p\)-value passes a threshold, and that threshold decides how many candidates enter \(\widetilde S\). The terminal backward step keeps its own level \(\alpha/m\) whatever the forward correction (Appendix~\ref{sec:method-details:backward}); the interleaved backward step, when used, applies the forward correction against \(|S|\). We consider three options.

\paragraph{(i) No correction.}
The simplest option compares each \(p_j(S)\) directly to \(\alpha\). This makes no adjustment for the fact that the forward step screens \(|\bar S|\) candidates per round, and at large effect sizes several entangled companions of the principals can cross the gate and enter \(\widetilde S\), leaving the terminal backward step to remove them. The semi-synthetic experiments (Figure~\ref{fig:celeba:adjust}) show that it does: precision reaches \(0.95\) almost as early as with the corrected variants, but recall comes later (\(n = 1000\) instead of \(n = 750\) at \(\eta = 5\), and \(\eta = 8\) instead of \(\eta = 7\) at \(n = 500\)). The unadjusted gate is therefore not a recall-friendly choice in disguise: it delays recall.

\paragraph{(ii) FDR via Benjamini--Hochberg.}
Under its usual independence or positive-dependence conditions, the Benjamini--Hochberg procedure \citep{benjamini1995controlling} controls the expected proportion of false discoveries among rejections at level \(\alpha\). At each forward step, candidate \(p\)-values are ordered \(p_{(1)} \le \cdots \le p_{(|\bar S|)}\) and the largest \(k\) such that \(p_{(k)} \le k\alpha/|\bar S|\) is identified; the candidates with \(p\)-value at most \(p_{(k)}\) are eligible to enter \(S\). If this rejection set is nonempty, NEXIS admits its most significant candidate. Although only one coordinate enters, this is not generally equivalent to the Bonferroni rule: several jointly small $p$-values can make $k>1$ and admit a minimum above $\alpha/|\bar S|$. The interleaved backward step, when used, treats its batch of retained coordinates symmetrically. FDR is the natural choice when \(|\mathcal{S}^\star|\) is expected to be moderate to large: it pays a smaller multiplicity tax than FWER as the truth set grows.

\paragraph{(iii) FWER via Bonferroni.}
Bonferroni \citep{bonferroni1936teoria} controls the family-wise error rate by comparing each \(p\)-value against the conservative threshold \(\alpha/|\bar S|\) (or \(\alpha/|S|\) in the interleaved backward step). It is the most stringent of the three. The selection guarantee does not require a Bonferroni correction in the forward step: with fixed $m$, all three gates above admit every remaining principal proxy with probability tending to one by test consistency. Precision instead comes from the terminal level \(\alpha/m\), which corrects over the dictionary coordinates; no correction over all tested subsets is needed because the proof bounds each false survivor only at the fixed set $\mathcal{S}^\star$ (Appendix~\ref{sec:proofs:recovery}).

\paragraph{Trade-offs and default.}
Empirically (Figure~\ref{fig:celeba:adjust}) FDR and FWER are essentially indistinguishable in our setting because the Benjamini--Hochberg rejection set is effectively a singleton at the leading discoveries. We therefore adopt FWER as the more conservative default. In larger truth-set regimes FDR may admit several jointly significant candidates and be the better operating point; the swap is again a one-line change.

\subsection{Spectral-gap gate}
\label{sec:method-details:rho}

Principal Alignment (Assumption~\ref{ass:principal_alignment}) asks each direct modifier to be summarized by one principal coordinate that leaves no residual modifier information in the others after conditioning on it. This is a population-level idealisation. Real learned dictionaries exhibit \emph{partial} entanglement \citep{chanin2024absorption}: a proxy coordinate \(j' \neq j_k\) correlated with the principal coordinate \(j_k\) of \(W^{\mathrm{dir},k}\) can retain a small residual interaction signal even after \(j_k\) has entered \(S\). At large \(n\), this residual can squeak through the conditional gate without the terminal backward step always catching it, because experimental power amplifies finite-resolution misalignment in exactly the same way it amplifies genuine signals.

To guard against this residual signal, we add a relative-magnitude check to the existing \(p\)-value gate. For $S\neq\emptyset$, let
\begin{equation}
\label{eq:rho_gate}
|T_n(j^\star \mid S)| \;\ge\; \rho \cdot \min_{j \in S} |T_n(j \mid S_{<j})|, \qquad \rho \in [0, 1],
\end{equation}
where \(T_n\) is the statistic of the chosen CATE-equivalence test and \(S_{<j}\) is the selection at the step at which \(j\) entered, so each selected coordinate keeps the statistic it was admitted with; the first candidate is governed only by the $p$-value gate. For $\rho>0$, a later candidate \(j^\star\) is retained only if its conditional test statistic is within a factor \(\rho^{-1}\) of the \emph{weakest} coordinate already in \(S\), while $\rho=0$ disables the check. The intuition is residual-by-construction: principal modifiers with comparable detectability produce comparable standardized signals; a residual proxy of an already-selected principal should produce a weaker one, with the gap increasing as alignment improves. In this comparison the common \(\sqrt n\) factor cancels, but the coordinate-specific standard-error estimates generally do not. Thus $\rho$ compares standardized conditional signal strengths, which combine CATE contrast, coordinate scale, and residual variability; it is not in general a ratio of raw CATE contrasts. When the standard errors are comparable, the ratio approximately tracks relative contrast magnitudes. For example, $\rho=0.5$ requires a new candidate's conditional test statistic to be at least half that of the weakest selected coordinate. This is a ratio of standardized signals, not of evidence on the $p$-value scale: halving a $t$-statistic from $4$ to $2$ moves a two-sided $p$-value from about $6{\times}10^{-5}$ to about $0.05$.

\paragraph{What the gate costs.}
The gate buys precision at a price in recall. Theorem~\ref{thm:nexis_consistency} is stated for Algorithm~\ref{alg:nexis}, i.e.\ for \(\rho = 0\). With \(\rho > 0\), the gate can only stop the forward step earlier, so the rest of the argument is unchanged, but its recall step needs one more condition: whenever a principal coordinate \(j_k\) is the best remaining candidate, it must also pass the gate, i.e.\ its standardized conditional signal must be at least \(\rho\) times that of the weakest coordinate already selected. Since the standardized statistics grow as \(\sqrt n\) times a coordinate-specific population signal, this condition compares population quantities and does not become easier with more data: a principal coordinate whose population signal is below \(\rho\) times that of an earlier selected coordinate is blocked at every \(n\). The value of \(\rho\) therefore states how unequal the detectability of the direct modifiers is allowed to be. With \(\rho = 0.5\), no direct modifier may have a population standardized signal below half that of the weakest one found before it; \(\rho \to 1\) requires all of them to be equally detectable, and \(\rho = 0\) imposes nothing.

\paragraph{Default and sensitivity.}
We default to \(\rho = 0.5\): it admits direct modifiers of comparable strength, and rejects residual proxies of an already-selected principal coordinate, whose conditional signal should be much weaker than the principal's. The CelebA ablation (Figure~\ref{fig:celeba:rho}, Appendix~\ref{sec:semi-synt:method-ablations}) shows the trade-off: lower \(\rho\) admits correlated companions of the principal coordinates and erodes precision, higher \(\rho\) blocks weaker principal coordinates and erodes recall, and \(\rho = 0.5\) attains both. On a new application, a sweep over \(\rho \in \{0.2, 0.5, 0.8\}\) is a useful diagnostic: stable selection across it indicates that the gate is not binding and that \(\widehat{\mathcal{S}}_n\) is driven by the conditional-independence structure rather than by the gate. Setting \(\rho = 0\) disables the gate, recovers Algorithm~\ref{alg:nexis}, and is appropriate when the dictionary has been verified to be near-orthogonal on a held-out grid.

\subsection{Backward steps}
\label{sec:method-details:backward}

NEXIS uses two kinds of backward step. The \textit{terminal backward step}, the last step of Algorithm~\ref{alg:nexis}, runs once, after the forward step has stopped at a selection \(\widetilde S\): it keeps \(j \in \widetilde S\) only if \(p_j(A) \le \alpha/m\) for every \(A \subseteq \widetilde S \setminus \{j\}\), the empty set included, and makes all removals simultaneously. The \textit{interleaved backward step} is an optional variant: at every round, after the forward step, it re-tests each retained coordinate \(j\) against the rest of the current selection \(S\), and removes it if \(p_j(S \setminus \{j\}) > \alpha/|S|\).

\paragraph{(i) Terminal backward step.}
This step delivers the precision part of Theorem~\ref{thm:nexis_consistency}, without requiring test validity at data-dependent conditioning sets. On the recall event \(\mathcal{S}^\star \subseteq \widetilde S\), the subsets tested for a coordinate \(j' \in \widetilde S \setminus \mathcal{S}^\star\) include \(A = \mathcal{S}^\star\). The null \(H_0(j' \mid \mathcal{S}^\star)\) holds by Lemma~\ref{lem:mb_structure} and does not depend on the data, so Assumption~\ref{ass:test}(a) applies to it directly, and \(j'\) survives only if \(p_{j'}(\mathcal{S}^\star) \le \alpha/m\). A union bound over the at most \(m\) coordinates outside \(\mathcal{S}^\star\) bounds the probability that any of them survives by \(\alpha + o(1)\). The price is the number of tests: at most \(|\widetilde S| \cdot 2^{|\widetilde S| - 1}\), exponential in \(|\widetilde S|\) and determined by the size reached by the forward step rather than directly by \(m\). This is feasible when the treatment effect has few direct modifiers, so that the forward step stops at a small \(\widetilde S\). We motivate this practical regime by analogy with the sparse mechanism shift principle of causal representation learning \citep{scholkopf2021toward}, which posits that small distribution changes tend to affect only a few factors of the causal factorization, not all of them simultaneously: we likewise expect the heterogeneity of a treatment effect to run through few direct modifiers. The measured sizes agree: over the \(6{,}300\) CelebA runs of the three dictionaries at \(\rho = 0.5\), \(|\widetilde S|\) has median \(2\) and maximum \(7\), and the forward step alone gives the same median and maximum on the main setting. Since a coordinate is dropped at its first failing subset, trying the largest subsets first, the same runs needed a median of \(4\) and a maximum of \(192\) tests, against the bound \(7 \cdot 2^6 = 448\). Since the terminal backward step runs at level \(\alpha/m\), discarding the coordinates rarely active in the sample before running NEXIS, as in Appendices~\ref{sec:uganda} and~\ref{sec:ghana}, raises this level, and the guarantee applies to the retained coordinates, with \(m\) their number, because the screen uses only \(\bm Z\), not the treatment or the outcome.

\paragraph{(ii) Interleaved backward step.}
The forward step followed by the terminal backward step mirrors the grow and shrink phases of IAMB \citep{tsamardinos2003algorithms}; interleaving a backward step after each forward addition mirrors Inter-IAMB \citep{tsamardinos2003algorithms}, and HITON interleaves in the same way \citep{aliferis2010local}. The interleaved step is not required asymptotically: recall follows from the forward step and precision from the terminal backward step. It can help in finite samples, when a coordinate that mixes several modifiers enters before their principal coordinates and becomes redundant once they are in. Removing it keeps \(\widetilde S\) small, which cuts the cost of the terminal step and spares the principal coordinates the extra subset tests that include it. The interleaved step alone, however, is not theoretically sufficient. It tests \(j\) at \(S \setminus \{j\}\), where the current selection \(S\) is itself chosen by the data, so \(H_0(j \mid S \setminus \{j\})\) is a data-dependent null: Assumption~\ref{ass:test} covers fixed conditioning sets only, and the union bound behind the precision guarantee does not apply. This post-selection issue is what the terminal backward step resolves.

\paragraph{Trade-offs and default.}
In the backward-step ablation of Appendix~\ref{sec:semi-synt:method-ablations} (Figure~\ref{fig:celeba:backward}), the terminal backward step raises precision for a small delay in recall, while the interleaved backward step changes no run: with two well-aligned modifiers on sparse SAE codes, no mixing coordinate enters ahead of the principals. The terminal step is therefore the default; we keep the interleaved step as an option for settings with more, or more entangled, modifiers.

\subsection{Minimal Python Implementation Snippet}
\label{sec:method-details:code}

We provide a self-contained Python implementation of NEXIS as in Algorithm~\ref{alg:nexis}: the linear treatment-interaction test of Equation~\ref{eq:linear_test}, the forward step with Bonferroni multiple-testing correction, and the terminal backward step, with no spectral-gap gate (\(\rho = 0\)) and no interleaved backward step. The implementation relies on \texttt{pandas} \citep{mckinney2011pandas} for tabular operations, \texttt{NumPy} \citep{harris2020array} and \texttt{SciPy} \citep{virtanen2020scipy} for numerical routines, and \texttt{statsmodels} \citep{seabold2010statsmodels} for the OLS interaction \(t\)-test. Switching to a GCM or PCM variant of Appendix~\ref{sec:method-details:test}, swapping Bonferroni for Benjamini--Hochberg, enabling the spectral-gap gate of Appendix~\ref{sec:method-details:rho}, or adding the interleaved backward step of Appendix~\ref{sec:method-details:backward} is a localised modification of the routines below.

\paragraph{Neural Exposure Interaction Search.}
The forward step admits the most significant candidate until none passes the Bonferroni gate, and collects the selected coordinates in \texttt{S\_candidate} (the \(\widetilde S\) of Appendix~\ref{sec:method-details:backward}). The terminal backward step then filters \texttt{S\_candidate} into the output \texttt{S}: it keeps each candidate only if it passes the gate \(\alpha/m\) given every subset of the other candidates, largest subsets first, with all removals made simultaneously. The subset generator is lazy and \texttt{all} stops at the first failing subset, while \texttt{S\_candidate} remains unchanged until every candidate is assessed; thus the implementation matches Algorithm~\ref{alg:nexis} and uses at most $|\widetilde S|2^{|\widetilde S|-1}$ terminal tests. The procedure is parameterized by a generic CATE equivalence test \texttt{cci\_test}, which returns a \(p\)-value per candidate.
\begin{lstlisting}[language=Python]
from itertools import combinations

def NEXIS(T, Y, Z, alpha=0.05, cci_test=cci_test_linear):
    m, S_candidate = Z.shape[1], []
    def p(A, j):  # p-value of H0(j | A)
        return cci_test(T, Y, Z, A, [j]).loc[j, "p_value"]
    # forward step with Bonferroni gate
    while len(S_candidate) < m:
        bar_S = [c for c in Z.columns if c not in S_candidate]
        tests = cci_test(T, Y, Z, S_candidate, candidates=bar_S)
        j_star = tests["p_value"].idxmin()
        if tests.loc[j_star, "p_value"] > alpha / len(bar_S):
            break
        S_candidate.append(j_star)
    # terminal backward step over every subset A of S_candidate\{j}
    S = []
    for j in S_candidate:
        rest = [c for c in S_candidate if c != j]
        subsets = (list(A) for r in range(len(rest), -1, -1)
                   for A in combinations(rest, r))
        if all(p(A, j) <= alpha / m for A in subsets):
            S.append(j)
    return S
\end{lstlisting}

\paragraph{Linear treatment-interaction test (default).}
Fits Equation~\ref{eq:linear_test} via OLS and tests the candidate's treatment interaction \(\delta_j\) via a \(t\)-test.
\begin{lstlisting}[language=Python]
import pandas as pd
import statsmodels.api as sm

def cci_test_linear(T, Y, Z, S, candidates):
    df = Z[S].copy() if S else pd.DataFrame(index=Z.index)
    df["_T"] = T.astype(int)
    for s in S:
        df[f"_T_x_{s}"] = df["_T"] * df[s]
    out = []
    for j in candidates:
        df_j = df.copy()
        df_j[j] = Z[j]
        df_j[f"_T_x_{j}"] = df_j["_T"] * df_j[j]
        model = sm.OLS(Y, sm.add_constant(df_j)).fit()
        t_test = model.t_test(f"_T_x_{j} = 0")
        out.append((j, float(t_test.pvalue)))
    return pd.DataFrame(out, columns=["candidate", "p_value"]).set_index("candidate")
\end{lstlisting}

\newpage
\section{Semi-synthetic experiments (CelebA)}
\label{sec:semi-synt}

This section reports the controlled empirical validation of NEXIS announced in Section~\ref{sec:repr_learning} and Theorem~\ref{thm:nexis_consistency}. The goal is to verify, on a setting where the ground-truth direct-modifier set $\mathcal{S}^\star$ is known by construction, that NEXIS recovers it under the regime predicted by the theory and that the practical-guidance defaults of Appendix~\ref{sec:method-details} hold up across DGP conditions and design choices. We use CelebA \citep{liu2015deep} as a high-dimensional, visually realistic source of pre-treatment observations $\bm X$, embed each image with a foundation model, learn a sparse-autoencoder dictionary on the embeddings, and inject a known treatment-modification structure on top of two fixed binary attributes. Because the ground-truth modifiers are fixed but \emph{not} given to NEXIS, the evaluation tests both Principal Alignment of the SAE dictionary and the conditional-selection behaviour of the algorithm itself.

\subsection{Experimental setup}
\label{sec:semi-synt:setup}

\paragraph{Foundation-model encoder.}
We embed the $19{,}867$ images of the official CelebA validation split (aligned and cropped) with \textbf{SigLIP~2} \citep{tschannen2025siglip}, a contrastively pretrained ViT-B/16 (timm \texttt{vit\_base\_patch16\_siglip\_224}, weights \texttt{v2\_webli}), at $224$\,px, taking the $196$ final-layer patch tokens ($768$-dimensional) and not using the attention-pooling head. These images are used for SAE training, for the ground-truth labels below, and as the pool from which the benchmark samples its units. All embeddings are pre-computed once and cached, decoupling the SAE step from the (expensive) encoder forward pass.

\paragraph{Sparse autoencoder.}
We train a \textbf{TopK SAE} \citep{gao2025scaling} of hidden width $m = 9{,}216$ ($= 12 \times 768$) on the individual SigLIP~2 patch tokens of the same images ($19{,}867 \times 196 \approx 3.89 \times 10^6$ training vectors per epoch), following the protocol of \citet{mencattini2025exploratory}: a linear encoder followed by ReLU and TopK, unit-norm decoder rows, mean-squared reconstruction loss, Adam with learning rate $5 \times 10^{-4}$ and gradient-norm clipping at $1$, $20$ epochs, batches of $20$ images, initialization seed $0$. The SAE is trained without labels, treatments or outcomes, but on the same images the benchmark samples from; we check reproducibility with the $k = 20$ SAE retrained on an independent sample of $19{,}867$ training-split images that shares no identities with the validation split (Appendix~\ref{sec:semi-synt:reproducibility}).
Two sparsity levels are evaluated, $k = 5$ and $k = 20$ active features. Each image is encoded by applying the trained encoder to the mean $\bar{\bm h}$ of its $196$ patch tokens, and for each level we expose two views of the representation:
\begin{itemize}[leftmargin=1.4em, itemsep=1pt, topsep=2pt]
    \item $\bm Z = \mathrm{TopK}(\mathrm{ReLU}(\bm W_{\mathrm{enc}} \bar{\bm h} + \bm b_{\mathrm{enc}}))$, the sparse \emph{post}-TopK codes (near-orthogonal, $L_0 \le k$);
    \item $\bm Z_{\mathrm{pre}} = \bm W_{\mathrm{enc}} \bar{\bm h} + \bm b_{\mathrm{enc}}$, the dense, continuous linear \emph{pre}-activations before ReLU and TopK (correlated, no sparsity constraint).
\end{itemize}
The main reported setting is $k = 20$ on $\bm Z$; the alternatives ($k = 5$ on $\bm Z$, and $k = 20$ on $\bm Z_{\mathrm{pre}}$) are reported as ablations in Appendix~\ref{sec:semi-synt:model-ablations}.

\paragraph{Data-generating process.}
We fix two CelebA attributes, $W_1 = $\textsc{Wearing\_Hat} ($\sim 4.7\%$ prevalence) and $W_2 = $\textsc{Eyeglasses} ($\sim 7.0\%$ prevalence), as the latent direct effect modifiers. For each unit $i$ we draw \textit{i.i.d.} from  $W_{1} \sim \mathrm{Bernoulli}(\hat p_{W_1})$, $W_{2} \sim \mathrm{Bernoulli}(\hat p_{W_2})$, $T \sim \mathrm{Bernoulli}(0.5)$, and sample without replacement from the CelebA bucket an image $\bm x_i$, matching $(w_{1,i}, w_{2,i})$ and pre-compute its representation $\bm z_i = \psi(\bm x_i)$. The outcome is sampled from: 
\begin{equation}
\label{eq:celeba_dgp}
Y \;=\; \beta_{W_1} W_{1} + \beta_{W_2} W_{2} \;+\; T \cdot \big[\tau_0 + \eta \cdot (\gamma_{W_1} W_{1} + \gamma_{W_2} W_{2})\big] \;+\; \varepsilon,
\end{equation}
with $\tau_0 = 0.5$, $\gamma_{W_1} = +1$, $\gamma_{W_2} = -1$, $\beta_{W_1} = 0.3$, $\beta_{W_2} = -0.2$, $\sigma_\varepsilon = 1$. The scalar $\eta$ scales the heterogeneous component of the treatment effect and is the dial we use to traverse the power frontier.

\paragraph{Ground-truth labels.}
The ground-truth direct-modifier set $\mathcal{S}^\star$ is defined empirically as the SAE neurons most selective for each attribute over the $19{,}867$ images of the pool. Concretely, for each neuron $j$ and attribute $A \in \{W_1, W_2\}$, we sweep thresholds over $Z_{\mathrm{pre}}^{j}$ and report the best-threshold $F_1(j, A)$; we use $\bm Z_{\mathrm{pre}}$ rather than $\bm Z$ because the continuous pre-activations give a smoother, threshold-free alignment score that better reflects each neuron's intrinsic selectivity. The ground-truth principal coordinate for $A$ is then $\arg\max_j F_1(j, A)$. This labelling is computed once on all $19{,}867$ images (no sampling, no holdout) and held fixed throughout. Under the main setting ($k = 20$), $\mathcal{S}^\star = \{5348, 5537\}$ (dim 5348 $\to W_1$, dim 5537 $\to W_2$); under $k = 5$ the principal coordinates shift to $\{7044, 5732\}$, reflecting the different sparsity structure.

\paragraph{Sweeps and reporting.}
We sweep two axes, holding the other fixed:
\begin{itemize}[leftmargin=1.4em, itemsep=1pt, topsep=2pt]
    \item \emph{Effect-size sweep:} $\eta \in \{1, 2, \ldots, 10\}$, with $n \in \{500, 2000\}$ fixed.
    \item \emph{Sample-size sweep:} $n \in \{50, 100, 200, 350, 500, 750, 1000, 2000, 3500, 5000, 10{,}000\}$, with $\eta \in \{2, 5\}$ fixed.
\end{itemize}
Each $(\eta, n)$ cell is repeated over $50$ random seeds; reported curves are means with $\pm 1.96\,\mathrm{SE}$ shaded bands. Performance is reported as precision, recall, and intersection-over-union (IoU) of $\widehat{\mathcal{S}}_n$ against $\mathcal{S}^\star$.

\paragraph{Baselines.}
Classic CATE methods, e.g., forests \citep{wager2018estimation}, meta-learners \citep{kunzel2019metalearners}, causal trees \citep{athey2016recursive} and rule ensembles \citep{bargagli2020causal}, do not target the characterization sought here: they return pointwise effects, partitions or subgroups, not a set of coordinates. The task-aligned baseline is therefore marginal interaction screening, the standard multiple-testing pipeline for this question. We compare NEXIS to three marginal coordinate-wise variants that test each $j \in [m]$ for marginal effect modification (treating it as if it were the only modifier in the dictionary): \emph{Marginal Testing} (no multiple-testing correction), \emph{Marginal Testing (FDR)} \citep{benjamini1995controlling}, and \emph{Marginal Testing (FWER)} (Bonferroni). They share NEXIS's per-test statistic and threshold but skip the conditional, sequential structure, the precise design that, by Section~\ref{sec:power_paradox}, is statistically inadequate for direct-modifier identification on a learned representation.

\paragraph{Default NEXIS configuration.}
Unless stated otherwise, NEXIS is run with $\alpha = 0.05$, the linear treatment-interaction test of Equation~\ref{eq:linear_test}, FWER (Bonferroni) correction in the forward step, spectral gap $\rho = 0.5$, and the forward step followed by the terminal backward step. These are the defaults recommended in Appendix~\ref{sec:method-details}. The GCM and PCM nuisances are cross-fitted with $3$ folds, and their LightGBM variants use $50$ trees of depth at most $4$ with at most $15$ leaves.

\paragraph{What we test.}
Appendix~\ref{sec:semi-synt:reference} reports the main DGP (Figure~\ref{fig:celeba:dgp}) and its reproduction with an independently trained SAE (Figure~\ref{fig:celeba:replica}). Appendix~\ref{sec:semi-synt:alignment} tests Principal Alignment with supervision, and NEXIS under a controlled violation of it (Figure~\ref{fig:celeba:violation}). Appendices~\ref{sec:semi-synt:model-ablations} and~\ref{sec:semi-synt:method-ablations} ablate the dictionary (sparsity $k$, $\bm Z$ vs.\ $\bm Z_{\mathrm{pre}}$) and NEXIS (test, correction, $\rho$, backward steps); the test ablation adds a non-monotone U-shape DGP (Table~\ref{tab:celeba_test_story}). Appendix~\ref{sec:semi-synt:dgp-ablations} varies the number of direct modifiers, $r \in \{0, 1, 3\}$. Every DGP includes prognostic main effects of the attributes, $\beta_{W_1} = 0.3$ and $\beta_{W_2} = -0.2$ in the main setting. The CATE-equivalence test targets the treatment interaction, so its null holds for a coordinate that carries only prognostic signal: at $r = 1$, \textsc{Eyeglasses} keeps only its prognostic effect and NEXIS never selects its coordinate, and at $r = 0$ the share of runs with any false discovery is at most $0.015$ at every $n$.

\clearpage
\subsection{Main results}
\label{sec:semi-synt:reference}

\begin{figure}[h]
  \centering
  \includegraphics[width=0.97\textwidth]{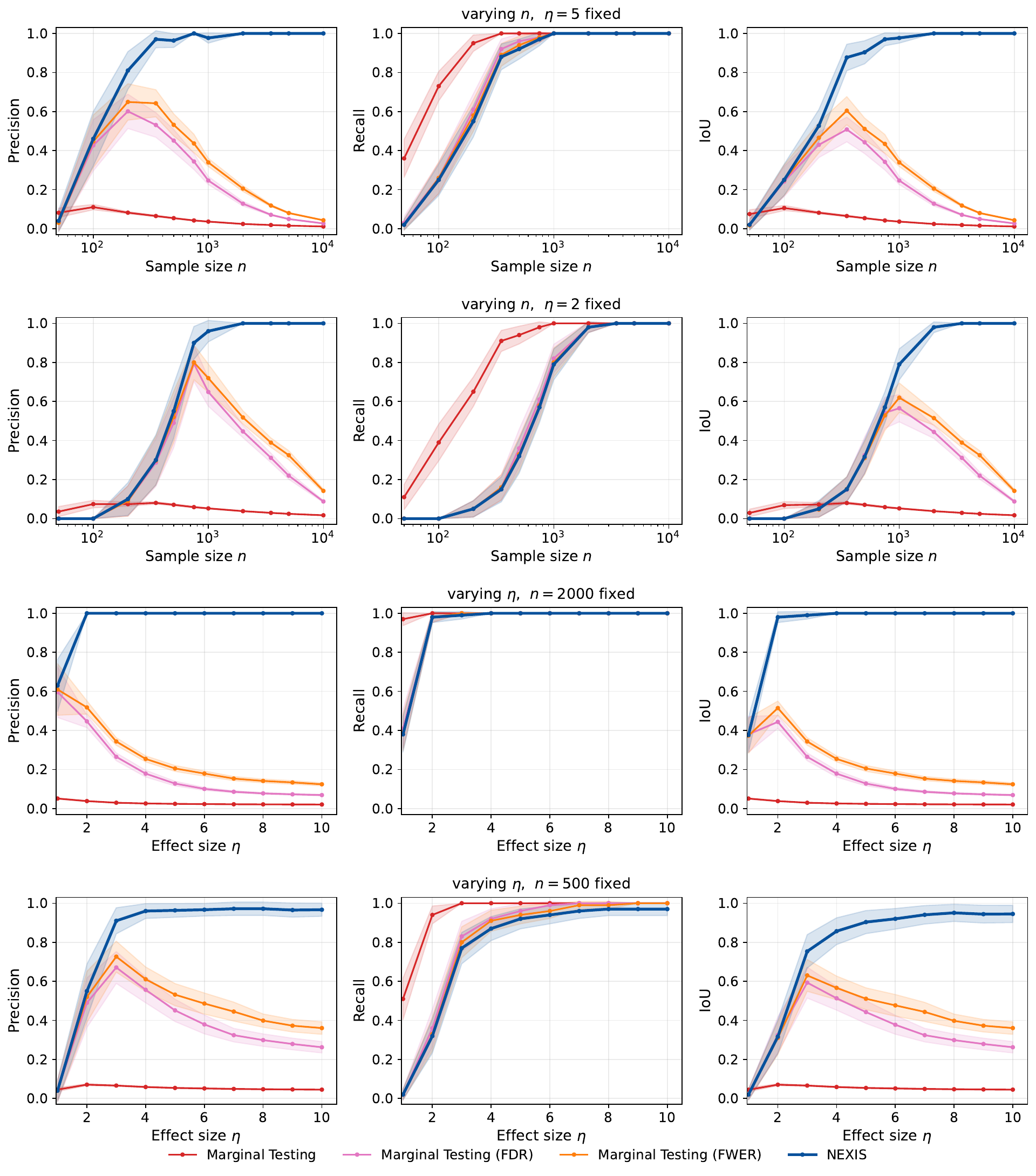}
  \caption{\small \textbf{Reference figure.} NEXIS vs.\ marginal baselines on the main setting ($k = 20$, $\bm Z$, FWER, $\rho = 0.5$, linear test, terminal backward step). Rows traverse the four DGP conditions; columns show precision, recall, IoU. NEXIS achieves both high recall and high precision in the high-power regime ($\eta \ge 2$ at $n = 2000$, or $n \ge 750$ at $\eta = 5$), while the marginal baselines achieve recall comparably early but never recover precision: as $\eta$ or $n$ grows, more entangled companions of $\mathcal{S}^\star$ become marginally significant and inflate the false-positive rate.}
  \label{fig:celeba:dgp}
\end{figure}

Figure~\ref{fig:celeba:dgp} reports the reference comparison and serves as the visual anchor for all subsequent ablations: each later figure should be read as a deviation from this one along a single axis. Three observations are worth surfacing.

\textit{(i) NEXIS attains the recovery regime predicted by Theorem~\ref{thm:nexis_consistency}.} Reading along the rows, recall crosses $0.95$ at $\eta = 2$ when $n = 2000$ and at $\eta = 7$ when $n = 500$. Precision reaches $0.95$ before recall ($n = 350$ at $\eta = 5$, $\eta = 4$ at $n = 500$), then stays above it. The IoU curve, which simultaneously penalises both error directions, is essentially the recall curve once one is past the under-powered regime.

\textit{(ii) The experimental power paradox of Section~\ref{sec:power_paradox} is visible cleanly.} Marginal baselines pass the recall bar comparably early (and Marginal Testing without correction passes it the earliest, by being a strictly looser gate). Their precision, however, never recovers: the unadjusted variant collapses below $0.1$ at moderate $\eta$ or $n$, and the FDR/FWER variants peak around $0.6$--$0.8$ at the smallest informative $n$ before \emph{deteriorating} as $\eta$ or $n$ grows. This is the predicted regime: more power surfaces more entangled companions of $\mathcal{S}^\star$ as marginally heterogeneous, none of which any marginal procedure has the structure to reject.

\textit{(iii) DGP sensitivity is graceful, not catastrophic.} Comparing the $\eta = 5$ and $\eta = 2$ rows, NEXIS shifts its recall onset from $n = 750$ to $n = 2000$, with no qualitative change in shape. Comparing the $n = 2000$ and $n = 500$ rows, NEXIS shifts its recall onset from $\eta = 2$ to $\eta = 7$ and saturates at a slightly lower precision plateau ($\sim 0.97$ vs.\ $1.0$) at $n = 500$ once $\eta$ is very large, a finite-sample residual companion entering at high SNR, consistent with the relaxed Principal Alignment discussion of Appendix~\ref{sec:method-details:rho}.

\paragraph{Reproducibility.}
\label{sec:semi-synt:reproducibility}
The SAE of the main setting is trained on the images the benchmark samples from. To check that the results do not depend on this, we retrain the $k = 20$ TopK SAE with the same recipe on an independent sample of $19{,}867$ images of the CelebA training split, which shares no identity with the validation split, and rerun the full main grid unchanged, with the same DGP, seeds and baselines; $\mathcal{S}^\star$ is again the best-$F_1$ coordinate of each attribute in the new dictionary. Figure~\ref{fig:celeba:replica} shows that, even when the concepts are not perfectly disentangled, NEXIS still identifies them. Recall is essentially unchanged (macro recall $0.729$ on the replica against $0.750$ on the main dictionary, averaged over the $42$ cells of the figure), while macro precision and IoU drop from $0.808$ and $0.745$ to $0.745$ and $0.680$, and the IoU at $n = 2000$, $\eta = 5$ from $1.000$ to $0.920$. NEXIS keeps dominating the marginal baselines, whose best macro precision on the replica is $0.312$. The remaining precision gap at large $n$ comes from coordinates of the replica dictionary that carry residual signal of the same concepts and count as false discoveries against the coordinate-level $\mathcal{S}^\star$.

\begin{figure}[h]
  \centering
  \includegraphics[width=0.97\textwidth]{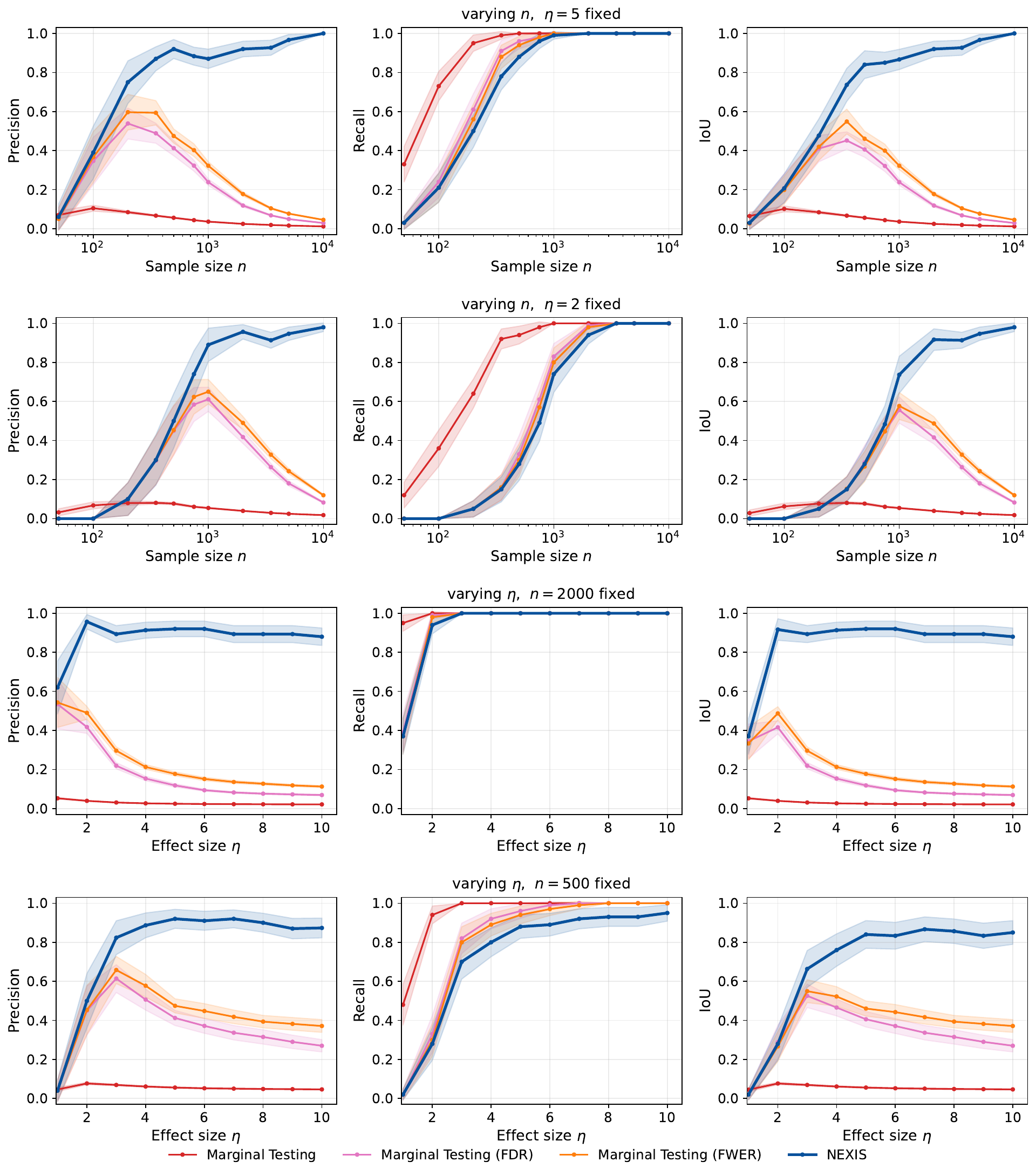}
  \caption{\small \textbf{Reproducibility: independently trained SAE.} Same layout as Figure~\ref{fig:celeba:dgp}, with the $k = 20$ TopK SAE retrained with the same recipe on $19{,}867$ CelebA training-split images that share no identity with the validation split; same DGP, seeds and baselines. Recall matches the main dictionary and NEXIS dominates the marginal baselines; precision saturates later because a few coordinates retain residual signal of the same concepts. Reference: Figure~\ref{fig:celeba:dgp}.}
  \label{fig:celeba:replica}
\end{figure}

\clearpage
\subsection{Testing Principal Alignment with supervision}
\label{sec:semi-synt:alignment}

Principal Alignment (Assumption~\ref{ass:principal_alignment}) cannot be checked without supervision, but in the CelebA benchmark the direct modifiers are observed attributes, so we test it on the main dictionary (TopK SAE, $k = 20$, $m = 9{,}216$, codes $\bm Z$ of the $19{,}867$ pool images).

\paragraph{Marginal alignment.}
Each principal coordinate is the most aligned coordinate by a wide margin: its activation separates \textsc{Wearing\_Hat} with AUC $0.969$ and \textsc{Eyeglasses} with AUC $0.956$, against $0.755$ and $0.714$ for the runner-up among the other $9{,}215$ coordinates (Figure~\ref{fig:celeba:pa_spectrum}), and $90\%$ and $95\%$ of its $100$ most activating images carry the attribute, against prevalences of $4.7\%$ and $7.0\%$ (Figure~\ref{fig:celeba:pa_top}). The alignment is not monosemantic: $9$ and $11$ other coordinates still reach AUC $0.6$, and some narrower coordinates reach a higher top-$100$ purity while covering only a small part of the attribute.

\paragraph{Screening-off.}
To test the screening-off of Equation~\ref{eq:principal_alignment} itself, we compare held-out logistic probes of $W^k$ on $Z^{j_k}$ alone (a linear term plus indicators at the deciles of its positive values, to allow a nonlinear link) with probes that also include the $K \in \{10, 50, 200\}$ other coordinates most associated with $W^k$ given $Z^{j_k}$, selected on a disjoint half of the images (Table~\ref{tab:celeba_pa_screening}, $20$ random splits). The probes are $\ell_2$-penalized ($C = 1$) and fitted on one half; AUCs are computed on the other half, on which the likelihood-ratio test compares unpenalized fits with and without the $K$ extra coordinates. The AUC rises only from $0.970$ to at most $0.988$ for \textsc{Wearing\_Hat} and from $0.956$ to at most $0.989$ for \textsc{Eyeglasses}, and the principal coordinate alone carries $94$--$96\%$ of the held-out information the augmented probe extracts. With about $10{,}000$ held-out images, a likelihood-ratio test of the $K$ extra coordinates still rejects exact conditional independence ($p < 10^{-14}$ in every split), so Principal Alignment holds only approximately, as a population-level idealization. This residual leakage does not affect recovery: in the same setting NEXIS returns exactly $\mathcal{S}^\star = \{5348, 5537\}$ in all $50$ runs at $n = 2000$ and $\eta = 5$, with mean IoU $0.999$ over $\eta \in \{3, \ldots, 10\}$ at $n = 2000$.

\begin{table}[h]
\centering\small
\caption{\small \textbf{Screening-off test of Principal Alignment.} Held-out AUC of logistic probes of $W^k$ on the principal coordinate $Z^{j_k}$ alone and on $Z^{j_k}$ plus the $K$ other coordinates most associated with $W^k$ given $Z^{j_k}$, selected on the training half; mean over $20$ random 50/50 splits, with the $2.5\%$--$97.5\%$ range over splits for $\Delta$AUC. Info.\ share: share of the held-out log-loss reduction of the augmented probe that the principal coordinate alone achieves.}
\label{tab:celeba_pa_screening}
\begin{tabular}{@{}lrcccc@{}}
\toprule
Modifier & $K$ & AUC $Z^{j_k}$ & AUC $Z^{j_k}$ + $K$ others & $\Delta$AUC & info.\ share \\
\midrule
\textsc{Wearing\_Hat} & 10 & 0.970 & 0.981 & +0.011 [+0.007, +0.014] & 0.96 \\
 & 50 &  & 0.985 & +0.015 [+0.010, +0.019] & 0.95 \\
 & 200 &  & 0.988 & +0.018 [+0.009, +0.024] & 0.95 \\
\midrule
\textsc{Eyeglasses} & 10 & 0.956 & 0.983 & +0.026 [+0.020, +0.030] & 0.95 \\
 & 50 &  & 0.989 & +0.033 [+0.028, +0.036] & 0.94 \\
 & 200 &  & 0.989 & +0.033 [+0.029, +0.037] & 0.94 \\
\bottomrule
\end{tabular}
\end{table}

\begin{figure}[h]
  \centering
  \includegraphics[width=0.8\textwidth]{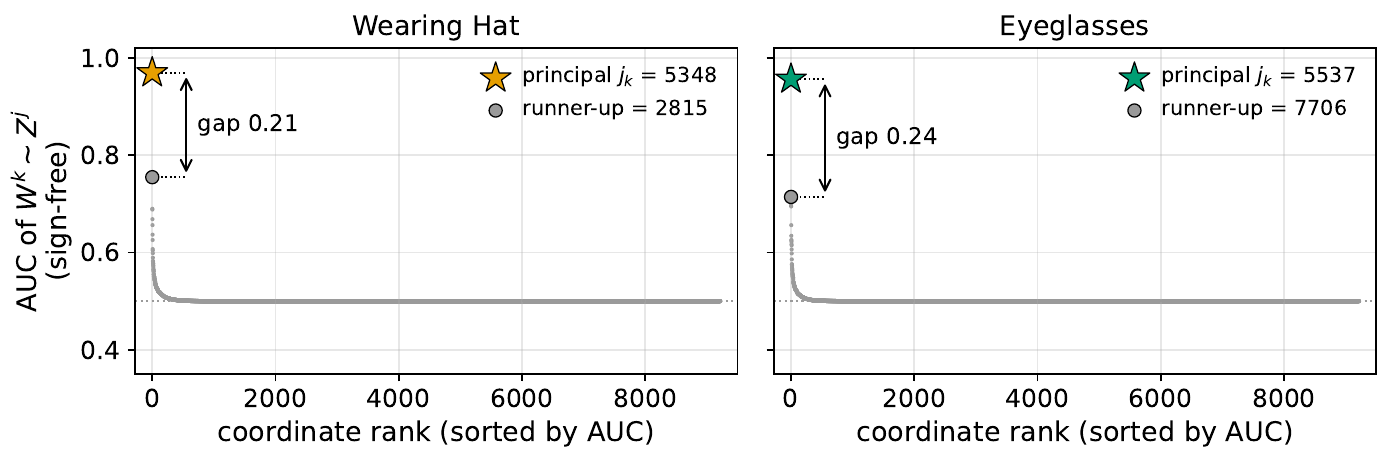}
  \caption{\small \textbf{Alignment spectrum.} Each point is one coordinate $Z^j$ of $\bm Z$, with the sign-free AUC, $\max(\mathrm{AUC}, 1 - \mathrm{AUC})$, of $Z^j$ as a score for $W^k$, \textsc{Wearing\_Hat} (left) and \textsc{Eyeglasses} (right), over the $19{,}867$ pool images; coordinates are sorted by AUC. The principal coordinate (star) leads the runner-up (dot) by $0.21$ and $0.24$.}
  \label{fig:celeba:pa_spectrum}
\end{figure}

\begin{figure}[h]
  \centering
  \includegraphics[width=0.9\textwidth]{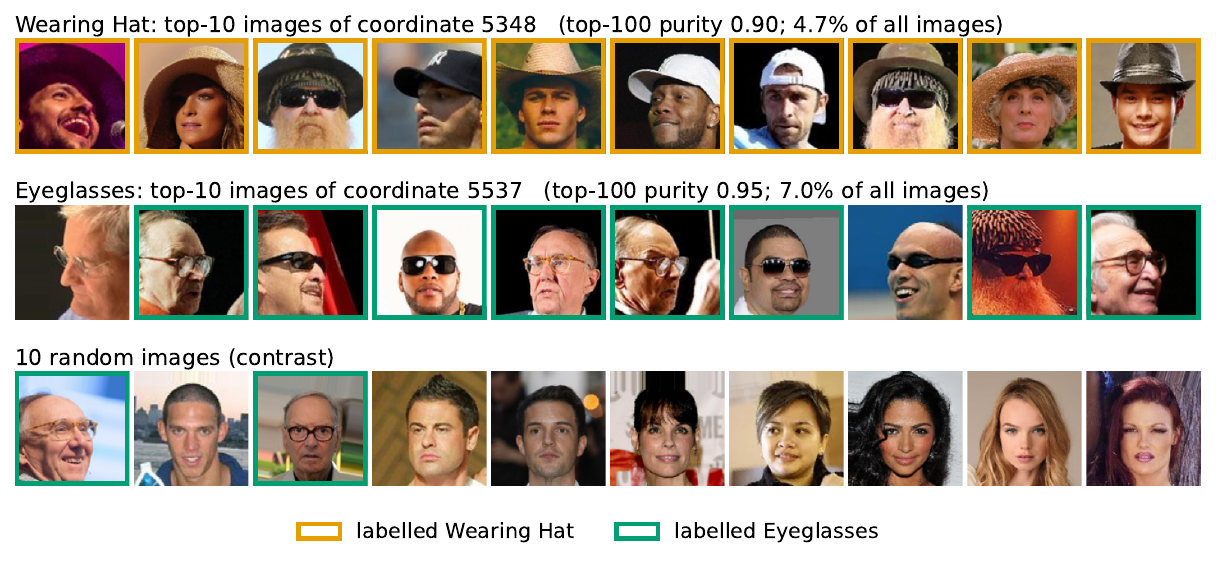}
  \caption{\small \textbf{Top-activating images of the principal coordinates.} The $10$ most activating images of coordinate $5348$ (\textsc{Wearing\_Hat}, top row) and $5537$ (\textsc{Eyeglasses}, middle row), and $10$ random images for contrast (bottom row). Framed images carry the attribute.}
  \label{fig:celeba:pa_top}
\end{figure}

\newpage
\paragraph{Controlled violation.}
To see what happens when Principal Alignment fails, we split the \textsc{Wearing\_Hat} principal coordinate $j_1$ in two: for each pool image we draw $U \sim \mathrm{Bernoulli}(0.5)$ once, set $Z^{j_1} \leftarrow U Z^{j_1}$, and append $Z^{j_{\mathrm{new}}} = (1 - U) Z^{j_1}$. The two halves sum to the original coordinate, but neither screens off the modifier alone, so no map as in Assumption~\ref{ass:principal_alignment} exists and the mean Markov boundary of the CATE grows from $\{j_1, j_2\} = \{5348, 5537\}$ to $\{j_1, j_{\mathrm{new}}, j_2\}$. On the draws of the main setting at $\eta = 5$ (Figure~\ref{fig:celeba:violation}), NEXIS selects both halves in $48$ of $50$ runs at $n = 2000$ and returns exactly $\{j_1, j_{\mathrm{new}}, j_2\}$ in all $50$ runs at $n = 10^4$, whereas Marginal Testing (FWER) keeps the three target coordinates but selects $47.9$ coordinates on average (precision $0.06$). The violation costs power, since each half is active on only about half of the images that activate the original coordinate: NEXIS needs $n = 3500$ instead of $n = 750$ to reach a mean IoU of $0.95$, and in the other $2$ runs at $n = 2000$ the spectral-gap gate stops the forward step before either half enters. Theorem~\ref{thm:nexis_consistency} assumes Principal Alignment; in this experiment, where it fails, NEXIS still recovers the larger Markov boundary once $n$ is large enough.

\begin{figure}[h]
  \centering
  \includegraphics[width=0.97\textwidth]{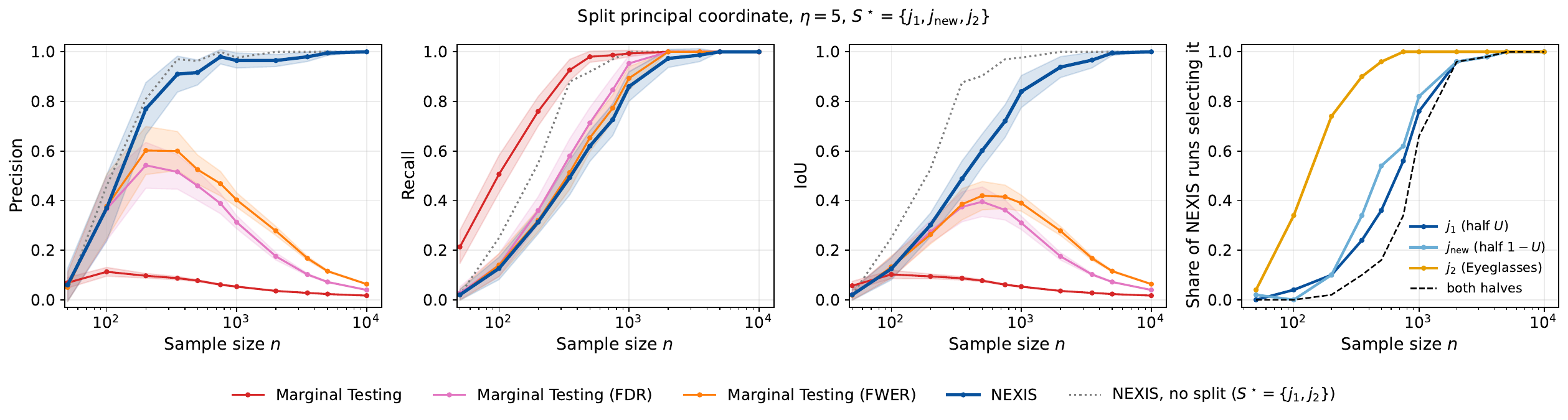}
  \caption{\small \textbf{Controlled violation of Principal Alignment.} The \textsc{Wearing\_Hat} principal coordinate is split into two complementary halves, so $\mathcal{S}^\star = \{j_1, j_{\mathrm{new}}, j_2\}$; $\eta = 5$. Left: precision, recall and IoU against $n$; dotted: NEXIS on the unsplit dictionary. Right: share of NEXIS runs selecting each target coordinate.}
  \label{fig:celeba:violation}
\end{figure}

\clearpage
\subsection{Model ablations: SAE design choices}
\label{sec:semi-synt:model-ablations}

The next two ablations vary the choice of dictionary, holding the NEXIS configuration at its default. We evaluate two competing SAE design decisions: the TopK sparsity level $k$, and whether to use the sparse post-TopK codes $\bm Z$ or the dense pre-activations $\bm Z_{\mathrm{pre}}$. Both deviate from the main setting along a single axis and are read against Figure~\ref{fig:celeba:dgp}.

\paragraph{Sparsity level: $k = 5$ vs.\ $k = 20$.}
Figure~\ref{fig:celeba:k5} repeats the reference comparison on the $k = 5$ dictionary. Recall thresholds are essentially unchanged ($\eta = 2$ at $n = 2000$; $n = 750$ at $\eta = 5$); precision, however, requires noticeably more power to saturate ($\eta = 2$ at $n = 2000$ in the effect-size sweep; $n = 5000$ at $\eta = 5$ in the sample-size sweep, against $n = 350$ for the main setting). With only five active features per image, attribute-specific signal is forced to spread across several correlated coordinates, so the conditional selection retains a slightly larger set in which a few non-principal companions persist before the backward gate prunes them. $k = 20$ is the better operating point.

\begin{figure}[h]
  \centering
  \includegraphics[width=0.97\textwidth]{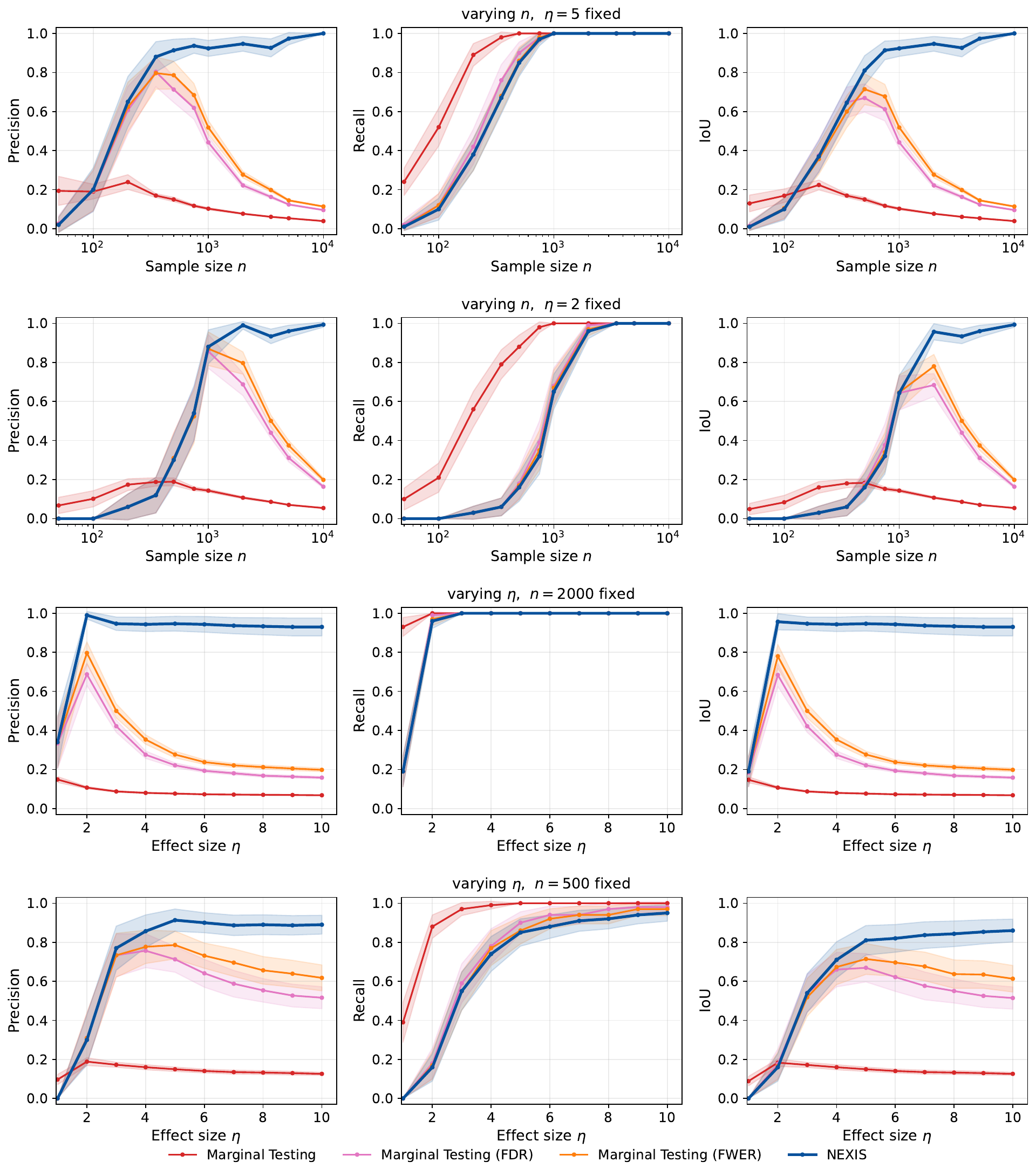}
  \caption{\small \textbf{Ablation: SAE sparsity level $k = 5$.} Same layout as Figure~\ref{fig:celeba:dgp}. Recall thresholds match the main setting; precision saturates more slowly because the lower-$k$ dictionary spreads attribute information across more coordinates. Reference: Figure~\ref{fig:celeba:dgp}.}
  \label{fig:celeba:k5}
\end{figure}

\paragraph{Feature view: $\bm Z$ vs.\ $\bm Z_{\mathrm{pre}}$.}
Figure~\ref{fig:celeba:precode} repeats the reference comparison using the dense pre-activations $\bm Z_{\mathrm{pre}}$ at $k = 20$. Recall crosses $0.95$ at the same point ($n = 750$ at $\eta = 5$); precision saturates at the same plateau but requires modestly more data ($n = 500$ at $\eta = 5$, vs.\ $n = 350$ on $\bm Z$). This is consistent with the geometry: the post-TopK codes $\bm Z$ are near-orthogonal, so conditioning on the already-selected set introduces little noise and the Bonferroni gate remains tight; $\bm Z_{\mathrm{pre}}$ retains continuous correlated pre-activations across the dictionary, which the conditional regression must absorb at a small cost in finite-sample efficiency. Sparse codes $\bm Z$ are the preferable input.

\begin{figure}[h]
  \centering
  \includegraphics[width=0.97\textwidth]{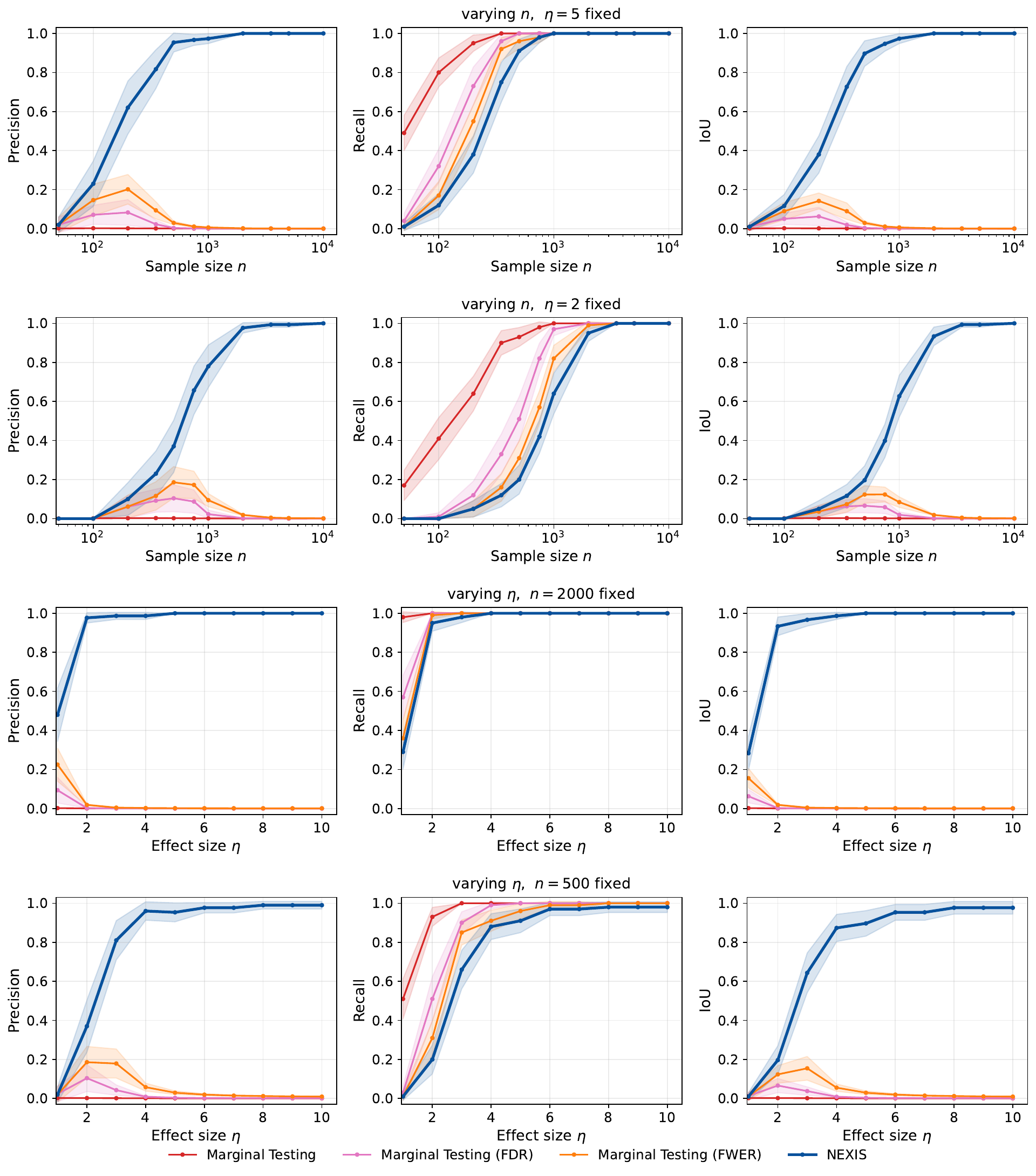}
  \caption{\small \textbf{Ablation: pre-TopK activations $\bm Z_{\mathrm{pre}}$.} Same layout as Figure~\ref{fig:celeba:dgp}, with the SAE feature view replaced by the dense pre-activations $\bm Z_{\mathrm{pre}}$ at $k = 20$. On this dense view every coordinate is a linear projection of the same embedding and inherits marginal heterogeneity, so the FWER baseline selects about $600$ coordinates on average and its precision collapses; the precision and IoU of NEXIS shift slightly to the right relative to the sparse-code reference, and the qualitative ordering is unchanged. Reference: Figure~\ref{fig:celeba:dgp}.}
  \label{fig:celeba:precode}
\end{figure}

\clearpage
\subsection{Method ablations: NEXIS design choices}
\label{sec:semi-synt:method-ablations}

We now hold the dictionary at the main setting ($k = 20$, $\bm Z$) and vary the four design choices internal to NEXIS: the CATE equivalence test, the multiple-testing correction, the spectral-gap parameter $\rho$, and the backward-elimination step. Each ablation deviates from the default along a single axis and is read against Figure~\ref{fig:celeba:dgp}.

\paragraph{CATE equivalence test.}
Figure~\ref{fig:celeba:test} compares the linear treatment-interaction test (default, Equation~\ref{eq:linear_test}) to the doubly-robust GCM and to the projected covariance measure (PCM) of \citet{lundborg2024projected} (Appendix~\ref{sec:method-details:test}), each with a quadratic and a LightGBM variant (of the nuisance regressions for the GCM, of the projection for the PCM). As expected, the linear test is the most sample-efficient: it crosses recall $0.95$ at $n = 750$ when $\eta = 5$ and at $\eta = 2$ when $n = 2000$. Both GCM variants need $n = 2000$ at $\eta = 5$ and $\eta = 4$ at $n = 2000$. Both PCM variants need $n = 3500$ at $\eta = 5$ and never reach recall $0.95$ at $n = 2000$ on the grid $\eta \le 10$ (recall about $0.6$ at $n = 2000$, $\eta = 5$): the PCM pays for the sample split, since in each direction the statistic is evaluated on $n/2$ observations, which scales it by about $1/\sqrt{2}$. At $n = 500$, the recall of all four stays at or below $0.04$ up to $\eta = 10$. The linearity the test needs is that of the conditional CATE \(\E[\tau \mid \bm Z^{S \cup \{j\}}]\) in the coordinates, and here it holds approximately: \(\tau\) in Equation~\ref{eq:celeba_dgp} is by construction \emph{linear} in the binary attributes $W_1, W_2$, and the principal coordinates $Z^{j_k}$ act close to indicators of these attributes. Linearity is a property of the data-generating process; a dictionary can at most expose it, when its aligned coordinates are near-indicators of the modifiers, but cannot create it. When \(\tau\) is linear in the modifiers and the dictionary aligns them in this way, the linear test is well specified. The GCM and PCM variants pay a power penalty for a nonparametric flexibility that this DGP does not reward, and the nuisance (or projection) model barely matters: the quadratic and LightGBM curves of each test nearly coincide. The two GCM variants even select the same sets in every cell: the outcome-side nuisances enter only once $S \neq \emptyset$, where $\bm Z^S$ is one or two sparse coordinates that are zero on about $90\%$ of the samples, so both fits give nearly identical statistics that never fall on opposite sides of the gate.

\begin{figure}[h]
  \centering
  \includegraphics[width=0.97\textwidth]{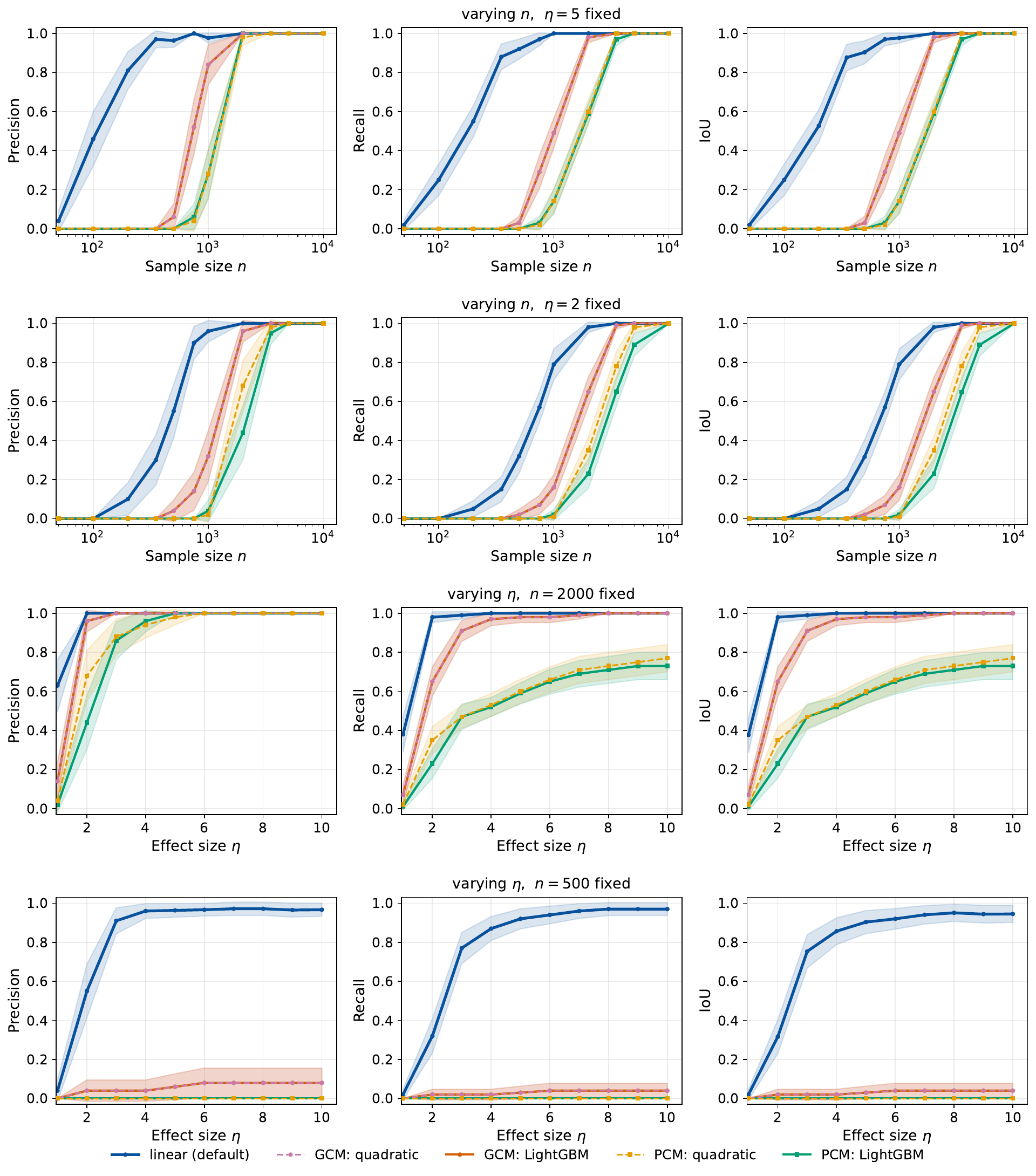}
  \caption{\small \textbf{Ablation: CATE equivalence test.} Linear (default), GCM and PCM, each with a quadratic or a LightGBM variant (of the nuisance regressions for the GCM, of the projection for the PCM), on the main setting. The linear test is close to well specified, since \(\tau\) in Equation~\ref{eq:celeba_dgp} is linear in the attributes and the principal coordinates act as near-indicators of them, and dominates in this regime. Reference: Figure~\ref{fig:celeba:dgp}.}
  \label{fig:celeba:test}
\end{figure}

The split buys power against non-monotone heterogeneity. Table~\ref{tab:celeba_test_story} adds a U-shape DGP (panel B) that keeps the prognostic effects of Equation~\ref{eq:celeba_dgp} but sets $\tau = \tau_0 + \eta\,\big(g_1(Z_{\mathrm{pre}}^{5348}) - g_2(Z_{\mathrm{pre}}^{5537})\big)$, where $g_k(z)$ is the square of the standardized pre-activation, residualized on $(1, z)$ over the pool and scaled to unit variance, so that $\tau$ varies with both principal coordinates while having zero covariance with each. The linear test and both GCM variants are blind to it: their recall never exceeds $0.10$ on this grid and is zero from $n = 2000$ on. The PCM recovers both modifiers: its IoU reaches $0.95$ at $n = 1000$ with the LightGBM projection ($1.00$ at $n = 2000$), and at $n = 3500$ with the quadratic one ($0.77$ at $n = 2000$). The choice of test is thus a trade-off between power and the class of alternatives it sees. The linear test is the most efficient when the conditional CATE is linear in the coordinates. The GCM is a compromise: more expressive than the linear test, since it is robust to misspecified nuisances, and more efficient than the PCM when the heterogeneity is monotone, but blind to non-monotone heterogeneity at every $n$. The PCM sees non-monotone heterogeneity at the cost of the split. We therefore use the linear test by default, the GCM when the heterogeneity may be nonlinear but monotone, and the PCM when non-monotone heterogeneity cannot be ruled out and enough data are available.

\begin{table}[h]
\centering\small
\caption{\small \textbf{CATE equivalence test under a linear and a non-monotone CATE.} $n^\star$: smallest $n$ on the grid with mean IoU at least $0.95$ at $\eta = 5$ (--: never); IoU: mean IoU at $n = 2000$, $\eta = 5$. A: main setting, $\tau$ linear in \textsc{Wearing\_Hat} and \textsc{Eyeglasses}, codes $\bm Z$. B: U-shape DGP, pre-activations $\bm Z_{\mathrm{pre}}$. $50$ seeds per cell, default NEXIS otherwise.}
\label{tab:celeba_test_story}
\begin{tabular}{@{}lcccc@{}}
\toprule
 & \multicolumn{2}{c}{A. linear CATE} & \multicolumn{2}{c}{B. U-shape} \\
\cmidrule(lr){2-3}\cmidrule(lr){4-5}
Test & $n^\star$ & IoU & $n^\star$ & IoU \\
\midrule
linear (default) & 750 & 1.00 & -- & 0.00 \\
GCM: quadratic & 2000 & 0.98 & -- & 0.00 \\
GCM: LightGBM & 2000 & 0.98 & -- & 0.00 \\
PCM: quadratic & 3500 & 0.60 & 3500 & 0.77 \\
PCM: LightGBM & 3500 & 0.59 & 1000 & 1.00 \\
\bottomrule
\end{tabular}
\end{table}

\paragraph{Multiple-testing correction.}
Figure~\ref{fig:celeba:adjust} compares the FWER (Bonferroni, default) gate of the forward step to the FDR (Benjamini--Hochberg) gate and to the unadjusted run; the terminal backward step keeps its level $\alpha/m$ throughout. With no correction, the terminal backward step removes the extra coordinates admitted by the forward step, so precision reaches $0.95$ almost as early as with FWER ($\eta = 2$ at $n = 2000$; $n = 500$ instead of $n = 350$ at $\eta = 5$), but recall comes later: $n = 1000$ instead of $n = 750$ at $\eta = 5$, and $\eta = 8$ instead of $\eta = 7$ at $n = 500$. FDR and FWER are essentially indistinguishable here because the Benjamini--Hochberg rejection sets at the leading discoveries are effectively singletons, so the two gates coincide. We keep FWER as the more conservative default; in larger truth-set regimes FDR may be preferable.

\begin{figure}[h]
  \centering
  \includegraphics[width=0.97\textwidth]{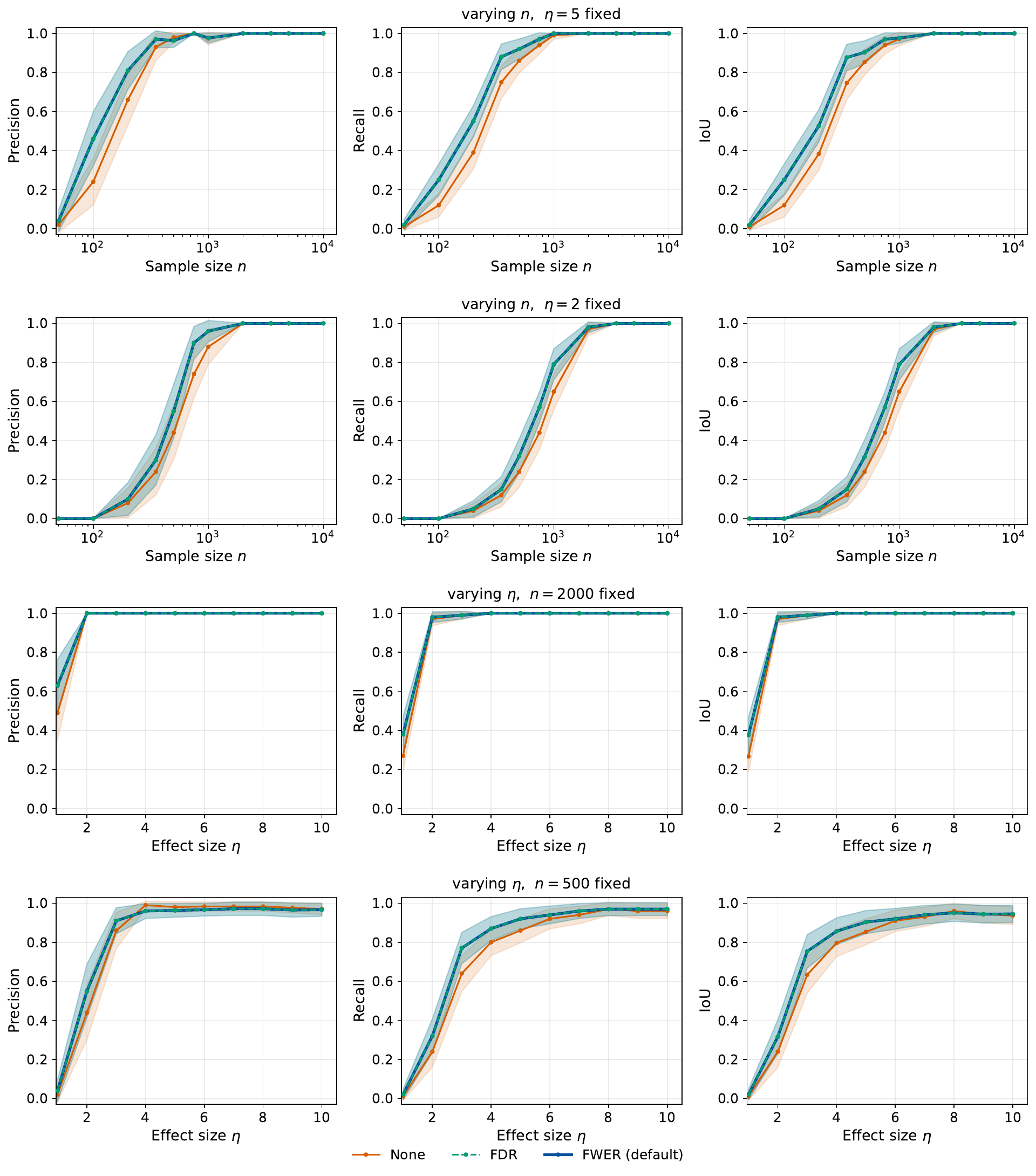}
  \caption{\small \textbf{Ablation: multiple-testing correction.} None, FDR (Benjamini--Hochberg), and FWER (Bonferroni, default), on the main setting. FWER and FDR are equivalent at this truth-set size; the unadjusted gate mainly delays recall. Reference: Figure~\ref{fig:celeba:dgp}.}
  \label{fig:celeba:adjust}
\end{figure}

\paragraph{Spectral gap $\rho$.}
Figure~\ref{fig:celeba:rho} sweeps $\rho \in \{0, 0.2, 0.5, 0.8\}$, where $\rho = 0$ disables the gap entirely (Equation~\ref{eq:rho_gate} of Appendix~\ref{sec:method-details:rho}). The lower-$\rho$ runs ($0$ and $0.2$) admit correlated companions of the principal coordinates as soon as the conditional gate is loose enough, hurting precision in the large-$\eta$ and large-$n$ regimes for at most a marginal recall gain. The higher-$\rho$ run ($0.8$) is conservative in the opposite direction: it blocks coordinates whose conditional $t$-statistic is comparable to that of an already-selected principal but slightly smaller, occasionally including the second principal coordinate itself, and only crosses recall $0.95$ at $\eta = 7$ in the $n = 2000$ row. The default $\rho = 0.5$ gives the best trade-off in standardized conditional signal strength (Appendix~\ref{sec:method-details:rho}).

\begin{figure}[h]
  \centering
  \includegraphics[width=0.97\textwidth]{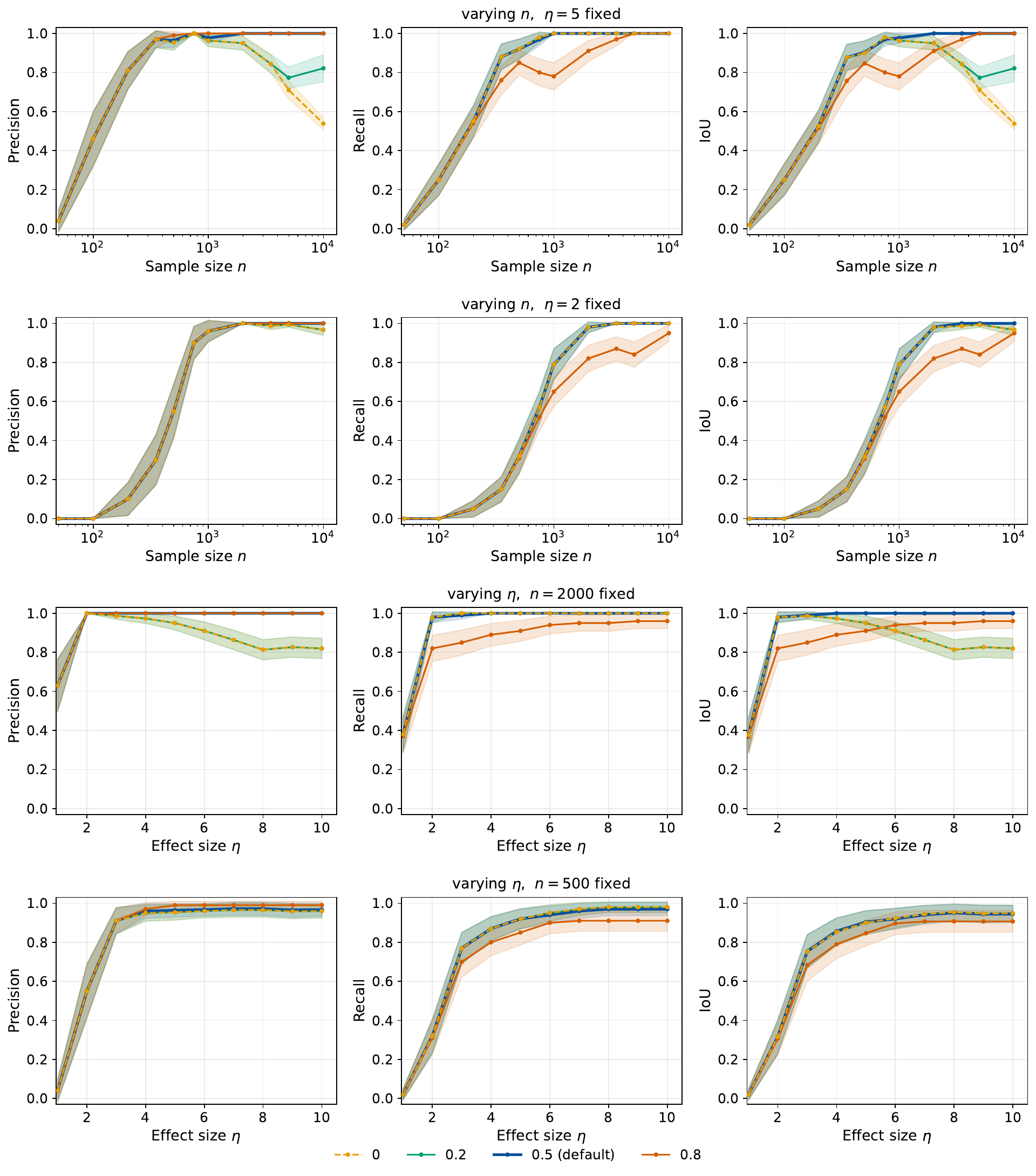}
  \caption{\small \textbf{Ablation: spectral gap $\rho$.} $\rho \in \{0, 0.2, 0.5, 0.8\}$, on the main setting. Low $\rho$ admits correlated companions and erodes precision; high $\rho$ blocks weaker principal coordinates and erodes recall. The default $\rho = 0.5$ jointly attains both. Reference: Figure~\ref{fig:celeba:dgp}.}
  \label{fig:celeba:rho}
\end{figure}

\paragraph{Backward elimination.}
Figure~\ref{fig:celeba:backward} compares the forward step alone and the default (forward step followed by the terminal backward step), each with and without the interleaved backward step. The terminal backward step removes the residual companions left by the forward step: precision reaches $0.95$ at $n = 350$ instead of $n = 2000$ at $\eta = 5$, and stays at or above $0.96$ at $n = 500$ for $\eta \ge 4$ (forward step alone: $0.72$ at $\eta = 10$), for a small delay in recall ($n = 750$ instead of $n = 500$ at $\eta = 5$; $\eta = 7$ instead of $\eta = 5$ at $n = 500$). The interleaved backward step changes no run. It matters only when a coordinate that mixes several modifiers enters before their principal coordinates and becomes redundant once they are in, which two well-aligned modifiers on sparse SAE codes never produce; we keep it as a variant for more entangled settings.

\begin{figure}[h]
  \centering
  \includegraphics[width=0.97\textwidth]{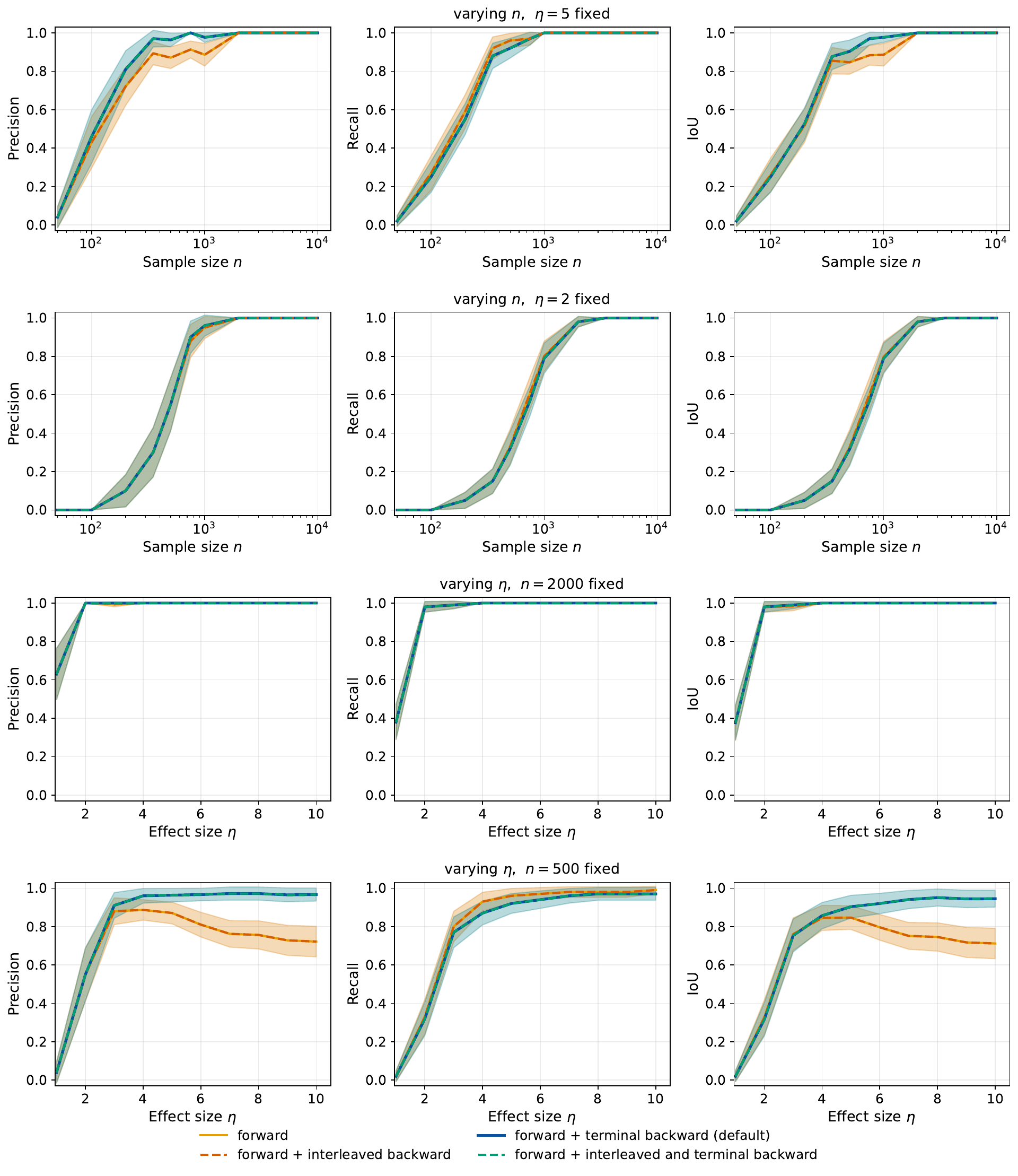}
  \caption{\small \textbf{Ablation: backward steps.} Forward step alone, with the interleaved backward step, with the terminal backward step (default), and with both, on the main setting. The terminal backward step raises precision for a small delay in recall; the interleaved backward step changes no run at $|\mathcal{S}^\star| = 2$ on a near-orthogonal $\bm Z$, so the dashed curves overlap the solid ones. Reference: Figure~\ref{fig:celeba:dgp}.}
  \label{fig:celeba:backward}
\end{figure}

\clearpage
\subsection{Additional data generating process ablations}
\label{sec:semi-synt:dgp-ablations}

We now hold the dictionary and NEXIS at the main setting and vary the number $r$ of direct modifiers, with the same sweeps as in Appendix~\ref{sec:semi-synt:reference}. At $r = 1$, only \textsc{Wearing\_Hat} modifies the effect: \textsc{Eyeglasses} is still sampled with its prognostic effect $\beta_{W_2}$, but with no interaction, and $\mathcal{S}^\star = \{5348\}$. At $r = 3$, \textsc{Sideburns} joins the two modifiers of Equation~\ref{eq:celeba_dgp}, with $(\gamma_{W_1}, \gamma_{W_2}, \gamma_{W_3}) = (+1, -1, +1)$ and $(\beta_{W_1}, \beta_{W_2}, \beta_{W_3}) = (0.3, -0.2, 0.3)$, the three attributes sampled independently at their CelebA prevalence, and $\mathcal{S}^\star = \{5348, 5537, 1683\}$. At $r = 0$, there is no effect modification: $\tau = \tau_0$ for every unit, the prognostic effects are those of the main setting, and $\mathcal{S}^\star = \emptyset$. Macro metrics average over the $42$ cells of a figure.

\paragraph{No modifier ($r = 0$).}
With $\gamma = 0$, $\eta$ multiplies nothing, and a larger constant effect $\tau_0$ leaves every interaction statistic unchanged, so the four DGP conditions collapse to one; we sweep $n$ with $200$ seeds per value. Any selected coordinate is then a false discovery. Table~\ref{tab:celeba_r0} reports the share of runs with at least one. For NEXIS it is at most $0.015$ at every $n$, and $0.003$ pooled over the $2{,}200$ runs, below $\alpha = 0.05$, while uncorrected marginal testing makes a false discovery in $99.5\%$ to $100\%$ of the runs. At $r = 0$, NEXIS makes a false discovery in exactly the same runs as the Bonferroni marginal baseline, since with no modifier its first forward step is that test.

\begin{table}[h]
\centering\small
\caption{\small \textbf{DGP ablation: no effect modification ($r = 0$).} Share of runs with at least one false discovery, $\alpha = 0.05$, $200$ seeds per $n$.}
\label{tab:celeba_r0}
\begin{tabular}{@{}rccccc@{}}
\toprule
 & \multicolumn{4}{c}{share of runs with any false discovery} & mean false \\
$n$ & NEXIS & Marginal & Marginal (FWER) & Marginal (FDR) & discoveries (NEXIS) \\
\midrule
50 & 0.000 & 0.995 & 0.000 & 0.000 & 0.000 \\
100 & 0.000 & 1.000 & 0.000 & 0.000 & 0.000 \\
200 & 0.000 & 1.000 & 0.000 & 0.000 & 0.000 \\
350 & 0.000 & 1.000 & 0.000 & 0.000 & 0.000 \\
500 & 0.000 & 1.000 & 0.000 & 0.000 & 0.000 \\
750 & 0.010 & 1.000 & 0.010 & 0.010 & 0.010 \\
1000 & 0.015 & 1.000 & 0.015 & 0.015 & 0.015 \\
2000 & 0.000 & 1.000 & 0.000 & 0.000 & 0.000 \\
3500 & 0.005 & 1.000 & 0.005 & 0.005 & 0.005 \\
5000 & 0.000 & 1.000 & 0.000 & 0.000 & 0.000 \\
10000 & 0.000 & 1.000 & 0.000 & 0.000 & 0.000 \\
\midrule
all~$n$ & 0.003 & 1.000 & 0.003 & 0.003 & 0.003 \\
\bottomrule
\end{tabular}
\end{table}

\clearpage
\paragraph{One modifier ($r = 1$).}
NEXIS recovers one modifier about as well as two (Figure~\ref{fig:celeba:r1}): macro recall and IoU are $0.749$ and $0.734$, against $0.750$ and $0.745$ at $r = 2$, and the IoU at $n = 2000$, $\eta = 5$ is $1.000$ in both; macro precision is lower, $0.734$ against $0.808$. The reason is that at $r = 2$ recovery is limited by \textsc{Wearing\_Hat}, the harder of the two modifiers to recover. Its per-modifier macro recall, the share of runs whose selection contains its principal coordinate, is $0.713$ at $r = 2$ against $0.787$ for \textsc{Eyeglasses}, and only $0.749$ at $r = 1$, where it is the only modifier. The purely prognostic \textsc{Eyeglasses} coordinate $5537$ is selected in none of the $2{,}100$ NEXIS runs at $r = 1$.

\begin{figure}[h]
  \centering
  \includegraphics[width=0.97\textwidth]{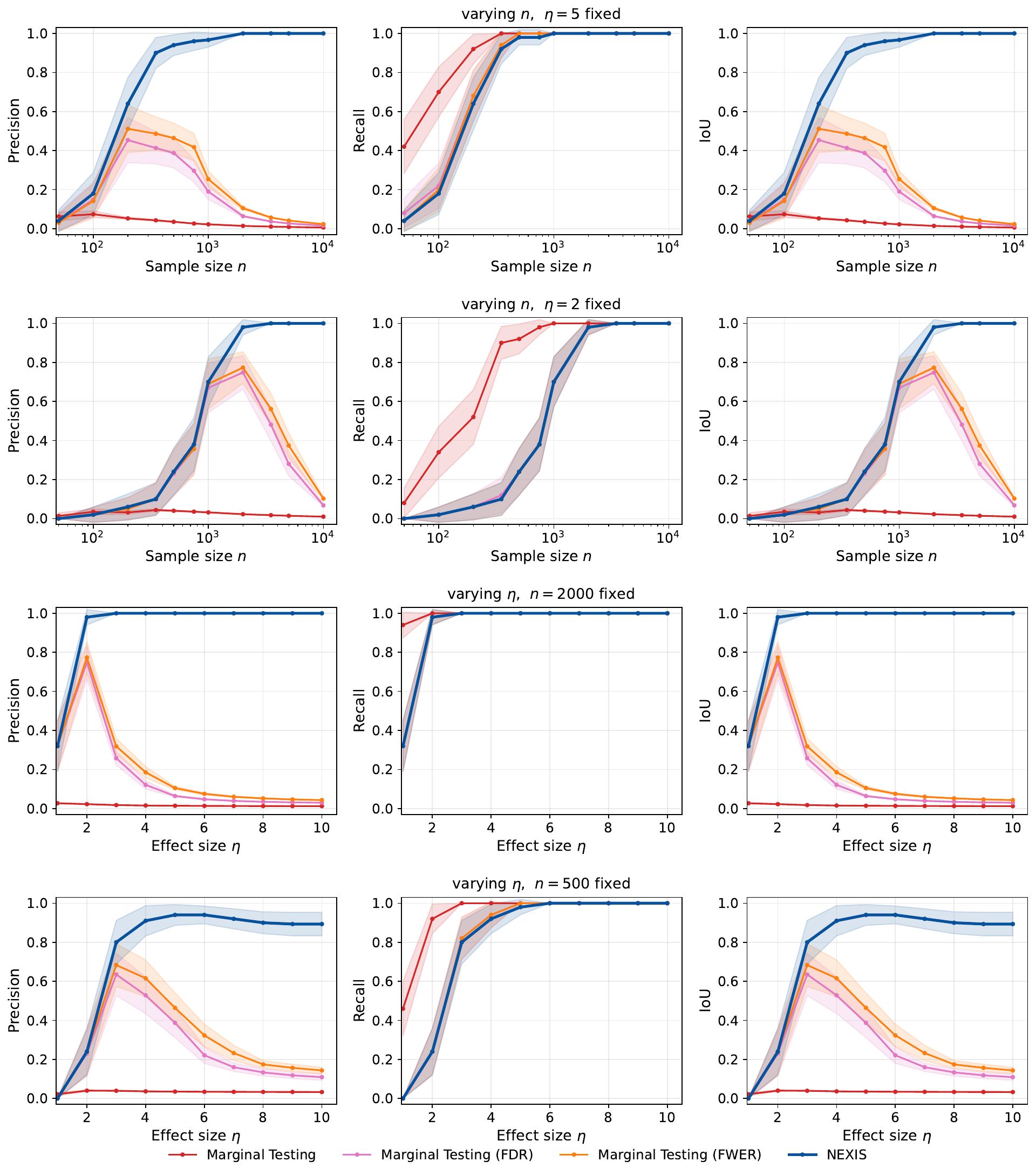}
  \caption{\small \textbf{DGP ablation: one direct modifier ($r = 1$).} Same layout as Figure~\ref{fig:celeba:dgp}, with only \textsc{Wearing\_Hat} modifying the effect; \textsc{Eyeglasses} keeps its prognostic effect. Recall and IoU are close to those with two modifiers. Reference: Figure~\ref{fig:celeba:dgp}.}
  \label{fig:celeba:r1}
\end{figure}

\clearpage
\paragraph{Three modifiers ($r = 3$).}
The third modifier delays recall (Figure~\ref{fig:celeba:r3}): recall crosses $0.95$ at $n = 1000$ when $\eta = 5$ and at $n = 3500$ when $\eta = 2$, against $n = 750$ and $n = 2000$ at $r = 2$, and at $n = 500$ it stays below $0.95$ up to $\eta = 10$. Macro precision drops less, from $0.808$ to $0.761$, and the IoU at $n = 2000$, $\eta = 5$ is $0.953$; macro recall and IoU are $0.660$ and $0.644$, against $0.750$ and $0.745$. \textsc{Sideburns} is the least aligned of the three concepts: its principal coordinate separates it with AUC $0.805$ on $\bm Z$, against $0.969$ and $0.956$ for the other two (Appendix~\ref{sec:semi-synt:alignment}). Consistently, its per-modifier macro recall is the lowest, $0.597$, against $0.657$ for \textsc{Wearing\_Hat} and $0.726$ for \textsc{Eyeglasses}.

\begin{figure}[h]
  \centering
  \includegraphics[width=0.97\textwidth]{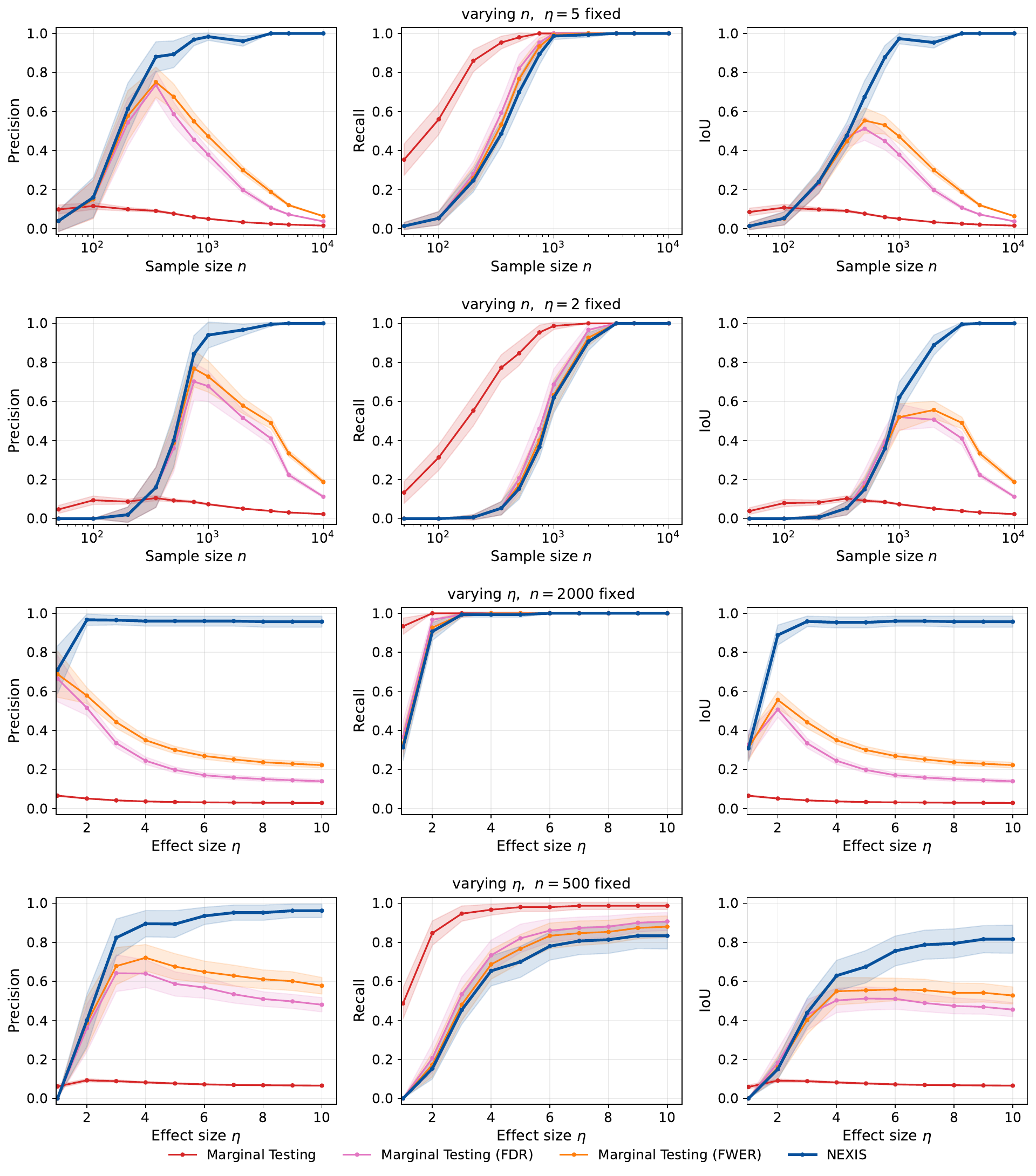}
  \caption{\small \textbf{DGP ablation: three direct modifiers ($r = 3$).} Same layout as Figure~\ref{fig:celeba:dgp}, with \textsc{Sideburns} added as a third modifier. Recall is delayed, most for the less aligned \textsc{Sideburns}. Reference: Figure~\ref{fig:celeba:dgp}.}
  \label{fig:celeba:r3}
\end{figure}

\clearpage
\subsection{Practical guidance and summary}
\label{sec:semi-synt:guidance}

Aggregating across the ablations, the recommended default configuration is: $k = 20$ TopK sparse codes $\bm Z$, linear conditional test, FWER (Bonferroni) correction, $\rho = 0.5$, forward step followed by the terminal backward step. This is the configuration we run by default in the applied analyses (Appendix~\ref{sec:uganda}, Appendix~\ref{sec:ghana}), with two principled deviations:
\begin{itemize}[leftmargin=1.4em, itemsep=1pt, topsep=2pt]
    \item \emph{Switch the test} when the conditional CATE cannot be assumed close to linear in $\bm Z$: to the GCM when the heterogeneity may be nonlinear but monotone, and to the PCM when non-monotone heterogeneity cannot be ruled out and enough observations are available. Both pay a power penalty on the semi-synthetic benchmark (Figure~\ref{fig:celeba:test}), and only the PCM covers non-monotone alternatives (Table~\ref{tab:celeba_test_story}, Appendix~\ref{sec:method-details:test}). Both applied analyses use the linear test.
    \item \emph{Lower $\rho$} only when the dictionary is verified to be near-orthogonal (e.g., on the SAE codes used here), where the spectral gap is not the binding constraint. \emph{Raise $\rho$} only when the principal modifiers are expected to have similar standardized conditional signal strengths; in the absence of such prior, $\rho = 0.5$ is the safer choice.
\end{itemize}

A compact summary of the ablations is given in Table~\ref{tab:celeba_summary}.

\begin{table}[h]
\centering\small
\caption{\small Summary of CelebA ablations. Each row deviates from the main setting ($k{=}20$, $\bm Z$, FWER, $\rho{=}0.5$, linear test, terminal backward step) along a single axis.}
\label{tab:celeba_summary}
\setlength{\tabcolsep}{4pt}
\begin{tabular}{@{}lll@{}}
\toprule
Ablation (figure)                                  & Varies                                & Finding \\
\midrule
SAE sparsity (Fig.~\ref{fig:celeba:k5})            & $k{=}5$ vs.\ $20$                     & Too sparse ($k=5$) leads to same recall, lower precision \\
Feature view (Fig.~\ref{fig:celeba:precode})       & $\bm Z_{\mathrm{pre}}$ vs.\ $\bm Z$    & $\bm Z_{\mathrm{pre}}$ needs more data \\
Conditional test (Fig.~\ref{fig:celeba:test})      & Linear vs.\ GCM/PCM                    & Linear best on a linear CATE; only PCM sees a U-shape \\
MHT correction (Fig.~\ref{fig:celeba:adjust})      & None / FDR / FWER                     & FDR $\equiv$ FWER; unadjusted delays recall \\
Spectral gap (Fig.~\ref{fig:celeba:rho})           & $\rho \in \{0, 0.2, 0.5, 0.8\}$       & $\rho{=}0.5$ jointly optimal \\
Backward steps (Fig.~\ref{fig:celeba:backward})    & Interleaved / terminal                & Terminal raises precision; interleaved inactive here \\
\bottomrule
\end{tabular}
\end{table}

Across all ablations, no single design choice is catastrophic: NEXIS degrades gracefully along each axis, and consistently outperforms every marginal baseline on precision in every condition tested. This is the empirical counterpart of the asymmetric guarantee of Theorem~\ref{thm:nexis_consistency}: \textit{recall} is the easy direction; \textit{precision} is what the conditional, sequential, Bonferroni-gated structure of NEXIS buys.

\clearpage
\subsection{Computational budget}
\label{sec:semi-synt:compute}

\begin{table}[h]
\centering\small
\caption{Compute budget for the CelebA semi-synthetic experiments. GPU refers to a single NVIDIA RTX 2080 Ti (the replica SAE of Appendix~\ref{sec:semi-synt:reproducibility} was trained on one H100).}
\label{tab:celeba_compute}
\begin{tabular}{@{}lll@{}}
\toprule
Component & Hardware & Runtime \\
\midrule
SigLIP~2 embedding ($19{,}867$ images) & GPU & $\sim$10\,min \\
SAE training (main $k \in \{5, 20\}$, replica $k = 20$) & GPU & $\sim$1\,h \\
Principal coordinates and Principal Alignment tests & CPU & minutes \\
All experiments (NEXIS, baselines) & CPU & $\sim$400 CPU-h ($\sim$10\,h on 40 cores) \\
\bottomrule
\end{tabular}
\end{table}

\newpage
\section{Case Study 1: Youth Opportunities Program (Uganda)}
\label{sec:uganda}
 
This appendix complements Section~\ref{sec:uganda_main} with the full pipeline, results, and limitations of the YOP case study.
 
\subsection{Trial design and outcomes}
 
\paragraph{Program.} The Youth Opportunities Program \citep{blattman2014generating} offered cash grants of approximately \$382 per group member, plus optional vocational training, to self-organized youth groups in Northern Uganda in the aftermath of the Lord's Resistance Army conflict. The program targeted young adults with limited economic opportunity.
 
\paragraph{Randomization and sample.} Treatment was randomized at the \emph{group level} (the unit of randomization; self-selected youth groups of $\sim$15--20 members), with groups clustered within geographic communities. Baseline data were collected pre-treatment; endline outcomes were measured 2--4 years later. The analytic sample consists of $2{,}082$ individuals across $439$ groups and $327$ distinct geographic communities, with a treatment rate of $39.6\%$ ($825$ treated). Sites span Karamoja, Teso, Lango, and West Nile sub-regions. Figure~\ref{fig:uganda_geo} shows the geographic distribution of communities by district and by language group.
 
\begin{figure}[h]
  \centering
  \begin{subfigure}[t]{0.495\textwidth}
    \centering
    \includegraphics[width=\textwidth]{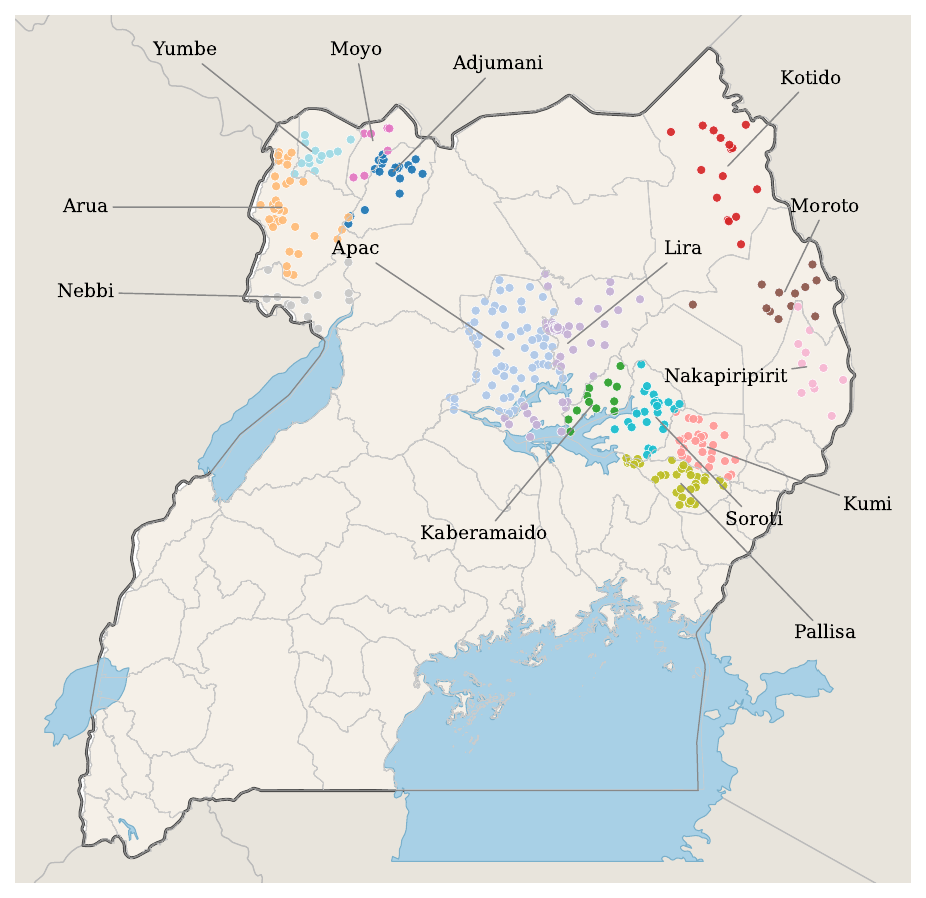}
    \caption{\small Communities by district.}
    \label{fig:uganda_geo:districts}
  \end{subfigure}\hfill
  \begin{subfigure}[t]{0.495\textwidth}
    \centering
    \includegraphics[width=\textwidth]{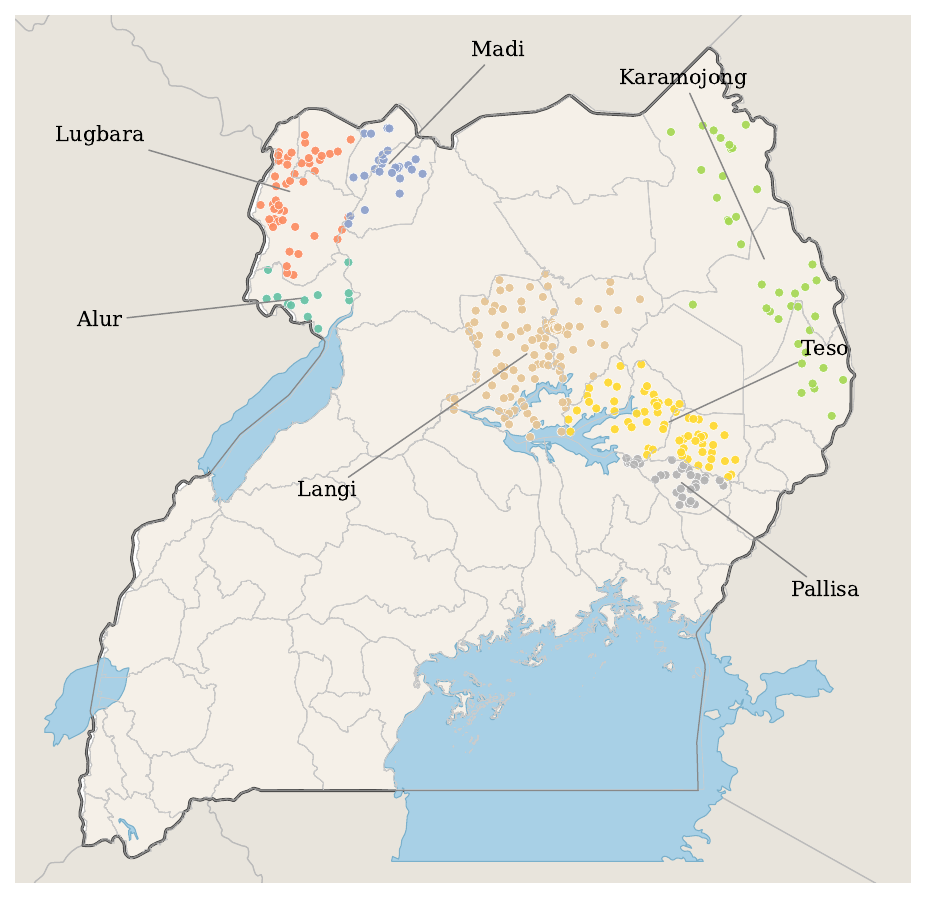}
    \caption{\small Communities by language group.}
    \label{fig:uganda_geo:languages}
  \end{subfigure}
  \caption{\small Geographic distribution of the $331$ YOP RCT sites ($327$ distinct communities) across Northern Uganda, by district and by primary language group.}
  \label{fig:uganda_geo}
\end{figure}
 
\paragraph{Outcomes.} We analyse two endline outcomes: \emph{skilled employment}, a binary indicator of any skilled trade engagement at endline; and \emph{log business assets}, the logged real business asset value at endline. Both capture the productive-capacity channel through which the program was expected to operate, the first via labour market participation and the second via capital accumulation.

\paragraph{Estimand and identification.}
Treatment was randomized at the group level, but we define $T$ by grant receipt. Of the $265$ groups assigned to the grant, $29$ never received it, $21$ for administrative reasons and $8$ for theft or diversion \citep{blattman2014generating}. We treat them as untreated; this exploratory as-treated analysis assumes non-receipt is as good as random. This may fail, since receiving groups were slightly more educated and wealthier, and is a limitation of our application. The ATE only summarizes whether the program worked on average; NEXIS asks \emph{for whom} and \emph{under which environmental conditions} the effect varies. 
 
\paragraph{Data hierarchy and candidate pool structure.} Variables operate at three distinct levels of the data hierarchy. \emph{(i)} Individual level: outcomes, demographics (age, sex, parental education). \emph{(ii)} Group level: treatment $T$ (constant within group), group composition (share female members). \emph{(iii)} Community/site level: language group (7 categories: Alur, Langi, Lugbara, Madi, Teso, Karamojong, Pallisa\footnote{The language variable is inherited from \citet{blattman2014generating}
and codes communities into seven ethnolinguistic clusters based on the dominant
language group of each district. Group~7 (labelled here \emph{Pallisa}) departs
from this pattern: it comprises all communities in Pallisa district, which lies
outside the administrative Teso sub-region and contains a roughly even mix of
Iteso (Ateso-speaking, Eastern Nilotic) and Bagwere/Banyole (Lugwere/Lunyole-speaking,
Bantu) communities.}), the $5{\times}5$ km Landsat-7 tile centred on the community centroid, hand-crafted spectral indices (NDVI, NDWI, MNDWI, NDBI, EVI, BSI; mean and std per tile $\rightarrow$ 12 scalars), and SAE neural features. All site-level features are constant within community; all individuals at the same site inherit the same satellite-derived features. We discuss the implications in Section~\ref{sec:uganda:limitations}.
 
\subsection{Satellite imagery pipeline}
\label{sec:uganda:imagery}
 
\paragraph{Google Earth Engine extraction.} We re-extract satellite imagery for all $331$ RCT sites directly from Google Earth Engine using \textbf{Landsat-7 ETM+} surface-reflectance imagery, on the following recipe:
\begin{itemize}[leftmargin=1.4em, itemsep=1pt, topsep=2pt]
    \item \emph{Time window.} 2005--2007 three-year cloud-free median composite. The YOP program was rolled out from $\sim$2008 onward, so this window is strictly pre-treatment and as close to the trial baseline as the Landsat-7 record over Northern Uganda allows without substantial cloud contamination.
    \item \emph{Spatial resolution.} 30\,m/pixel; $5{\times}5$\,km tiles centred on each site's GPS centroid.
    \item \emph{Bands.} Blue (B1), Green (B2), Red (B3), NIR (B4), SWIR1 (B5), SWIR2 (B7).
    \item \emph{Visualisation for VLM.} False-colour composites NIR/Green/SWIR1, $2$--$98$ percentile stretch per band per tile.
\end{itemize}
 
\paragraph{Prithvi-EO embeddings.} We use \textbf{Prithvi-EO} \citep{jakubik2023foundation} (IBM/NASA geospatial foundation model, 100M parameters), a Vision Transformer pretrained on Harmonized Landsat Sentinel-2 (HLS) multispectral imagery. For each tile we average the patch tokens of the \textbf{last (12th) encoder layer} into one $768$-dimensional embedding. This provides a rich, pretrained representation of land-cover and landscape structure without task-specific supervision.
 
\paragraph{Sparse autoencoder.} We train a \textbf{TopK SAE} \citep{gao2025scaling} on Prithvi-EO embeddings extracted from a \emph{Uganda national satellite grid} (full-country coverage, same Landsat-7 2005--2007 time window). The $331$ RCT sites are held out from SAE training; the national-grid corpus provides geographic diversity for learning a rich feature dictionary without data leakage. Whitening statistics (mean and std) are fit on the national corpus and applied to RCT-site embeddings at inference time.
\begin{itemize}[leftmargin=1.4em, itemsep=1pt, topsep=2pt]
    \item \emph{Architecture.} Encoder = linear $768 \rightarrow 1{,}024$ with bias; TopK activation with $k{=}25$; decoder = unit-norm column matrix $1{,}024 \rightarrow 768$ with bias.
    \item \emph{Training.} 2{,}000 epochs; batch size 64; learning rate $2{\times}10^{-4}$; 5-fold cross-validation on the national corpus.
\end{itemize}
 
\paragraph{Feature filtering.} Only SAE neurons active in at least $5$ of the $331$ RCT sites (activity threshold $Z_j > 0$) are retained. This yields $146$ active neurons out of $1{,}024$, forming the neural candidate set for NEXIS. The threshold prevents highly sparse neurons (active at 1--4 sites) from entering the regression, where they would have insufficient variation to reliably estimate an interaction effect.
 
\subsection{NEXIS instantiation}
\label{sec:uganda:nexis}
 
\paragraph{Representations.} Each unit representation unions $146$ SAE atoms with $24$ hand-crafted covariates: individual-level demographics, group composition (share female), community-level language-group dummies, and spectral indices (NDVI, NDWI, MNDWI, NDBI, EVI, BSI; mean and standard deviation per tile, $12$ scalars) computed from the same Landsat tiles that feed the SAE. Total: $170$ candidates. Hand-crafted features compete with neural ones in the selection. Language-group dummies aggregate the 14 districts into 7 ethnolinguistic clusters, increasing statistical power per group (mean ${\approx}47$ communities per cluster vs.\ ${\approx}24$ per individual district). The clustering is inherited from \citet{blattman2014generating} and reflects the dominant language of each district; it is not a modelling choice made here. A sensitivity analysis replacing language-group dummies with individual district dummies certifies Pallisa at the single-district level ($p \approx 7.9{\times}10^{-5}$), confirming it does not depend on aggregation. The Karamojong and Lugbara signals span multiple small districts and lose significance individually, so the language-group representation is load-bearing for those two discoveries: they are detectable only once districts are pooled into a cluster large enough to support reliable interaction estimation.

\paragraph{Search.} NEXIS searches the $170$ candidates in a single pass, in which the trial covariates and the satellite-derived features compete symmetrically.
 
\paragraph{Test.} The default linear $T \times Z_j$ interaction $t$-test (Equation~\ref{eq:linear_test}, Appendix~\ref{sec:method-details:test}), conditioning on the already-selected set $S$. We keep the default for its power, which matters at $n = 2{,}082$ with $170$ candidates: the GCM and PCM tests pay a measurable cost in power when the CATE is close to linear (Appendix~\ref{sec:method-details:test}).
 
\paragraph{Correction.} We run Algorithm~\ref{alg:nexis} at $\alpha = 0.05$. The forward step admits the most significant candidate $j$ if $p_j(S) \le \alpha / |\bar S|$, with the spectral-gap gate at $\rho = 0.5$ (Appendix~\ref{sec:method-details:rho}): a candidate whose conditional $t$-statistic is less than half that of the weakest selected coordinate is not admitted. Once the forward step stops at $\widetilde S$, the backward step keeps $j \in \widetilde S$ only if $p_j(A) \le \alpha/m$ for every $A \subseteq \widetilde S \setminus \{j\}$, with $m = 170$, i.e., $\alpha/m \approx 2.9 \times 10^{-4}$.
 
\paragraph{Standard errors.} Standard homoskedastic OLS, no clustering; see Section~\ref{sec:uganda:limitations} for limitations discussion.
 
\subsection{Discovered modifiers}
\label{sec:uganda:results}
 
Table~\ref{tab:nexis_appendix} reports the coordinates selected by the forward step, in two tiers. \textit{Certified} modifiers also survive the backward step: the Karamojong and Pallisa language groups and vegetation spatial heterogeneity (neuron 533) for skilled employment, and NDVI for log business assets. \textit{Candidates} are selected by the forward step but not certified: the Lugbara language group and perennial river presence (neuron 339) for skilled employment, and structured agricultural landscape (neuron 820) for log business assets. A candidate carries a conditional signal, since it entered the selection given the coordinates selected before it, but at some subset of the other selected coordinates its $p$-value exceeds $\alpha/m$. We read the candidates as hypotheses, not certified modifiers. For each coordinate, the table reports subgroup GATEs (active vs.\ inactive), the contrast $\Delta$ and two $p$-values. The \emph{marginal} $p$-value tests the unconditional $T \times Z_j$ interaction and characterizes the raw signal of each coordinate. The \emph{certification} $p$-value is the largest $p_j(A)$ over the subsets $A \subseteq \widetilde S \setminus \{j\}$ tested by the backward step, and the coordinate is certified if it is at most $\alpha/m$. The original analysis of \citet{blattman2014generating} reports significant positive average effects on both outcomes. The difference in means is also positive in our sample ($+0.32$ for skilled employment, $+0.61$ for log business assets), replicating the headline program-level finding.
 
\begin{table}[h]
  \centering\small
  \caption{Coordinates selected by NEXIS on YOP: \textit{certified} modifiers survive the backward step, \textit{candidates} are selected by the forward step only. GATE estimates for active vs.\ inactive subgroups (s.e.\ in parentheses), contrast $\Delta$, \emph{marginal} $p$-value of the unconditional linear $T \times Z_j$ interaction, and \emph{certification} $p$-value, i.e., the largest $p_j(A)$ over the subsets tested by the backward step, compared with $\alpha/m = 0.05/170 \approx 2.9{\times}10^{-4}$. Satellite-atom modifiers are reported under their VLM-assigned semantic label; activations are binarized at $z>0$ for those, and at the sample median for NDVI.}
  \label{tab:nexis_appendix}
  \footnotesize\setlength{\tabcolsep}{2pt}
  \begin{tabular}{@{}l ccc cc@{}}
  \toprule
  & \multicolumn{2}{c}{GATE} & & \multicolumn{2}{c}{$p$-value} \\
  \cmidrule(lr){2-3}\cmidrule(lr){5-6}
  Modifier & Active & Inactive & $\Delta$ & Marginal & Certification \\
  \midrule
  \multicolumn{6}{@{}l}{\textit{Panel A: Skilled Employment}} \\[2pt]
  \multicolumn{6}{@{}l}{\quad\textit{Certified}} \\
    Karamojong & $-0.030$ (0.060) & $+0.372$ (0.022) & $-0.403$ (0.063) & $7.7{\times}10^{-10}$ & $7.6{\times}10^{-8}$ \\
    Pallisa      & $+0.674$ (0.058) & $+0.288$ (0.022) & $+0.386$ (0.062) & $6.3{\times}10^{-8}$ & $1.0{\times}10^{-4}$ \\
    Vegetation spatial heterogeneity & $+0.214$ (0.038) & $+0.373$ (0.025) & $-0.159$ (0.045) & $6.7{\times}10^{-5}$ & $2.7{\times}10^{-4}$ \\[2pt]
  \multicolumn{6}{@{}l}{\quad\textit{Candidates}} \\
    Lugbara    & $+0.092$ (0.061) & $+0.347$ (0.023) & $-0.255$ (0.065) & $1.6{\times}10^{-5}$ & $3.2{\times}10^{-4}$ \\
    Perennial river presence       & $+0.089$ (0.097) & $+0.330$ (0.021) & $-0.242$ (0.100) & $2.1{\times}10^{-4}$ & $4.4{\times}10^{-4}$ \\
  \midrule
  \multicolumn{6}{@{}l}{\textit{Panel B: Log Business Assets}} \\[2pt]
  \multicolumn{6}{@{}l}{\quad\textit{Certified}} \\
    NDVI                            & $+0.668$ (0.061) & $+0.552$ (0.068) & $+0.115$ (0.092) & $5.6{\times}10^{-5}$ & $5.6{\times}10^{-5}$ \\[2pt]
  \multicolumn{6}{@{}l}{\quad\textit{Candidates}} \\
    Structured agricultural landscape & $+0.368$ (0.094) & $+0.649$ (0.051) & $-0.282$ (0.107) & $1.7{\times}10^{-2}$ & $1.7{\times}10^{-2}$ \\
  \bottomrule
  \end{tabular}
\end{table}
 
\paragraph{Results interpretation.}
We add three observations. First, none of the individual-level demographics (age, sex, parental education) is selected on either outcome. This is informative in its own right: within the precision of the trial, the program's impact shows no detectable difference along these axes, so there is no evidence that it favours some of the demographic groups represented over others. 
Second, the ethnolinguistic-group discoveries are geographically and historically coherent: these groups aggregate districts by dominant language but broadly index distinct regional geographies. The two dampening groups (Karamojong and, as a candidate, Lugbara) sit at distinct national peripheries in the northeast and the northwest respectively, and were both conflict-affected or isolated at trial time (baseline 2008, endline 2010--2012), while the amplifying group (Pallisa) lies outside the northern conflict core in eastern Uganda, with a mixed Iteso--Bagwere farming economy and shorter supply chains to central markets. Karamoja was under a government disarmament campaign through 2011, and West Nile was historically isolated from national infrastructure, underserved by post-LRA reconstruction flows, and further constrained by its position at the DRC and South Sudan border limiting cross-border trade integration. The pattern suggests that post-conflict market recovery is a binding constraint on how much of the grant compounds into durable skilled employment, a moderator invisible to survey demographics alone.
Third, on the formal status of the selected coordinates: under the assumptions of Theorem~\ref{thm:nexis_consistency} (Measurement and Representation Sufficiency, Principal Alignment, faithfulness, and validity and consistency of the CATE-equivalence test), NEXIS recovers the principal proxies of the direct effect modifiers with asymptotic probability at least $1-\alpha$, and the substantive interpretations offered in the main text would carry the prescriptive force we describe. In practice, those assumptions are stated rather than verified: we do not have ground truth on the true direct-modifier set for YOP, and entanglement on real learned dictionaries is non-negligible. Certified modifiers and candidates alike should therefore be read as hypotheses, the certified ones being the strongest our pipeline can produce on this dataset (see also Section~\ref{sec:uganda:limitations}), intended as inspirational starting points for the next iteration of the program rather than as established causal claims.

\subsection{Marginal screening vs.\ NEXIS}
\label{sec:uganda:marginal}
 
To quantify the experimental power paradox of Section~\ref{sec:power_paradox} on this dataset, we compare NEXIS to a pure marginal screen: each of the $170$ coordinates tested individually for a $T \times Z_j$ interaction, without conditioning. Table~\ref{tab:uganda_marginal} reports the result.
 
\begin{table}[h]
  \centering\small
  \caption{Marginal screening vs.\ NEXIS on YOP. ``Marginal'' counts features for which the unconditional $T \times Z_j$ interaction is significant at level $0.05$, uncorrected. ``NEXIS'' counts the coordinates selected by the forward step, split into those certified by the backward step and the candidates.}
  \label{tab:uganda_marginal}
  \begin{tabular}{@{}lcc@{}}
    \toprule
    Outcome & Marginal & NEXIS \\
    \midrule
    skilled employment & $71$ features & $3$ certified $+$ $2$ candidates \\
    log business assets & $45$ features & $1$ certified $+$ $1$ candidate \\
    \bottomrule
  \end{tabular}
\end{table}
 
Marginal testing on $170$ tests yields a small handful of expected false positives under the global null; the observed $45$--$71$ ``discoveries'' are far above that threshold and dominated by redundancy among correlated SAE neurons (a dense cluster of neurons active in overlapping sites all surface as marginally significant). NEXIS conditions each forward step on the features already selected, which eliminates the downstream significance of redundant features, and the backward step keeps only the coordinates that remain significant at level $\alpha/m$ given every subset of the others. The forward step selects $5$ and $2$ coordinates, of which the backward step certifies $3$ and $1$; under the assumptions of Theorem~\ref{thm:nexis_consistency}, the certified set equals the set of principal proxies with asymptotic probability at least $1-\alpha$.
 
\subsection{VLM interpretation protocol}
\label{sec:uganda:vlm}

NEXIS selects coordinates of the learned dictionary without knowing what they encode. The labels used in this appendix are therefore assigned afterwards, only to interpret the retained neurons, and play no role in selection or inference. For each retained neuron, we ask a vision-language model (Qwen2.5-VL-72B-Instruct \citep{bai2025qwen2}, 4-bit quantised, on a single H100 80\,GB GPU) what distinguishes the tiles where the neuron is most active from those where it is inactive:
\begin{enumerate}[label=(\roman*), labelindent=\parindent, leftmargin=*, itemsep=1pt, topsep=2pt]
    \item rank the $331$ RCT sites by activation;
    \item collect the top-$12$ (highest activation) and bottom-$12$ (zero or near-zero activation) tiles;
    \item present both sets side by side with a prompt that, in condensed form, reads: \emph{``These are pairs of satellite images from Uganda (Landsat-7, 2005--2007). The left column shows sites where a learned visual feature is strongly active; the right column shows sites where it is inactive. Describe in one short phrase what landscape or environmental property distinguishes the active from the inactive sites.''};
    \item post-process the response into a concise label.
\end{enumerate}
A label is a hypothesis about the visual content of a neuron, to be read together with its top- and bottom-activating tiles and its activation map (Figure~\ref{fig:uganda_neurons}).

\subsection{Limitations}
\label{sec:uganda:limitations}
 
We flag three limitations of the present analysis.
 
\paragraph{Multilevel inference.}
Our tests treat individuals as independent. Individual outcomes are the unit of analysis, the guarantee of Theorem~\ref{thm:nexis_consistency} is stated for i.i.d.\ units, and the individual-level test gives the power needed to generate hypotheses. Treatment, however, is assigned by group, and the satellite features vary by community. We therefore repeat the analysis with standard errors clustered by treatment-assignment group ($G = 439$) and district fixed effects, as in the original evaluation \citep{blattman2014generating}, on the $128$ candidates active in at least $5$ groups per treatment arm. NEXIS then certifies only the Karamojong and Lugbara language groups for skilled employment, and one satellite atom, labelled \emph{perennial water presence} by the VLM (neuron 261), for log business assets. The four environmental modifiers of Table~\ref{tab:nexis_appendix} keep some evidence: given the other coordinates selected for the same outcome, a randomization test that replays the group lottery within district gives them $p$-values between $1.1{\times}10^{-3}$ and $1.3{\times}10^{-2}$, short of $\alpha/m$. The loss comes from the number of units, not from the effects: the effective sample shrinks from individuals to groups, and community-level atoms are active in few groups, e.g., the river atom in $4$ treated groups. Clustering alone would leave the point estimates unchanged and only widen their uncertainty. The language groups vary at a coarser level still, but clustering at their own level ($7$ ethnolinguistic clusters) would ask whether the effect generalizes to other regions, which needs more sites. We therefore present the YOP modifiers as suggestive hypotheses.

\paragraph{Community-level environmental exploration.}
Because every individual at a given community shares the same satellite tile, the satellite-derived features cannot capture within-community heterogeneity in environmental exposure (e.g., proximity to a river within a community). The $5{\times}5$\,km tile is centred on the community centroid; finer-grained, individual-level GPS would allow a more precise read of the modifier. We see this as a measurement gap to close in next-generation deployments rather than a methodological issue with NEXIS.
 
\paragraph{Linear interaction test.}
The linear $T \times Z_j$ test we use here is consistent only against alternatives in which heterogeneity is linear in $Z_j$ given the conditioning set; threshold or U-shape interactions in a continuous candidate could in principle pass undetected. The bite is mild on the YOP candidate pool: the language-group dummies, sex, and the binarized demographic indicators are by construction immune (linearity is fully general for binary covariates), and the SAE atoms enter the regression as TopK-sparse activations that are zero on most sites and clustered near a small set of values when active, making a linear-in-$Z_j$ alternative a reasonable working approximation. The continuous spectral indices (NDVI and the rest) and the share-female group composition do carry the assumption non-trivially. The PCM test of Appendix~\ref{sec:method-details:test} is a drop-in replacement that also detects non-monotone alternatives; we use the linear test here for its power at the present sample size.

\subsection{Computational budget}
\label{sec:uganda:compute}
 
\begin{table}[h]
\centering\small
\caption{Compute budget for the YOP analysis.}
\label{tab:uganda_compute}
\begin{tabular}{@{}lll@{}}
\toprule
Component & Hardware & Runtime \\
\midrule
GEE imagery extraction (RCT $+$ national grid) & cloud CPU & $\sim$1--2\,h \\
Prithvi-EO embedding extraction               & RTX 2080 Ti & $\sim$30\,min \\
SAE training (1{,}024 hidden, 2{,}000 epochs) & RTX 2080 Ti & $\sim$1\,h \\
NEXIS analysis (both outcomes, 170 candidates)& CPU         & $<$\,5\,min \\
VLM interpretation (4 neurons, top/bottom 12) & H100 80\,GB & $\sim$30\,min \\
\bottomrule
\end{tabular}
\end{table}

\newpage

\section{Case Study 2: LEAP 1000 Programme (Ghana)}
\label{sec:ghana}
 
\subsection{Study design and identification strategy}
\label{sec:ghana:design}

\paragraph{Program.}
The Ghana Livelihood Empowerment Against Poverty 1000 (LEAP 1000) is implemented by Ghana's Department of Social Welfare. It provides bimonthly cash transfers to extremely poor households with children aged 0--1, with a short- to medium-term objective to reduce poverty and improve household welfare and nutrition among extremely poor households with newborns. Its evaluation was conducted by UNICEF Innocenti and ISSER (University of Ghana), in collaboration with the Carolina Population Center (University of North Carolina at Chapel Hill) and the Navrongo Health Research Centre.

\paragraph{Study design.}
The evaluation covers $2{,}331$ households observed at two waves: baseline 2015 and endline 2017 (balanced panel). Geographic coverage spans $162$ communities with GPS centroids across five districts (East Mamprusi, Karaga, Yendi, Bongo, Garu-Tempane) in two regions (Northern and Upper East Ghana), illustrated in Figure~\ref{fig:ghana_districts}. Treatment and comparison groups are constructed via a regression discontinuity design: eligibility is determined by a proxy means test (PMT) score cutoff, and a household is treated ($T = 1$) if its PMT score is at or below the cutoff, so that poorer households receive the transfer. Households just below and just above the threshold form the treatment and comparison groups respectively \citep{GhanaLEAP1000EvaluationTeam2018}. Consistent with this design, $154$ of $162$ communities contain both treated and comparison households, with a median within-community treatment share of $\approx 50\%$. 

\begin{figure}[h!]
  \centering
  \includegraphics[width=0.6\textwidth]{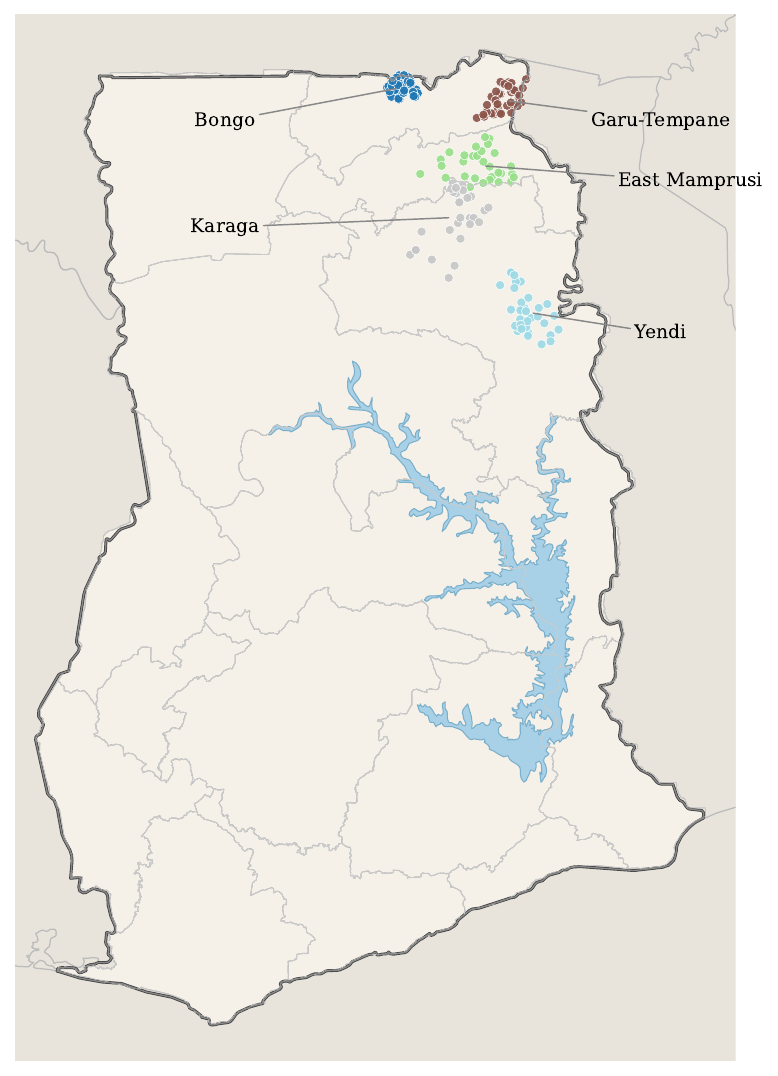}
  \caption{\small Geographic distribution of the $162$ LEAP 1000 communities
  across the five evaluation districts in Northern and Upper East Ghana.}
  \label{fig:ghana_districts}
\end{figure}

\paragraph{Sample.}
Of the $2{,}331$ households, $1{,}185$ are treated ($50.8\%$) and $1{,}146$ are
comparison. Baseline adult-equivalent expenditure: mean $120.9$~GH\textcent/month;
treated arm $117.8$, comparison $124.0$ --- near-perfect balance on most of the 24
pre-treatment covariates. These include household composition, head of household
characteristics, housing and WASH conditions, and livelihood indicators; derived
aggregate indices (e.g., housing deprivation, livelihood diversity) are also
constructed and included.

\paragraph{Outcome.}
Adult-equivalent household consumption expenditure per month, deflated to constant
Greater Accra August-2017 prices. This is the primary welfare measure in the
evaluation and the quantity most directly targeted by the cash transfer.

\paragraph{Estimand: local average treatment effect on the treated.}
The regression discontinuity design identifies the effect of eligibility for households
near the PMT cutoff. We interpret this as the local average treatment effect on the
treated (local ATT) for the marginal eligible population under the explicit assumption
that noncompliance is ignorable and does not alter the eligibility contrast. This
assumption and possible selection into the balanced panel are limitations.

\paragraph{Identification: difference-in-differences.}
Let $\text{wave}_t$ be an indicator equal to $1$ at endline and $0$ at baseline.
With two waves we estimate:
\[
Y_{t} = \alpha + \beta_T T + \beta_{\text{wave}} \cdot \text{wave}_t
+ \delta \cdot (T \times \text{wave}_t) + \varepsilon_{t},
\]
where $\delta$ identifies the difference-in-differences (DiD) local ATT under
parallel trends. The assumption is supported by: (i) near-perfect baseline balance
on most of the 24 pre-treatment covariates, making a differential trend implausible;
(ii) the near-random within-community treatment assignment induced by PMT-score
proximity, which further reduces systematic differences between arms; (iii) the
absence of anticipation effects, as 2015 measurements were collected before
transfers had begun. The DiD estimate is $\hat\delta = +7.35$~GH\textcent/month (s.e.\ $3.84$, clustered by community). Both arms
show a nominal decline from 2015 to 2017 because 2017 values are expressed in
constant August-2017 prices while 2015 nominal expenditures were higher in current
prices; the DiD difference of $+7.35$~GH\textcent/month is the real causal
estimate.

\subsection{NEXIS under quasi-experimental identification}
\label{sec:ghana:quasi}

LEAP 1000 is a quasi-experiment: treatment assignment is not randomized but governed by a PMT score threshold, and causal identification of the average effect relies on two assumptions, local as-if-randomization near the cutoff and parallel trends across waves. Theorem~\ref{thm:nexis_consistency} is stated for a randomized experiment; we establish here that its two guarantees extend to this setting under natural analogues of those assumptions, each addressing a distinct step in the proof.
Throughout this subsection, $\tau$, $\mathcal S^\star$, and the structural assumptions are understood under the distribution of the marginal eligible population near the cutoff.

\paragraph{Asymptotic recovery under conditional parallel trends.}
NEXIS is applied to the first-differenced outcome $\Delta Y := Y_{\mathrm{post}} - Y_{\mathrm{pre}}$, using treatment $T$ as defined by threshold assignment. The asymptotic exact-recovery guarantee (Equation~\ref{eq:selection_consistency}) rests on Assumption~\ref{ass:test}: the effect heterogeneity-equivalence tests must be valid under the null and consistent under alternatives. For this application, we impose that the first-differenced untreated potential outcome is mean-independent of $T$ given every fixed subset tested by NEXIS:
\begin{equation}
    \label{eq:cpt}
    \E[\Delta Y(0) \mid T, \bm{Z}^{A}] = \E[\Delta Y(0) \mid \bm{Z}^{A}]
    \quad \text{a.s. for every } A \subseteq [m].
\end{equation}
This is a subset-wise covariate-adjusted parallel trends condition \citep{callaway2021difference}: the untreated counterfactual trend carries no residual mean dependence on treatment assignment at any conditioning set used by the algorithm. Together with the regularity conditions of the chosen test, Equation~\ref{eq:cpt} supplies the application-specific validity required by Assumption~\ref{ass:test}(a); test consistency under alternatives remains Assumption~\ref{ass:test}(b). The recovery proof of Section~\ref{sec:proofs:recovery} then applies unchanged. Equation~\ref{eq:cpt} is supported here by the near-random within-community assignment induced by PMT-score proximity and by the strong baseline balance documented above.

\paragraph{Causal interpretation in the local population.}
Within a bandwidth around the PMT cutoff, local as-if-randomization makes units on either side comparable in expectation for variables varying continuously at the threshold \citep{lee2010regression}. Together with Equation~\ref{eq:cpt}, this identifies the conditional local ATT from the observed treatment groups. It is not what upgrades conditioning on $\bm W^{\mathrm{dir}}$ to intervening on it: once the local CATE is identified, $\tau(\bm W^{\mathrm{dir}})=\tau^{do}(\bm W^{\mathrm{dir}})$ follows from the definition of $\bm W^{\mathrm{dir}}$ as the complete direct-modifier set and the back-door argument of Section~\ref{sec:proofs:causal}. The identified interventional characterization is accordingly restricted to the marginal eligible population near the cutoff.

The subset-wise condition is a strong extension of standard conditional parallel trends because it must hold for every tested $A$; local as-if-randomization is standard. Both are precisely stateable and empirically supportable but not formally testable from the data alone. Together with the structural and testing assumptions of Theorem~\ref{thm:nexis_consistency}, they extend its guarantees to this quasi-experimental setting, with the heterogeneity characterization restricted to the marginal eligible population near the PMT cutoff.

\subsection{Satellite data and SAE pipeline}
\label{sec:ghana:pipeline}
 
\paragraph{Imagery extraction.}
Landsat-8 OLI imagery was extracted from Google Earth Engine for all 162 LEAP community centroids. We use a cloud-free median composite aligned to the 2015 baseline, at 30\,m/pixel resolution. Tiles are rendered as false-colour composites (NIR/Green/SWIR2, 2--98 percentile stretch per band) for VLM interpretation, and a set of spectral indices is derived from the same imagery. A Ghana national satellite grid extracted in the same time window serves as the SAE training corpus; the 162 LEAP sites are held out.
 
\paragraph{Prithvi-EO embeddings.}
Each tile is passed through Prithvi-EO \citep{jakubik2023foundation} (IBM/NASA geospatial Vision Transformer pretrained on Harmonized Landsat Sentinel-2 multispectral imagery). We average the patch tokens of the last (12th) encoder layer into one 768-dimensional embedding. Whitening statistics are fit on the national corpus and applied to the LEAP embeddings at inference.
 
\paragraph{Sparse Autoencoder.}
A TopK SAE with 4,096 hidden dimensions ($k=25$) is trained for 2,000 epochs (batch size 256, learning rate $2\times10^{-4}$) on the national grid embeddings. The larger dictionary (4,096 vs.\ 1,024 for Uganda) reflects Ghana's more geographically diverse training corpus. Of 4,096 neurons, 131 are active in at least 5 of the 162 LEAP communities and enter the candidate pool; neurons active in fewer than 5 communities offer insufficient variation for interaction testing.
 
\subsection{NEXIS configuration}
\label{sec:ghana:nexis}

\paragraph{Representations.}
Each unit representation combines features that operate at two levels
of the data hierarchy: $131$ learned landscape features from the SAE and $12$
spectral indices computed from the same imagery (NDVI, NDWI, MNDWI, NDBI, EVI, BSI; mean and standard deviation per tile) at the community level, and $24$
household-level survey covariates, for a total of $167$ candidates. The satellite
features are community-level constants, every household within a community
shares the same tile, while survey covariates vary at the household level.
NEXIS treats all of them symmetrically as candidate effect modifiers in the
selection procedure. Standard errors for the interaction tests are clustered by
community ($G = 162$), which is the appropriate grouping unit for satellite
features: treating within-community households as independent observations would
severely understate variability for community-level regressors. We note that
community is a geographic grouping variable rather than the true
threshold-assignment unit.

\paragraph{Search.}
NEXIS searches the $167$ candidates jointly in a single pass.

\paragraph{Interaction test.}
Linear $T \times Z_j$ interaction test conditional on the already-selected set,
with cluster-robust (CR1S) standard errors clustered by community ($G = 162$) and
$t$-statistics referred to a $t_{G-1}$ distribution. 

\paragraph{Correction.} We run Algorithm~\ref{alg:nexis} at $\alpha = 0.05$, as in the Uganda application (Appendix~\ref{sec:uganda:nexis}). The forward step admits the most significant candidate $j$ if $p_j(S) \le \alpha / |\bar S|$, with the spectral-gap gate at $\rho = 0.5$. Once it stops at $\widetilde S$, the backward step keeps $j \in \widetilde S$ only if $p_j(A) \le \alpha/m$ for every $A \subseteq \widetilde S \setminus \{j\}$, with $m = 167$, i.e., $\alpha/m \approx 3.0 \times 10^{-4}$. An exploratory analysis without correction is reported in
Section~\ref{sec:ghana:exploratory}.

\subsection{Discovered effect modifiers}
\label{sec:ghana:results}
 
NEXIS retains two satellite-derived landscape features, and no spectral index or household covariate (Table~\ref{tab:ghana_nexis}). The forward step selects no other coordinate and both pass the backward step, so both are certified and there are no candidates. Both carry high-confidence VLM semantic labels. The difference-in-differences estimate of the local ATT is $+7.35$~GH\textcent/month (s.e.\ $3.84$); in the few communities where these features are present, the estimated effect is $6$--$8\times$ larger.

\paragraph{Marginal screening vs.\ NEXIS.} As for Case Study~1 (Appendix~\ref{sec:uganda:marginal}), a pure marginal screen, testing each of the $167$ coordinates individually for a $T \times Z_j$ interaction at level $0.05$ without correction (standard errors clustered by community), selects $18$ coordinates ($17$ SAE atoms and one survey covariate, farming household), against the $2$ certified by NEXIS.
 
\begin{table}[h]
\centering\small
\caption{
NEXIS discoveries for LEAP 1000, both certified by the backward step. GATE estimates (in GH\textcent/month) for active vs.\ inactive subgroups (s.e.\ clustered by community in parentheses), contrast $\Delta$, \emph{marginal} $p$-value of the unconditional linear $T \times Z_j$ interaction, and \emph{certification} $p$-value, i.e., the largest $p_j(A)$ over the subsets tested by the backward step, compared with $\alpha/m = 0.05/167 \approx 3.0{\times}10^{-4}$; both use standard errors clustered by community. Satellite-atom modifiers are reported under their VLM-assigned semantic label; activations are binarized at $z>0$.}
\label{tab:ghana_nexis}
\begin{tabular}{@{}lrrrrr@{}}
\toprule
& & & & \multicolumn{2}{c}{$p$-value} \\
\cmidrule(l){5-6}
Modifier & GATE (active) & GATE (inactive) & $\Delta$ & Marginal & Certification \\
\midrule
Ephemeral waterways  & $+42.9\ (15.8)$ & $+6.0\ (3.9)$ & $+36.9$ & $3.2\times10^{-8}$ & $3.2\times10^{-8}$ \\
Closed-canopy forest & $+56.2\ (19.7)$ & $+6.4\ (3.9)$ & $+49.8$ & $2.0\times10^{-5}$ & $2.0\times10^{-5}$ \\
\bottomrule
\end{tabular}
\end{table}
 
\paragraph{Ephemeral waterways.}
The VLM describes this feature (confidence: high) as capturing narrow seasonal streams with adjacent vegetation; inactive tiles show uniform land cover with no visible watercourses. It is active in 6 communities (83 households).
 
\textit{Mechanism hypothesis:} seasonal water access provides a complementary input that cash transfers alone cannot supply. In Northern Ghana's predominantly rainfed smallholder system, proximity to an ephemeral stream enables micro-irrigation of adjacent plots. A LEAP transfer in this context can be invested in seeds, fertiliser, or simple tools that water-adjacent households can productively deploy; the same transfer in a waterless community generates only short-run consumption smoothing with no durable agricultural returns. The ephemeral character of the waterways is relevant: these are seasonal flood corridors active during the wet season, and the bimonthly transfer timing means at least one payment arrives during the agricultural season. By removing acute liquidity constraints, the cash transfer enables households to capitalise on this pre-existing environmental asset for dry-season production.
 
The temporal analysis supports this reading. We present paired 2015/2017 Landsat-8 composites of the six active communities to the VLM with a neutral prompt, which asks for every land-cover change between the two years without the feature's label. The VLM reports a cropland increase in five of the six communities and more bare soil from clearing in the sixth (Table~\ref{tab:ghana_temporal}). To locate the change, we compute the per-pixel NDVI change between the two composites, net of the tile's median change (2017 is greener overall), and keep the patches where it exceeds $0.05$ (at least $1.8$\,ha, water masked). Local NDVI increase covers $7.9\%$ and $3.5\%$ of the tiles of communities 951 and 624, where it forms clear, compact patches (boxes in Figure~\ref{fig:ghana_neurons}), between $1.7\%$ and $2.7\%$ in 675, 395 and 655, and $0.3\%$ in 1265, where the change the VLM reports coincides with an expanding water body. Cropland expansion fits this hypothesis: where seasonal streams keep soils moist beyond the rains, the binding constraint on cultivating more land is cash for seeds, fertiliser and hired labour rather than water, so a transfer can be turned into additional cultivated plots, while without such water extra land yields little. Whether the extra harvest raises consumption, or higher consumption and cropland expansion both follow from the transfer, our data cannot tell, since both are observed over the same 2015--2017 window. Since these communities also contain comparison households, the change is consistent with, though not attributable to, LEAP-enabled intensification.
 
\begin{table}[h]
\centering\small
\caption{Per-community temporal analysis for the six waterway-active LEAP communities, 2015 baseline vs.\ 2017 endline composites: cropland change reported by the VLM (neutral prompt) and share of the tile where NDVI rose by more than $0.05$ beyond the tile's median change (patches $\ge 1.8$\,ha, water masked).}
\label{tab:ghana_temporal}
\begin{tabular}{@{}llc@{}}
\toprule
Community & Cropland (VLM) & Local NDVI increase (\% area) \\
\midrule
951  & $\uparrow$ increase & $7.9$ \\
675  & $\uparrow$ increase & $1.7$ \\
395  & $\uparrow$ increase & $2.7$ \\
1265 & $\uparrow$ increase & $0.3$ \\
655  & not listed & $1.9$ \\
624  & $\uparrow$ increase & $3.5$ \\
\bottomrule
\end{tabular}
\end{table}
 
\paragraph{Closed-canopy forest.}
VLM label (confidence: high): dense, continuous forest canopy with minimal breaks; inactive tiles show fragmented vegetation with open spaces. Active in 5 communities (42 households).
 
\textit{Mechanism hypothesis:} closed-canopy forest patches are rare endowments in this predominantly savannah landscape. Households adjacent to forest access non-timber forest products, fuelwood, and forest-edge cultivation, and benefit from ecological services (soil moisture retention, microclimate buffering) absent in open degraded savannah. A transfer in these communities complements existing forest-based livelihood strategies, enabling production intensification in ways not feasible elsewhere. Because these forest resources provide a natural safety net for basic consumption (e.g., foraging), the cash transfer can be diverted away from emergency food purchases and directed toward higher-value dietary diversity or productive assets.  The anomalously large GATE may additionally reflect a selection effect: forest-proximate communities may be better positioned to translate additional income into sustained consumption gains.
 
\paragraph{No household-covariate discoveries.}
No survey covariate survives FWER or FDR correction. The overall pattern is consistent with
LEAP 1000's targeting design: by restricting to the most deprived households in
northern Ghana, the program selects an unusually homogeneous population in terms of
household demographics, and the residual heterogeneity is environmental rather than
demographic.
 
\subsection{VLM interpretation protocol}
\label{sec:ghana:vlm}

Labels are assigned as in Case Study~1 (Appendix~\ref{sec:uganda:vlm}): after selection, only to interpret the retained features, with the same model. For each retained feature:
\begin{enumerate}[label=(\roman*), labelindent=\parindent, leftmargin=*, itemsep=1pt, topsep=2pt]
    \item rank all communities in the Ghana national grid by activation;
    \item collect the top-$12$ and bottom-$12$ satellite tiles;
    \item present both sets side by side with a prompt that, in condensed form, reads: \emph{``These are pairs of satellite images from Ghana (Landsat-8, 2015). The left column shows sites where a learned visual feature is strongly active; the right column shows sites where it is inactive. Describe in one short phrase what landscape or environmental property distinguishes the active from the inactive sites.''};
    \item post-process the response into a concise label and record the model's confidence, reported with each label.
\end{enumerate}

\subsection{Exploratory analysis}
\label{sec:ghana:exploratory}
 
Running NEXIS without multiple-testing correction on an earlier, smaller pool, which keeps the $72$ neurons active in at least $10$ communities together with the $24$ survey covariates and $6$ spectral indices, reveals additional patterns that are not statistically sufficient to certify discoveries but are informative for hypothesis generation. This pool excludes the two certified neurons, which are active in fewer communities. The most substantive pattern is a feature labelled by the VLM as \emph{sparse burn scar presence} (small irregular burn scars scattered across vegetation; confidence: medium), active in 12 communities geographically concentrated in the East Mamprusi and Karaga districts (Northern Region). The GATE is $+26.3$~GH\textcent/month in active communities vs.\ $+5.6$ in inactive ones; the conditional $p = 0.0085$ in this uncorrected run, but the GATE contrast is not significant ($p = 0.30$), and the signal does not survive FDR correction.

\begin{figure}[h]
  \centering
  \includegraphics[width=\textwidth]{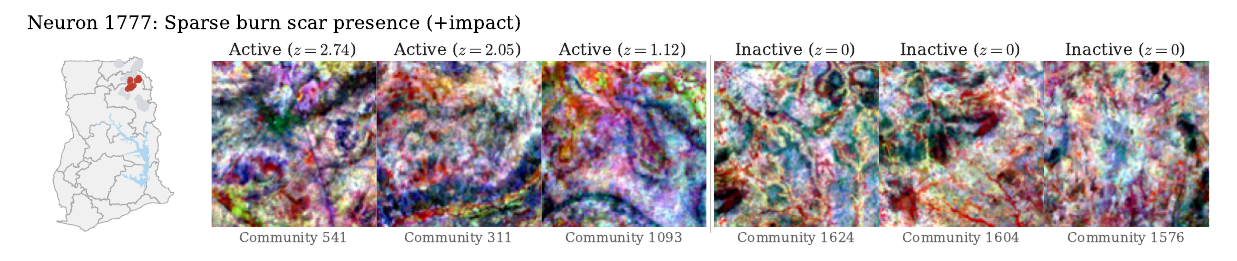}
  \caption{\small Top- and bottom-activating Landsat-8 tiles in false-color composite (NIR, Green, SWIR) for burn-scar, with the top 12 active communities concentrated in the East Mamprusi--Karaga cluster (see map). Active tiles show sparse irregular burn scars scattered across vegetation (VLM label: \emph{sparse burn scar presence}; confidence: medium); inactive tiles show uniform vegetation cover with no visible burn scars.}
  \label{fig:ghana_neuron_1777}
\end{figure}
 
The burn-scar signal has a natural external anchor: the Ghana National Development Planning Commission's 2015 Annual Progress Report \citep{NDPC2016GSGDA2015APR} documents, under natural disasters and vulnerability (Section 4.4.10), that $151$ metropolitan, municipal and district assemblies were affected by natural disasters in 2015, especially flooding and bush fires, with the bulk of incidents in the Eastern and Northern Regions. The burn scars in the 2015 imagery of the active communities are consistent with this regional record, in the same year the baseline was collected.
 
The mechanism hypothesis is that receiving a cash transfer in a community exposed to a recent fire shock amplifies program impact: fire events destroy subsistence assets and create acute liquidity needs, and a LEAP transfer in this context provides insurance-like relief that buffers the shock and translates into a larger net consumption gain, consistent with the broader literature on cash transfers as covariate-shock insurance. This contextualises the treatment effect: a larger increase in expenditure in these specific communities likely reflects an emergency coping mechanism to replace destroyed physical assets or staple crops, rather than a welfare-improving investment. The geographical specificity of the signal (East Mamprusi--Karaga cluster) makes it too localized to establish general patterns. We flag this as a direction to examine in larger-scale implementations with greater diversity of active communities.
 
\subsection{Limitations}
\label{sec:ghana:limitations}
 
\paragraph{Local identification.}
The regression discontinuity design identifies treatment effects only in the
neighborhood of the PMT eligibility cutoff: the analytic sample consists of
households close to the eligibility threshold, who are the ``best off'' among the
eligible population. The endline evaluation report \citep{GhanaLEAP1000EvaluationTeam2018}
notes explicitly that RDD estimates are likely lower bounds relative to the effect
for the average eligible household, precisely because the RDD sample selects
marginally eligible households rather than the most deprived. The heterogeneity
characterization produced by NEXIS inherits this scope restriction: it describes
which environmental conditions amplify or dampen the program effect for households
near the threshold, not for the full LEAP-eligible population. This is the standard
scope of any RDD-based analysis and does not affect the internal validity of the
findings, but should be borne in mind when extrapolating to program targeting
decisions that affect households far from the cutoff.

\paragraph{Community-level environmental exploration.}
Because every household in a given community shares the same satellite tile centred
on the community centroid, the satellite-derived features cannot capture
within-community heterogeneity in environmental exposure, for example, which
households are closest to a waterway or a forest patch. Household-level GPS
coordinates would allow more precise modifier assignment. 

\paragraph{Linear interaction test.}
The linear $T \times Z_j$ test is consistent only against alternatives in which
heterogeneity is linear in $Z_j$ given the conditioning set; threshold or
non-monotone interactions in a continuous candidate could in principle pass
undetected. The satellite atoms, as TopK-sparse activations, take values in a
small discrete range when active, making the linear approximation reasonable. Linearity is exact for the binary
survey covariates; the other survey covariates (counts, ratios, head age) and the
$12$ continuous spectral indices carry the assumption non-trivially. The PCM test of Appendix~\ref{sec:method-details:test}
is a drop-in replacement that also detects non-monotone alternatives, and
is the recommended choice when enough data are available.

\paragraph{Active-community counts.}
The discovered features are active in only 6 (83 households) and 5 (42 households) communities respectively. Clustered inference with so few active clusters can overstate precision: a wild cluster bootstrap by community, testing each neuron given the other, gives $p = 0.015$ and $p = 0.056$ respectively, far above $\alpha/m$. Community is also not the treatment-assignment unit (households are assigned by PMT score), so neither the clustered nor the bootstrap $p$-values are design-based. The magnitudes should therefore be treated with caution pending larger evaluations.
 
\subsection{Computational budget}
\label{sec:ghana:compute}
 
\begin{table}[h]
\centering\small
\caption{Compute budget for the LEAP 1000 analysis.}
\label{tab:ghana_compute}
\begin{tabular}{@{}lll@{}}
\toprule
Component & Hardware & Runtime \\
\midrule
GEE imagery extraction (162 LEAP communities $+$ national Ghana grid) & cloud CPU   & $\sim$1--2\,h \\
Prithvi-EO embedding extraction (LEAP $+$ national)                   & H100 80\,GB & $\sim$30\,min \\
SAE training (4{,}096 hidden, 2{,}000 epochs)                         & H100 80\,GB & $\sim$2\,h \\
NEXIS analysis (167 candidates)                                      & CPU         & $<$5\,min \\
VLM interpretation (2 neurons $\times$ top/bottom 12 images)          & H100 80\,GB & $\sim$30\,min \\
VLM temporal analysis (6 communities $\times$ 2 years)                & H100 80\,GB & $\sim$15\,min \\
\bottomrule
\end{tabular}
\end{table}
 